\documentclass[twoside,11pt]{article}

\RequirePackage[OT1]{fontenc}

\usepackage{algorithmic}
\usepackage{algorithm}

\usepackage{graphicx}
\usepackage{subfigure}
\usepackage{multicol}
\RequirePackage[OT1]{fontenc}
\usepackage{latexsym}
\usepackage{datetime}
\usepackage{alltt}
\usepackage{amsmath,amssymb,amsxtra,amsfonts,euscript}
\usepackage[colorlinks=false,allbordercolors={1 1 1}]{hyperref}
\usepackage{jmlr2e}

\def\ints{{{\rm Z} \kern -.35em {\rm Z} }}  % ins
\def\smallints{{{\rm Z} \kern -.3em {\rm Z} }}  % small ins
\def\pints{{{\rm I} \kern -.15em {\rm N} }}      % pints
\newcommand{\reals}{\mathbb R}
\newcommand{\nats}{\mathbb N}

\DeclareMathOperator{\tr}{tr}
\DeclareMathOperator{\Sp}{\mathrm Sp}

\def\cplx{{{\rm I} \kern -.45em {\rm C} }}       % complex
\def\l2{\rm {\mathcal L}^{2}(\reals)}            % l2
\def\im{\mathrm{Im}}

\newcommand{\be}{\begin{equation}}
\newcommand{\ee}{\end{equation}}
\newcommand{\bea}{\begin{eqnarray}}
\newcommand{\eea}{\end{eqnarray}}
\newcommand{\beaa}{\begin{eqnarray*}}
\newcommand{\eeaa}{\end{eqnarray*}}
\newcommand{\bnad}{\begin{nad}}
\newcommand{\enad}{\end{nad}}
\DeclareMathOperator{\diag}{diag}
\DeclareMathOperator{\rank}{rank}

\newcommand{\psdge}{\succcurlyeq}

\newcommand{\di}{{\,\mathrm{d}}}
\newcommand{\latop}[2]{\genfrac{}{}{0pt}{}{#1}{#2}}

\newcommand{\eps}{\epsilon}

\newcommand{\nin}{\in\!\!\!\!\!/\,}

\newcommand{\Expect}{\operatorname{\mathbb{E}}}

\renewcommand{\mod}{{\rm mod\,}}

\renewcommand{\ell}{l}

\def\GDd{\mathrm{G}(D,d)}
\def\bQ{\mathbf{Q}}
\def\bP{\mathbf{P}}

\def\sX{\mathcal{X}}
\def\sY{\mathcal{Y}}

\def\bSigma{\boldsymbol\Sigma}

\def\bx{\mathbf{x}}

\def\by{\mathbf{y}}

\def\bv{\mathbf{v}}
\def\bu{\mathbf{u}}
\def\bU{\mathbf{U}}
\def\bV{\mathbf{V}}
\def\bB{\mathbf{B}}
\def\b0{\mathbf{0}}
\def\bO{\mathbf{O}}
\def\bI{\mathbf{I}}

\def\bDelta{\boldsymbol{\Delta}}
\def\bA{\mathbf{A}}
\def\bX{\mathbf{X}}
\def\bZ{\mathbf{Z}}
\def\bL{\mathbf{L}}

\def\rmL{\mathrm{L}}

\def\bbH{\mathbb{H}}

\DeclareMathOperator*{\argmin}{arg\,min}

\jmlrheading{15}{2014}{193-251}{12/11; Revised 6/13}{1/14}{Zhang and Lerman}

\ShortHeadings{A Novel M-Estimator for Robust PCA}{Teng Zhang and Gilad Lerman}
\firstpageno{193}

\begin{document}

\title{A Novel M-Estimator for Robust PCA}

\author{\name Teng Zhang \email zhang620@umn.edu \\
       \addr Institute for Mathematics and its Applications\\
        University of Minnesota\\
        Minneapolis, MN 55455, USA
       \AND
       \name Gilad Lerman\thanks{Gilad Lerman is the corresponding author.} \email lerman@umn.edu \\
       \addr School of Mathematics\\
        University of Minnesota\\
        Minneapolis, MN 55455, USA}

\editor{Martin Wainwright}

\maketitle

\begin{abstract}%   <- trailing '%' for backward compatibility of .sty file
We study the basic problem of robust subspace recovery.
That is, we assume a data set that some of its points are sampled around a fixed
subspace and the rest of them are spread in the whole ambient space, and we aim to recover the fixed underlying subspace.
We first estimate ``robust inverse sample covariance'' by solving a convex minimization procedure; we then recover the subspace by the bottom eigenvectors of this matrix (their number correspond to the number of eigenvalues close to 0).
We guarantee exact subspace recovery under some conditions on the underlying data. Furthermore, we propose a fast iterative algorithm, which linearly converges to the matrix
minimizing the convex problem.
We also quantify the effect of noise and regularization and discuss many other practical and theoretical issues for improving the subspace recovery in various settings.
When replacing the sum of terms in the convex energy function (that we minimize) with the sum of squares of terms,
we obtain that the new minimizer is a scaled version of the inverse sample covariance (when exists). We thus interpret our minimizer and its subspace (spanned by its bottom eigenvectors)
as robust versions of the empirical inverse covariance and the PCA subspace respectively.
We compare our method with many other algorithms for robust PCA on synthetic and real data sets and demonstrate state-of-the-art speed and accuracy.
\end{abstract}


\section{Introduction}

The most useful paradigm in data analysis and machine learning is arguably the modeling of data by a low-dimensional subspace.
The well-known total least squares solves this modeling problem by finding the subspace minimizing the sum of squared errors of data points.
This is practically done via principal components analysis (PCA) of the data matrix.
Nevertheless, this procedure is highly sensitive to outliers.
Many heuristics have been proposed for robust recovery of the underlying subspace.
Recent progress in the rigorous study of sparsity and low-rank of data has resulted
in provable convex algorithms for this purpose.
Here, we propose a different rigorous and convex approach, which is a special M-estimator.

Robustness of statistical estimators has been carefully studied for several decades \citep{huber_book, robust_stat_book2006}.
A classical example is the robustness of the geometric median \citep{lopuha_rousseeuw_robust91}.
For a data set $\sX = \{\bx_i\}_{i=1}^N \subset \reals^D$,
the geometric median is the minimizer of the following function of $\by \in \reals^D$:
\be
\label{eq:geom_med}
\sum_{i=1}^N \|\by - \bx_i\|\,,
\ee
where $\|\cdot\|$ denotes the Euclidean norm.
This is a typical example of an M-estimator, that is, a minimizer of a function of the form $\sum_{i=1}^N \rho(r_i)$, where $r_i$ is a residual of the $i$th data point, $\bx_i$, from the parametrized object we want to estimate. Here, $r_i= \|\by - \bx_i\|$, $\rho(x)=|x|$ and we estimate $\by \in \reals^D$, which is parametrized by its $D$ coordinates.

There are several obstacles in developing robust and effective estimators for subspaces.
For simplicity, we discuss here estimators of linear subspaces and thus assume that the data is centered at the origin.\footnote{This is a common assumption to reduce the complexity of the subspace recovery problem \citep{candes_wright_robust_pca09,Xu2010,Xu2012,robust_mccoy}, where \citet{robust_mccoy} suggest centering by the geometric median. Nevertheless, our methods easily adapt to affine subspace fitting by simultaneously estimating both the offset and the shifted linear component, but the justification is a bit more complicated then.}
A main obstacle is due to the fact that the set of $d$-dimensional linear subspaces in $\reals^D$, that is, the Grassmannian $\GDd$, is not convex.
Therefore, a direct optimization on $\GDd$ (or a union of $\GDd$ over different $d$'s) will not be convex (even not geodesically convex)
and may result in several (or many) local minima.
Another problem is that extensions of simple robust estimators of vectors to subspaces (e.g., using $\ell_1$-type averages) can fail by a single far away outlier.
For example, one may extend the $d$-dimensional geometric median minimizing~\eqref{eq:geom_med} to the minimizer over $\rmL \in \GDd$ of the function
\be
\label{eq:geom_subs}
\sum_{i=1}^N \|\bx_i - \bP_{\rmL}\bx_i\| \equiv \sum_{i=1}^N \|\bP_{\rmL^\perp}\bx_i\|\,,
\ee
where $\rmL^\perp$ is the orthogonal complement of $\rmL$ and $\bP_{\rmL}$ and $\bP_{\rmL^\perp}$ are the orthogonal projections on $\rmL$ and $\rmL^\perp$ respectively \citep[see, e.g.,][]{Ding+06, lp_recovery_part1_11}.
However, a single outlier with arbitrarily large magnitude will enforce the minimizer of~\eqref{eq:geom_subs} to contain it.

The first obstacle can be resolved by applying a convex relaxation of the minimization of~\eqref{eq:geom_subs} so that subspaces are mapped into a convex set of matrices (the objective function may be adapted respectively).
Indeed, the subspace recovery proposed by \citet{Xu2010} can be interpreted in this way.
Their objective function has one component which is similar to~\eqref{eq:geom_subs}, though translated to matrices.
They avoid the second obstacle by introducing a second component, which penalizes inliers of large magnitude (so that outliers of large magnitude may not be easily identified as inliers).
However, the combination of the two components involves a parameter that needs to be carefully estimated.

Here, we suggest a different convex relaxation that does not introduce arbitrary parameters and its implementation is significantly faster. However, it introduces some restrictions on the distributions of inliers and outliers. Some of these restrictions have analogs in other works (see, e.g., \S\ref{sec:counterexample}), while others are unique to this framework (see \S\ref{sec:remedies} and the non-technical description of all of our restrictions in \S\ref{sec:this_work}).

\subsection{Previous Work}
\label{sec:previous}

Many algorithms (or pure estimators) have been proposed for robust subspace estimation or equivalently robust low rank approximation of matrices.
\citet{Maronna1976}, \citet[\S8]{huber_book}, \citet{Devlin1981}, \citet{Davies1987}, \citet{Xu1995}, \citet{Croux00principalcomponent} and \citet[\S6]{robust_stat_book2006} estimate a robust covariance matrix.
Some of these methods use M-estimators \citep[\S6]{robust_stat_book2006} and compute them via iteratively re-weighted least squares (IRLS) algorithms, which linearly converge \citep{Arslan2004}.
The convergence of algorithms based on other estimators or strategies is not as satisfying.
The objective functions of the MCD (Minimum Covariance Determinant) and S-estimators
converge \citep[\S6]{robust_stat_book2006}, but no convergence rates are specified.
Moreover, there are no guarantees for the actual convergence to the global optimum of these objective functions.
There is no good algorithm for the MVE (Minimum Volume Ellipsoid) or Stahel-Donoho
estimators \citep[\S6]{robust_stat_book2006}. Furthermore, convergence analysis is problematic for the online algorithm of \citet{Xu1995}.

\citet{Li_85}, \citet{Ammann1993}, \citet{Croux2007}, \citet{kwak08} and \citet[\S2]{robust_mccoy} find low-dimensional projections by ``Projection Pursuit'' (PP), now commonly referred to as PP-PCA  \citep[the initial proposal is due to Huber, see, e.g.,][p.~204 of first edition]{huber_book}.
The PP-PCA procedure is based on the observation that PCA maximizes the projective variance and can be implemented incrementally by computing the residual principal component or vector each time. Consequently, PP-PCA replaces this variance by a more robust function in this incremental implementation.
Most PP-based methods are based on
non-convex optimization and consequently lack satisfying guarantees.
In particular, \citet{Croux2007} do not analyze convergence of their non-convex PP-PCA and \citet{kwak08} only
establishes convergence to a local maximum.
\citet[\S2]{robust_mccoy} suggest a convex relaxation for PP-PCA.
However, they do not guarantee that the output of their algorithm
coincides with the exact maximizer of their energy (though they show that the energies of the two are sufficiently close).
\citet{Ammann1993} applies a minimization on the sphere, which is clearly not convex.
It iteratively tries to locate vectors spanning the orthogonal complement of the underlying subspace, that is, $D-d$ vectors for a subspace in $\GDd$.
We remark that our method also suggests an optimization revealing the orthogonal complement, but it requires a single convex optimization, which is completely different from the method of \citet{Ammann1993}.

\citet{Torre2001,Torre03aframework}, \citet{Brubaker2009} and \citet{Xu2010_highdimensional} remove possible outliers, followed by estimation of the underlying subspace by PCA.
These methods are highly non-convex. Nevertheless, \citet{Xu2010_highdimensional} provide a probabilistic analysis for their near
recovery of the underlying subspace.

The non-convex minimization of~\eqref{eq:geom_subs} as a robust
alternative for principal component analysis was suggested earlier by various authors for hyperplane modeling \citep{osborne_watson85, spath_watson1987, Nyquist_l1_88, Bargiela_Hartley93}, surface modeling \citep{watson2001orth_l1, Watson02}, subspace modeling \citep{Ding+06}
and multiple subspaces modeling \citep{MKF_workshop09}.
This minimization also appeared in a pure geometric-analytic context of general surface modeling without outliers \citep{DS91}.
\citet{lp_recovery_part1_11, lp_recovery_part2_11} have shown that this minimization
can be robust to outliers under some conditions on the sampling of the data.

\citet{ke_kanade03} tried to minimize (over all low-rank approximations) the element-wise $\ell_1$ norm of the difference of a given matrix and its low-rank approximation. \citet{Chandrasekaran_Sanghavi_Parrilo_Willsky_2009} and \citet{candes_wright_robust_pca09} have proposed to minimize a linear combination of such an $\ell_1$ norm and the nuclear norm of the low-rank approximation in order to find the optimal low-rank estimator.
\citet{candes_wright_robust_pca09} considered the setting where uniformly sampled elements of the low-rank matrix are corrupted, which does not apply to our outlier model (where only some of the rows are totally corrupted). \citet{Chandrasekaran_Sanghavi_Parrilo_Willsky_2009} consider a general setting, though their underlying condition is too restrictive; weaker condition was suggested by \citet{Hsu_rpca_11}, though it is still not sufficiently general.
Nevertheless, \citet{Chandrasekaran_Sanghavi_Parrilo_Willsky_2009} and \citet{candes_wright_robust_pca09} are ground-breaking to the whole area, since they provide rigorous analysis of exact low-rank recovery with unspecified rank.

\citet{Xu2010} and \citet{robust_mccoy} have suggested a strategy analogous to \citet{Chandrasekaran_Sanghavi_Parrilo_Willsky_2009} and \citet{candes_wright_robust_pca09} to solve the outlier problem. They divide the matrix $\bX$ whose rows are the data points as follows: $\bX=\bL+\bO$, where $\bL$ is low-rank and $\bO$ represents outliers (so that only some of its rows are non-zero). They minimize $\|\bL\|_* + \lambda \|\bO\|_{(2,1)}$, where $\|\cdot\|_*$ and $\|\cdot\|_{(2,1)}$ denote the nuclear norm and sum of $\ell_2$ norms of rows respectively and $\lambda$ is a parameter that needs to be carefully chosen. We note that the term $\|\bO\|_{(2,1)}$ is analogous to~\eqref{eq:geom_subs}.
\citet{Xu2012} have established an impressive theory showing that under some incoherency conditions,
a bound on the fraction of outliers and correct choice of the parameter $\lambda$, they can exactly recover the low-rank approximation.
\citet{Hsu_rpca_11}  and \citet{Agarwal_Negahban_Wainwright_2011} improved error bounds for this estimator as well as for the ones of \citet{Chandrasekaran_Sanghavi_Parrilo_Willsky_2009} and \citet{candes_wright_robust_pca09}.

In practice, the implementations by \citet{Chandrasekaran_Sanghavi_Parrilo_Willsky_2009}, \citet{candes_wright_robust_pca09}, \citet{Xu2010} and
\citet{robust_mccoy} use the iterative procedure described by \citet{rpca_code09}. The difference between the objective functions of the minimizer and its estimator obtained at the $k$th iteration is of order $O(k^{-2})$ \citep[Theorem 2.1]{rpca_code09}.
On the other hand, for our algorithm the convergence rate is of order $O(\exp(-ck))$ for some constant $c$ (i.e., it r-linearly converges).
This rate is the order of the Frobenius norm of the difference between the minimizer sought by our algorithm (formulated in~\eqref{eq:symmetric} below) and its estimator obtained at the $k$th iteration (it is also the order of the difference of the regularized objective functions of these two matrices). Recently, \citet{Agarwal_Negahban_Wainwright_2011b} showed that projected gradient descent algorithms for these estimators obtain linear convergence rates, though with an additional statistical error.

Our numerical algorithm can be categorized as IRLS.
\citet{Weiszfeld1937} used a procedure similar to ours to find the geometric median. \citet{Lawson61} later used it to solve uniform approximation problems by the limits of
weighted $\ell_p$-norm solutions. This procedure was generalized to various minimization problems, in particular, it is native to M-estimators \citep{huber_book, robust_stat_book2006}, and its linear convergence was proved for special instances \citep[see, e.g.,][]{Cline1972,Voss80,Chan99}.
Recently, IRLS algorithms were also applied to sparse
recovery and matrix completion \citep{Daubechies_iterativelyreweighted,Fornasier_low-rankmatrix}.

\subsection{This Work}
\label{sec:this_work}
We suggest another convex relaxation of the minimization of~\eqref{eq:geom_subs}. We note that the original minimization is over all subspaces $\rmL$ or equivalently all orthogonal projectors $\bP\equiv \bP_{\rmL^\perp}$. We can identify $\bP$ with a $D\times D$ matrix satisfying $\bP^2 = \bP$ and $\bP^T = \bP$ (where $\cdot^T$ denotes the transpose). Since the latter set is not convex, we relax it to include all symmetric matrices, but avoid singularities by enforcing unit trace. That is, we minimize over the set:
\be
\label{eq:def_H}
\bbH:=\{\bQ\in\reals^{D\times D}: \bQ=\bQ^T,\tr(\bQ)=1\}
\ee
as follows
\begin{equation}\label{eq:symmetric}
\hat{\bQ}=\argmin_{\bQ\in\bbH}F(\bQ),\,\,\text{where}\,F(\bQ):=\sum_{i=1}^{N}\|\bQ\bx_i\|.
\end{equation}

For the noiseless case (i.e., inliers lie exactly on $\rmL^*$), we estimate the subspace $\rmL^*$ by
\be
\label{eq:gms_no_noise}
\hat{\rmL}:=\ker(\hat{\bQ}).
\ee
If the intrinsic dimension $d$ is known (or can be estimate from the data), we estimate the subspace by the span of the bottom $d$ eigenvectors
of $\hat{\bQ}$ (or equivalently, the top $d$ eigenvectors of $-\hat{\bQ}$). This procedure is robust to sufficiently small levels of noise. We refer to it as the Geometric Median Subspace (GMS) algorithm and summarize it in Algorithm~\ref{alg:GMS}. We elaborate on this scheme throughout the paper,
\begin{algorithm}[htbp]
\caption{The Geometric Median Subspace Algorithm} \label{alg:GMS}
\begin{algorithmic}
\REQUIRE $\sX=\{\bx_i\}_{i=1}^N \subseteq
\mathbb{R}^{D}$: data, $d$: dimension of $\rmL^*$, an algorithm for minimizing~\eqref{eq:symmetric}\\
\ENSURE $\hat{\rmL}$: a $d$-dimensional linear subspace in $\reals^D$.\\
\textbf{Steps}:
\STATE
     $\bullet$ $\{\bv_i\}_{i=1}^d=$ \ the bottom $d$ eigenvectors of $\hat{\bQ}$ (see \eqref{eq:symmetric}) \\
    $\bullet$  $\hat{\rmL} = \Sp(\{\bv_i\}_{i=1}^d)$
\end{algorithmic}
\end{algorithm}

We remark that $\hat{\bQ}$ is semi-definite positive (we verify this later in Lemma~\ref{lemma:nonnegative}).
We can thus restrict $\bbH$ to contain only semi-definite positive matrices and thus make it even closer to a set of orthogonal projectors.
Theoretically, it makes sense to require that the trace of the matrices in $\bbH$ is $D-d$ (since they are relaxed versions of projectors onto the orthogonal complement of a $d$-dimensional subspace). However, scaling of the trace in~\eqref{eq:def_H} results in scaling the minimizer of \eqref{eq:symmetric} by a constant, which does not effect the subspace recovery procedure.

We note that~\eqref{eq:symmetric} is an M-estimator with residuals $r_i = \|\bQ\bx_i\|$, $1 \leq i \leq N$, and $\rho(x)=|x|$.
Unlike~\eqref{eq:geom_subs}, which can also be seen as a formal M-estimator, the estimator $\hat{\bQ}$ is unique under a weak condition that we will state later.

We are unaware of similar formulations for the problem of robust PCA.
Nevertheless, the Low-Rank Representation (LRR) framework of \citet{lrr_short, lrr_long} for modeling data by multiple subspaces (and not a single subspace as in here) is formally similar. LRR tries to assign to a data matrix $\bX$, which is viewed as a dictionary of $N$ column vectors in $\reals^D$, dictionary coefficients $\bZ$ by minimizing $\lambda \|\bZ\|_*+ \|(\bX(\bI-\bZ))^T\|_{(2,1)}$ over all $\bZ\in\reals^{N\times N}$, where $\lambda$ is a free parameter.
Our formulation can be obtained by their formulation with $\lambda=0$, $\bQ= (\bI-\bZ)^T$ and the additional constraint $\tr(\bZ)=D-1$ (which is equivalent with the scaling $\tr(\bQ)=1$), where $\{\bx_i\}_{i=1}^N$ are the row vectors of $\bX$ (and not the column vectors that represent the original data points).
In fact, our work provides some intuition for LRR as robust recovery of the low rank row space of the data matrix and its use (via $\bZ$) in partitioning the column space into multiple subspaces. We also remark that a trace 1 constraint is quite natural in convex relaxation problems and was applied, for example, in the convex relaxation of sparse PCA \citep{DAspremont07}, though the optimization problem there is completely different.

Our formulation is rather simple and intuitive, but results in the following fundamental contributions to robust recovery of subspaces:
\begin{enumerate}
\item We prove that our proposed minimization can achieve exact recovery under some assumptions on the underlying data (which we clarify below)
and without introducing an additional parameter.
\item We propose a fast iterative algorithm for achieving this minimization and prove its linear convergence.
\item We demonstrate the state-of-the-art accuracy and speed of our algorithm when compared with other methods on both synthetic and real data sets.
\item We establish the robustness of our method to noise and to a common regularization of IRLS algorithms.
\item We explain how to incorporate knowledge of the intrinsic dimension and also how to estimate it empirically.
\item We show that when replacing the sum of norms in \eqref{eq:symmetric} by the sum of squares of norms, then the modified minimizer $\hat{\bQ}$ is a scaled version of the
empirical inverse covariance. The subspace spanned by the bottom $d$ eigenvectors is clearly the $d$-dimensional subspace obtained by PCA. The original minimizer of \eqref{eq:symmetric}
can thus be interpreted as a robust version of the inverse covariance matrix.
\item We show that previous and well-known M-estimators \citep{Maronna1976, huber_book, robust_stat_book2006} do not solve the subspace recovery problem under a common assumption.
\end{enumerate}

\subsection{Exact Recovery and Conditions for Exact Recovery by GMS}
In order to study the robustness to outliers of our estimator for the underlying subspace,
we formulate the exact subspace recovery problem (see also \citealt{Xu2012}).
This problem assumes a fixed $d$-dimensional linear subspace $\rmL^*$,
inliers sampled from $\rmL^*$ and outliers sampled from its complement; it asks to recover $\rmL^*$ as well as identify
correctly inliers and outliers.

In the case of point estimators, like the geometric median minimizing~\eqref{eq:geom_med},
robustness is commonly measured by the breakdown point of the estimator \citep{huber_book, robust_stat_book2006}.
Roughly speaking, the breakdown point  measures
the proportion of arbitrarily large observations (that is, the proportion of ``outliers'') an estimator can handle before giving an
arbitrarily large result.

In the case of estimating subspaces, we cannot directly extend this definition, since the set of subspaces, that is, the Grassmannian (or unions of it),
is compact, so we cannot talk about ``an arbitrarily large result'', that is, a subspace with arbitrarily large distance from all other subspaces.
Furthermore, given an arbitrarily large data point, we can always form a subspace containing it; that is, this point is not
arbitrarily large with respect to this subspace.
Instead, we identify the outliers as the ones in the compliment of $\rmL^*$ and we are interested in the largest
fraction of outliers (or smallest fraction of inliers per outliers) allowing exact recovery of $\rmL^*$. Whenever an estimator can exactly recover
a subspace under a given sampling scenario we view it as robust and measure its effectiveness by the largest fraction of outliers it can tolerate.
However, when an estimator cannot exactly recover a subspace, one needs to bound from below the distance between the recovered subspace and the underlying
subspace of the model. Alternatively, one would need to point out at interesting scenarios where exact recovery cannot even occur
in the limit when the number of points approaches infinity.
We are unaware of other notions of robustness of subspace estimation (but of robustness of covariance estimation, which does not apply here; see, for example,
\S6.2.1 of \citealt{robust_stat_book2006}).

In order to guarantee exact recovery of our estimator we basically require three kinds of restrictions on the underlying data, which we explain here on a non-technical level
(technical discussion appears in \S\ref{sec:theory_minimization}).
First of all, the inliers need to permeate through the whole underlying subspace $\rmL^*$, in particular,
they cannot concentrate on a lower dimensional subspace of $\rmL^*$.
Second of all, outliers need to permeate throughout the whole complement of $\rmL^*$.
This assumption is rather restrictive and its violation is a failure mode of the algorithm.
We thus show that this failure mode does not occur when the knowledge of $d$ is used appropriately.
We also suggest some practical methods to avoid this failure mode when $d$ is unknown (see \S\ref{sec:theory_no_dim}).
Third of all, the ``magnitude' of outliers needs to be restricted.
We may initially scale all points to the unit sphere in order to avoid extremely large outliers.
However, we still need to avoid outliers concentrating along lines,
which may have an equivalent effect of a single arbitrarily large outlier.
Figure~\ref{fig:counter} (which appears later in \S\ref{sec:theory_minimization}) demonstrates cases where these assumptions are not satisfied.

The failure mode discussed above occurs in particular when the number of outliers is rather small and the dimension $d$ is unknown.
While we suggest some practical methods to avoid it (see \S\ref{sec:theory_no_dim}),
we also note that there are many modern applications with high percentages of outliers, where this failure mode may not occur.
In particular, computer vision data often contain high percentages of outliers \citep{Stewart99robustparameter, Chin+_hypotheis2012}.
However, such data usually involve multiple geometric models, in particular, multiple underlying linear subspaces.
We believe that the robust subspace modeling is still relevant to these kinds of data.
First of all, robust single subspace strategies can be well-integrated into common schemes of modeling data
by multiple subspaces. For example, the $K$-flats algorithm is
based on repetitive clustering and single subspace modeling per
cluster \citep{Tipping99mixtures, Bradley00kplanes, Tseng00nearest, Ho03, MKF_workshop09, LBF_journal12}
and the LBF and SLBF algorithms use local subspace modeling \citep{LBF_cvpr10, LBF_journal12}.
Second of all, some of the important preprocessing tasks in computer vision require single subspace modeling.
For example, in face recognition, a preprocessing step requires efficient subspace modeling of images of the same face
under different illuminating conditions \citep{Basri03, Basri_nearest_subspace11}.
There are also problems in computer vision with more complicated geometric models and large percentage of corruption,
where our strategies can be carefully adapted.
One important example is the synchronization problem, which finds an important application in Cryo-EM.
The goal of this problem is to recover rotation matrices $R_1$, $\ldots$, $R_N \in S0(3)$ from noisy
and mostly corrupted measurements of $R_i^{-1}R_j$ for some values of $1 \leq i,j \leq N$. \citet{wang_singer_2013} adapted ideas of both this work
and \cite{LMTZ2012} to justify and implement a robust solution for the synchronization problem.

\subsection{Recent Subsequent Work}

In the case where $d$ is known, \citet{LMTZ2012} followed this work and suggested a tight
convex relaxation of the minimization of \eqref{eq:geom_subs_l2} over all projectors $\bP_{\rmL^\perp}$ of rank $d$.
Their optimizer, which they refer to as the REAPER (of the needle-in-haystack problem)
minimize the same function $F(\bQ)$ (see \eqref{eq:symmetric}) over the set
\[
\bbH'=\{\bQ\in\reals^{D\times D}: \bQ=\bQ^T,\tr(\bQ)=1,\|\bQ\|\leq \frac{1}{D-d}\}.
\]
They estimate the underlying subspace by the bottom $d$ eigenvectors of the REAPER.
The new constraints in $\bbH'$
result in more elegant conditions for exact recovery and tighter probabilistic
theory (due to the tighter relaxation).
Since $d$ is known the failure mode of GMS mentioned above is avoided.
Their REAPER algorithm for computing the REAPER is based on the IRLS procedure of this paper with additional constraints, which complicate its analysis.
The algorithmic and theoretical developments of \citet{LMTZ2012} are based on the ones here.

While the REAPER framework applies a tighter relaxation, the GMS framework still has several advantages over the REAPER framework.
First of all, in various practical situations the dimension of the data is unknown and thus REAPER is inapplicable.
On the other hand, GMS can be used for dimension estimation, as we demonstrate in \S\ref{sec:dim}.
Second of all, the GMS algorithm is faster than REAPER
(the REAPER requires additional eigenvalue decomposition of a $D\times D$ matrix at each iteration of the IRLS algorithm).
Furthermore, we present here a complete theory for the linear convergence of the GMS algorithm,
where the convergence theory for the REAPER algorithm is currently incomplete.
Third of all, when the failure mode mentioned above is avoided, the empirical performances of REAPER and GMS are usually comparable (while GMS is faster).
At last, GMS and REAPER have different objectives with different consequences. REAPER aims to find a projector
onto the underlying subspace. On the other hand, GMS aims to find a ``generalized inverse covariance'' (see \S\ref{sec:l2}) and is
formally similar to other M-estimators (see \S\ref{sec:m_estimator_whole} and \S\ref{sec:m_estimator_tyler}).
Therefore, the eigenvalues and eigenvectors of the GMS estimator (i.e., the ``generalized inverse covariance'') can be interpreted as
robust eigenvalues and eigenvectors of the empirical covariance (see \S\ref{sec:dim} and \S\ref{sec:info_evd}).	

\subsection{Structure of This Paper}

In \S\ref{sec:theory_minimization} we establish exact and near subspace recovery via the GMS algorithm.
We also carefully explain the common obstacles for robust subspace recovery and the way they are handled by previous
rigorous solutions \citep{candes_wright_robust_pca09,Chandrasekaran_Sanghavi_Parrilo_Willsky_2009,Xu2012} as well as our solution.
Section \ref{sec:significance} aims to interpret our M-estimator in two different ways. First of all, it shows a formal similarity to
a well-known class of M-estimators \citep{Maronna1976, huber_book, robust_stat_book2006}, though clarifies the difference. Those estimators aims to robustly estimate the sample covariance. However, we show there that unlike our M-estimator, they cannot solve the subspace recovery problem (under a common assumption).
Second of all, it shows that non-robust adaptation of our M-estimator provides both direct estimation of the inverse covariance matrix as well as convex minimization equivalent to the non-convex
total least squares (this part requires full rank data and thus a possible initial dimensionality reduction but without any loss of information).
We thus interpret \eqref{eq:symmetric} as a robust estimation of the inverse covariance.
In \S\ref{sec:algorithm_whole} we propose an IRLS algorithm for minimizing~\eqref{eq:symmetric} and establish its linear convergence.
Section \ref{sec:prac_solutions} discusses practical versions of the GMS procedure that allow more general distributions than the ones guaranteed by the theory. One of these versions, the Extended GMS (EGMS) even provides robust alternative to principal components.
In \S\ref{sec:numerics} we demonstrate the state-of-the-art accuracy and speed of our algorithm when compared with other methods on both synthetic and real data sets and also numerically clarify some earlier claims. Section \ref{sec:proofs} provides all details of the proofs and \S\ref{sec:discussion} concludes with brief discussion.

\section{Exact and Near Subspace Recovery by GMS}
\label{sec:theory_minimization}

We establish exact and near subspace recovery by the GMS algorithm. In \S\ref{subsec:prob_formulation} we formulate the problems of exact and near subspace recovery. In \S\ref{sec:counterexample} we describe common obstacles for solving these problems and how they were handled in previous works; in \S\ref{sec:remedies} we formulate some conditions that the data may satisfy; whereas in \S\ref{subsec:recovery} we claim that these conditions are sufficient to avoid the former obstacles, that is, they guarantee exact recovery (see Theorem~\ref{thm:recovery}); We also propose weaker conditions for exact recovery and demonstrate their near-tightness in \S\ref{sec:verifiability}. Section~\ref{sec:unique} describes a simple general condition for uniqueness of GMS (beyond the setting of exact recovery).
Section~\ref{sec:prob} establishes (with some specified limitations) unique exact recovery with high probability under basic probabilistic models (see Theorems~\ref{thm:prob} and~\ref{thm:prob_b}); it also covers cases with asymmetric outliers. At last, \S\ref{sec:noisy} and \S\ref{sec:regularize} establish results for near recovery under noise and under regularization respectively.

%NEED TEXT HERE

\subsection{Problem Formulation}
\label{subsec:prob_formulation}
Let us repeat the formulation of the exact subspace recovery problem, which we motivated in \S\ref{sec:this_work} as a robust measure for the performance of our estimator.
We assume a linear subspace $\rmL^* \in \GDd$ and a data set $\sX=\{\bx_i\}_{i=1}^N$,
which contains inliers sampled from $\rmL^*$ and outliers sampled from $\reals^D\setminus\rmL^*$.
Given the data set $\sX$ and no other information, the objective of the exact subspace recovery problem is to exactly recover the underlying subspace $\rmL^*$.

In order to make the problem well-defined, one needs to assume some conditions on the sampled data set, which may vary with the proposed solution.
We emphasize that this is a formal mathematical problem, which excludes some ambiguous scenarios and allows us to determine admissible distributions of inliers and outliers.

In the noisy case (where inliers do not lie on $\rmL^*$, but perturbed by noise), we
ask about near subspace recovery, that is, recovery up to an error depending on the underlying noise level.
We argue below that in this case additional information on the model is needed.
Here we assume the knowledge of $d$, though under some assumptions we can estimate $d$ from the data (as we demonstrate later).
We remark that exact asymptotic recovery under some conditions on the noise distribution is way more complicated
and is discussed in another work \citep{Coudron_Lerman2012}.

\subsection{Common Difficulties with Subspace Recovery}\label{sec:counterexample}
We introduce here three typical enemies of subspace recovery and exemplify them in Figure~\ref{fig:counter}. We also explain how they are handled by the previous convex solutions for exact recovery of subspaces as well as low-rank matrices \citep{Chandrasekaran_Sanghavi_Parrilo_Willsky_2009, candes_wright_robust_pca09, Xu2012}.

A type 1 enemy occurs when the inliers are mainly sampled from a subspace $\rmL' \subset \rmL^*$. In this case, it seems impossible to recover $\rmL^*$. We would expect a good algorithm to recover $\rmL'$ (instead of $\rmL^*$) or a subspace containing it with slightly higher dimension
(see for example Figure~\ref{fig:counterexample1}).
\citet{Chandrasekaran_Sanghavi_Parrilo_Willsky_2009}, \citet{candes_wright_robust_pca09} and \citet{Xu2012} have addressed this issue by requiring incoherence conditions for the inliers.
For example, if $m$ and $N-m$ points are sampled from $L'$ and $L^* \setminus L'$ respectively, then the incoherency condition  of \citet{Xu2012} requires that $\mu \geq N/(\dim(\rmL^*) \cdot (N-m))$, where $\mu$ is their incoherency parameter. That is, their theory holds only when the fraction of points sampled from $L^* \setminus L'$ is sufficiently large.

\begin{figure}
\begin{center}
\subfigure[ \ Example of a type 1 enemy: $\rmL^*$ is a plane represented by a rectangle, ``inliers'' (in
$\rmL^*$) are colored blue and ``outliers'' (in $\reals^3 \setminus L^*$) red.
Most inliers lie on a line inside $\rmL^*$. It seems unlikely to distinguish between inliers, which are not on ``the main line'', and the outliers. It is thus likely to recover the main line instead of $\rmL^*$.]
{
\includegraphics[width=.45\textwidth,height=.3\textwidth]{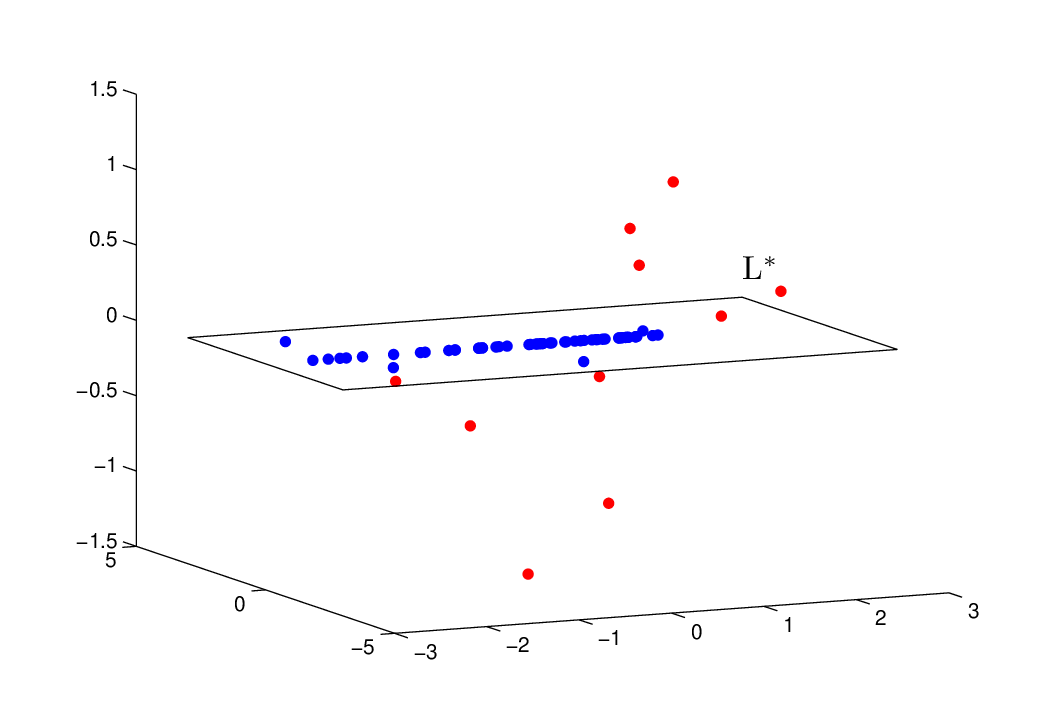}\label{fig:counterexample1}
}
  \hspace{5mm}
\subfigure[ \ Example of a type 2 enemy: $\rmL^*$ is a line represented by a black line segment, ``inliers'' (in
$\rmL^*$) are colored blue and ``outliers'' (in $\reals^3 \setminus L^*$) red.
All outliers but one lie within a plane containing $\rmL^*$, which is represented by a dashed rectangle.
There seems to be stronger distinction between the points on this plane and the isolated outlier than the original inliers and outliers.
Therefore, an exact recovery algorithm may output this plane instead of $\rmL^*$.]
{
\includegraphics[width=.45\textwidth,height=.3\textwidth]{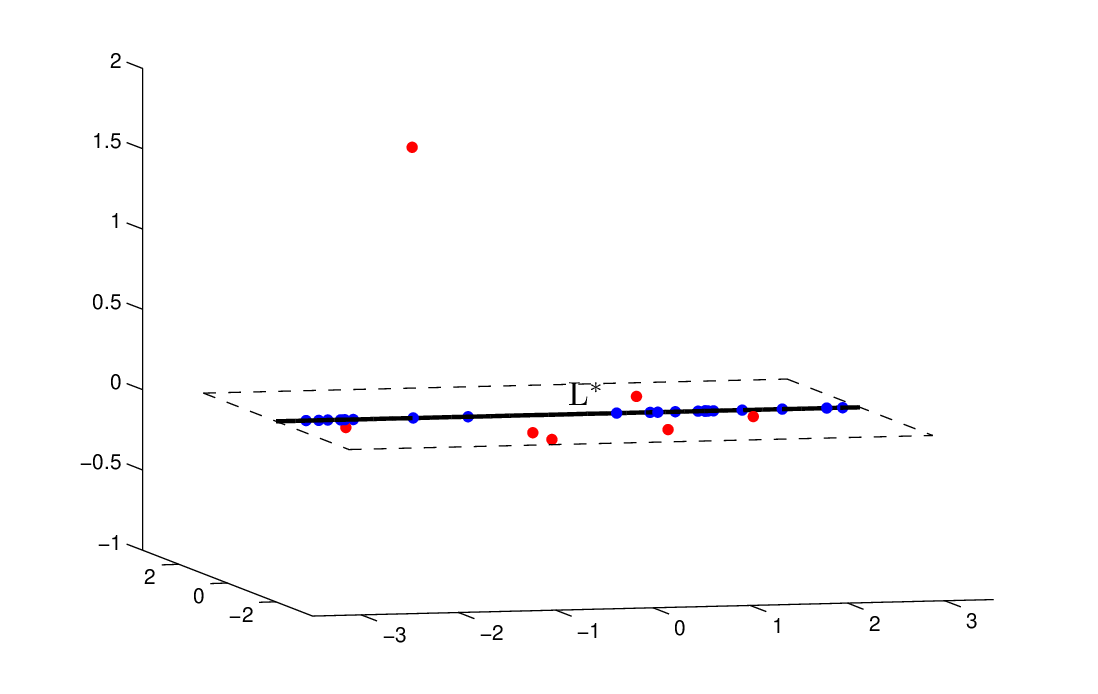}\label{fig:counterexample2}
} \\
\subfigure[ \ Example 1 of a type 3 enemy: The inliers (in blue) lie on the line $\rmL^*$ and there is a single outlier (in red) with relatively large magnitude. An exact recovery algorithm can output the line $\tilde{\rmL}$ (determined by the outlier) instead of $\rmL^*$. If the data is normalized to the unit circle, then any reasonable robust subspace recovery algorithm can still recover $\rmL^*$.]
{
\includegraphics[width=.45\textwidth,height=.3\textwidth]{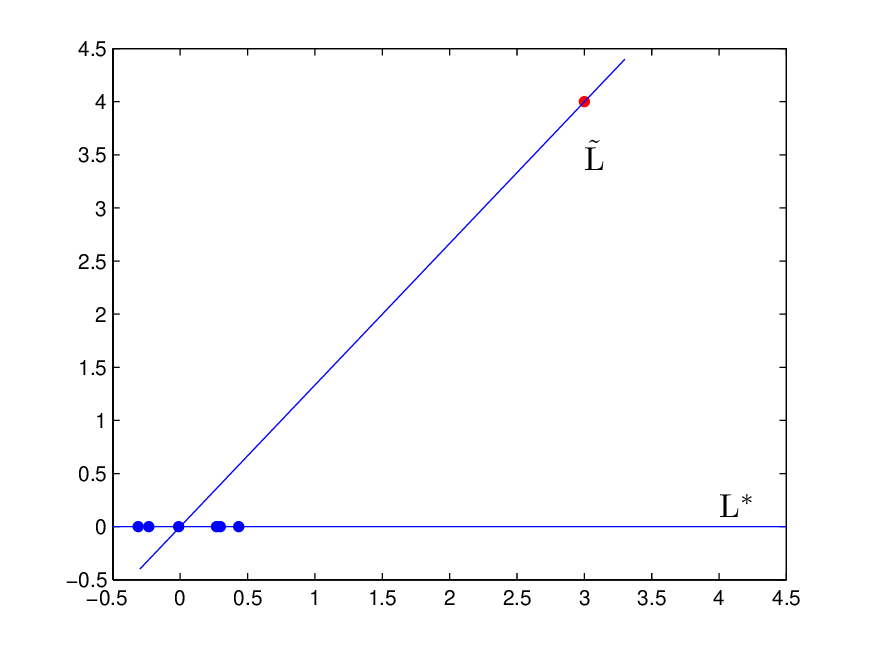}\label{fig:counterexample3}
}
 \hspace{5mm}
\subfigure[\ Example 2 of a type 3 enemy: Points are normalized to lie on the unit circle, inliers (in blue) lie around the line $\rmL^*$
and outliers (in red) concentrate around another line, $\tilde{\rmL}$. A subspace recovery algorithm can output $\tilde{\rmL}$
instead of $\rmL^*$.]
{
\includegraphics[width=.45\textwidth,height=.3\textwidth]{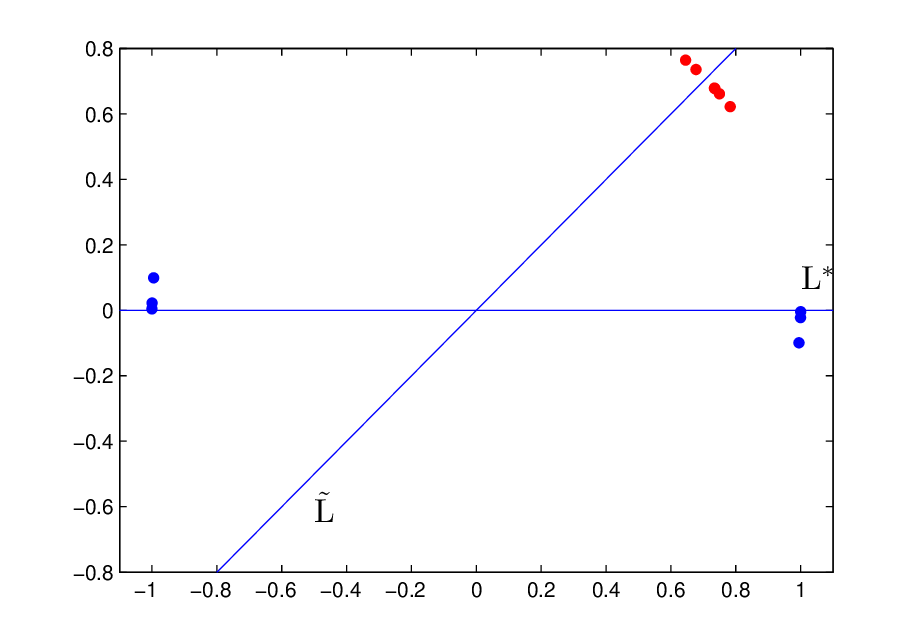}\label{fig:counterexample4}
}
\caption{Enemies of the mathematical formulation of exact subspace recovery. \label{fig:counter}}
\end{center}
\end{figure}

A type 2 enemy occurs when the outliers are mainly sampled from a subspace
$\tilde{\rmL}$ such that $\dim(\tilde{\rmL}\oplus \rmL^* )<D$. In this case $\rmL^*\oplus\tilde{\rmL}$ can be mistakenly identified as the low-rank subspace (see for example Figure~\ref{fig:counterexample2}). This is a main issue when the intrinsic dimension is unknown; if on the other hand the intrinsic dimension is known, then one can often overcome this enemy.
\citet{candes_wright_robust_pca09} handle it by assuming that the distribution of corrupted elements is uniform. \citet{Chandrasekaran_Sanghavi_Parrilo_Willsky_2009} address it by restricting their parameter $\mu$  (see their main condition, which is used in  Theorem~2 of their work, and their definition of $\mu$ in (1.2) of their work) and consequently limit the values of the mixture parameter (denoted here by $\lambda$).
On the other hand, \citet{Xu2012} use the true percentage of outliers to infer the right choice of the mixture parameter $\lambda$.
That is, they practically invoke model selection (for estimating this percentage) in order to reject $\tilde{\rmL}\oplus \rmL^*$
and choose the true model, which is $\rmL^*$.

A type 3 enemy occurs due to large magnitudes of outliers.
For example, a single outlier with arbitrarily large magnitude will be contained in the minimizer of~\eqref{eq:geom_subs}, which will thus be different than the underlying subspace (see for example Figure~\ref{fig:counterexample3}). Also, many outliers with not-so-small magnitudes that lie around a fixed line may have the effect of a single large outlier (see for example Figure~\ref{fig:counterexample4}).
This enemy is avoided by \citet{Chandrasekaran_Sanghavi_Parrilo_Willsky_2009}, \cite{candes_wright_robust_pca09} and \cite{Xu2012} by the additional mixture component of nuclear norm, which penalizes the magnitude (or combined magnitude) of the supposed inliers (so that outliers of large magnitude may not be easily identified as inliers). It is interesting to note that if the rank is used instead of the nuclear norm (as sometimes advocated), then it will not resolve this issue.

Another issue for our mathematical problem of exact subspace recovery is whether the subspace obtained by a proposed algorithm is unique.
Many of the convex algorithms depend on convex $\ell_1$-type methods that may not be strictly convex. But it may still happen that in the setting of pure inliers and outliers and under some conditions avoiding the three types of enemies, the recovered subspace is unique (even though it may be obtained by several non-unique minimizers). This is indeed the case in \citet{Chandrasekaran_Sanghavi_Parrilo_Willsky_2009}, \citet{candes_wright_robust_pca09}, \citet{Xu2012} and our own work.
Nevertheless, uniqueness of our minimizer (and not the recovered subspace) is important
for analyzing the numerical algorithm approximating it and for perturbation analysis (e.g., when considering near recovery with noisy data). It is also helpful for practically verifying the conditions we will propose for exact recovery.
Uniqueness of the minimizer (and not just the subspace) is also important in \citet{Chandrasekaran_Sanghavi_Parrilo_Willsky_2009} and \citet{candes_wright_robust_pca09} and they thus established conditions for it.

At last, we comment that subspace recovery with unknown intrinsic dimension may require a model selection procedure (possibly implicitly).
That is, even though one can provide a theory for exact subspace recovery (under some conditions), which might be stable to perturbations, in practice, some form of model selection will be necessary in noisy cases.
For example, the impressive theories by \citet{Chandrasekaran_Sanghavi_Parrilo_Willsky_2009} and \citet{Xu2012} require the estimation of the mixture parameter $\lambda$.
\citet{Xu2012} propose such an estimate for $\lambda$, which is based on knowledge of the data set (e.g., the distribution of corruptions and the fraction of outliers). However, we noticed that in practice this  proposal did not work well (even for simple synthetic examples), partly due to the fact that the deduced conditions are only sufficient, not necessary and there is much room left for improvement. The theory by \citet{candes_wright_robust_pca09} specified a choice for $\lambda$ that is independent of the model parameters, but it applies only for the special case of uniform corruption without noise; moreover, they noticed that other values of $\lambda$ could achieve better results.

\subsection{Conditions for Handling the Three Enemies}
\label{sec:remedies}

We introduce additional assumptions on the data to address the three types of enemies.
We denote the sets of exact inliers and outliers by
$\sX_1$ and $\sX_0$ respectively, that is, $\sX_1=\sX\cap\rmL^*$ and $\sX_0=\sX\setminus\rmL^*$.
The following two conditions simultaneously address both type 1 and type 3 enemies:
\begin{equation}\label{eq:condition1}
\min_{\bQ\in\bbH,\bQ\bP_{\rmL^{*\perp}}=\b0}\sum_{\bx\in\sX_1}\|\bQ\bx\|> \sqrt{2}\,\min_{\bv\in\rmL^{*\perp},\|\bv\|=1}\sum_{\bx\in\sX_0}|\bv^T\bx|,
\end{equation}
\begin{equation}\label{eq:condition2}
\min_{\bQ\in\bbH,\bQ\bP_{\rmL^{*\perp}}=\b0}\sum_{\bx\in\sX_1}\|\bQ\bx\|> \sqrt{2}\,\max_{\bv\in\rmL^*,\|\bv\|=1}\sum_{\bx\in\sX_0}|\bv^T\bx|.
\end{equation}

A lower bound on the common LHS of both~\eqref{eq:condition1} and~\eqref{eq:condition2} is designed to  avoid type 1 enemies. %stopped here
This common LHS is a weak version of the  permeance statistics, which was defined in (3.1) of \citet{LMTZ2012} as follows:
$$
{\cal P}(\rmL^*):= \min_{\latop{\bu \in \rmL^*}{\|\bu\|=1}}\sum_{\bx \in \sX_1} |\bu^T\bx|.
$$
Similarly to the permeance statistics, it is zero if and only if all inliers are contained in a proper subspace of $\rmL^*$. Indeed, if all inliers lie in a subspace $\rmL' \subset \rmL^*$, then this common LHS is zero with the minimizer $\bQ=\bP_{\rmL'^\perp\cap \rmL^*}/\tr(\bP_{\rmL'^\perp\cap \rmL^*})$. Similarly, if it is zero, then $\bQ \bx = \b0$ for any $\bx \in \sX_1$ and for some $\bQ$ with kernel containing $\rmL^{*\perp}$. This is only possible when $\sX_1$ is contained in a proper subspace of $\rmL^{*}$. Similarly to the permeance statistics, if the inliers nicely permeate through $\rmL^*$, then this common LHS clearly obtain large values.

The upper bounds on the RHS's of~\eqref{eq:condition1} and~\eqref{eq:condition2} address two complementing type 3 enemies.
If $\sX_0$ contains few data points of large magnitude, which are orthogonal to $\rmL^*$, then the RHS of~\eqref{eq:condition1} may be too large and \eqref{eq:condition1} may not hold.
If on the other hand $\sX_0$ contains few data points with large magnitude and a small angle with $\rmL^*$, then the RHS of~\eqref{eq:condition2} will be large so that~\eqref{eq:condition2}
may not hold. Conditions~\eqref{eq:condition1} and~\eqref{eq:condition2} thus complete each other.

The RHS of condition~\eqref{eq:condition2} is similar to the linear structure statistics (for $\rmL^*$), which was defined in (3.3) of \citet{LMTZ2012}. The linear structure statistics uses an $\ell_2$ average of dot products instead of the $\ell_1$ average used here and was applied in this context to $\reals^D$ (instead of $\rmL^*$) in \citet{LMTZ2012}.
Similarly to the linear structure statistics, the RHS of~\eqref{eq:condition2} is large when  outliers either have large magnitude or they lie close to a line (so that their combined contribution is similar to an outlier with a very large magnitude as exemplified in Figure~\ref{fig:counterexample4}).
The RHS of condition~\eqref{eq:condition2} is a very weak analog of the linear structure statics of ${\rmL^*}^\perp$ since it uses a minimum instead of a maximum. There are some significant outliers within ${\rmL^*}^\perp$ that will not be avoided by requiring~\eqref{eq:condition2}.
For example, if the codimension of $\rmL^*$ is larger than 1 and there is a single outlier with an arbitrary large magnitude orthogonal to $\rmL^*$, then the RHS of~\eqref{eq:condition2} is zero.

The next condition avoids type 2 enemies and also significant outliers within ${\rmL^*}^\perp$
(i.e., type 3 enemies) that were not avoided by condition~\eqref{eq:condition2}. This condition requires that any minimizer of the following oracle problem
\begin{equation}\label{eq:symmetric1}
\hat{\bQ}_0:=\argmin_{\bQ\in\bbH,\bQ\bP_{\rmL^*}=\b0}F(\bQ)
\end{equation}
satisfies
\be
\label{eq:condition3}
\rank(\hat{\bQ}_0) = D-d.
\ee
We note that the requirement $\bQ\bP_{\rmL^{*}}=\b0$ is equivalent to the condition $\ker(\bQ) \supseteq {\rmL^*}$ and therefore the rank of the minimizer is at most $D-d$.
Enforcing the rank of the minimizer to be exactly $D-d$ restricts the
distribution of the projection of $\sX$ onto ${\rmL^*}^\perp$.
In particular, it avoids its concentration on lower dimensional subspaces and is thus suitable to avoid type 2 enemies.
Indeed, if all outliers are sampled from
$\tilde{\rmL}\subset\rmL^{*\perp}$, then
any $\bQ\in\bbH$ with $\ker(\bQ)\supset\tilde{\rmL}+\rmL^*$ satisfies
$F(\bQ)=0$ and therefore it is a minimizer of the oracle
problem~\eqref{eq:symmetric}, but it contradicts~\eqref{eq:condition3}.

We note that this condition also avoids some type 3 enemies, which were not handled by conditions~\eqref{eq:condition1} and~\eqref{eq:condition2}. For example, any $D-d-1$ outliers with large magnitude orthogonal to $\rmL^*$ will not be excluded by requiring~\eqref{eq:condition1} or~\eqref{eq:condition2}, but will be avoided by~\eqref{eq:condition3}.

This condition is restrictive though, especially in very high ambient dimensions. Indeed,
it does not hold when the number of outliers is smaller than $D-d$ (since then the outliers are sampled from some $\tilde{\rmL}$ with $\dim(\tilde{\rmL}\oplus\rmL^*)<D$).
We thus explain in \S\ref{sec:theory_know_dim} and \S\ref{sec:egms} how to avoid this condition when knowing the dimension. We also suggest in \S\ref{sec:theory_no_dim} some practical solutions to overcome the corresponding restrictive lower bound on the number of outliers when the dimension is unknown.

\begin{example}
\label{example1}
We demonstrate the violation of the conditions above for the examples depicted in Figure~\ref{fig:counter}.
The actual calculations rely on ideas explained in \S\ref{sec:verifiability}.

For the example in Figure~\ref{fig:counterexample1}, which represents a type 1 enemy, both conditions \eqref{eq:condition1}
and \eqref{eq:condition2} are violated.
Indeed, the common LHS of \eqref{eq:condition1} and \eqref{eq:condition2} is $5.69$, whereas the RHS of \eqref{eq:condition1} is $8.57$ and the RHS of \eqref{eq:condition2} is larger than $10.02$ (this lower bound is obtained by substituting $\bv=[0,1,0]$ in the RHS of \eqref{eq:condition2}; note
that $\bv$ is a unit vector in $\rmL^*$).

For the example in Figure~\ref{fig:counterexample2}, which represents a type 2 enemy,
condition~\eqref{eq:condition3} is violated. Indeed, we obtained numerically a solution $\hat{\bQ}_0$
with $\rank(\hat{\bQ}_0) = 1 \neq D-d=2$ (one can actually prove in this case that $\hat{\bQ}_0$
is the projector onto the orthogonal complement of the plane represented by the dashed rectangle).

For the example in Figure~\ref{fig:counterexample3}, which represents a type 3 enemy, both conditions \eqref{eq:condition1}
and \eqref{eq:condition2} are violated. Indeed, the common LHS of \eqref{eq:condition1} and \eqref{eq:condition2} is $1.56$ and the RHS's of \eqref{eq:condition1} and \eqref{eq:condition2} are $5.66$ and $4.24$ respectively.
However, if we normalize all points to lie on the unit circle, then this enemy can be overcome. Indeed, for the normalized data, the common LHS of \eqref{eq:condition1} and \eqref{eq:condition2} is $6$ and the RHS's of \eqref{eq:condition1} and \eqref{eq:condition2} are $1.13$ and $0.85$ respectively.

For the example in Figure~\ref{fig:counterexample4}, which also represents a type 3 enemy, both conditions \eqref{eq:condition1}
and \eqref{eq:condition2} are violated. Indeed, the LHS of \eqref{eq:condition1} and \eqref{eq:condition2} are $5.99$ and the RHS's of \eqref{eq:condition1} and \eqref{eq:condition2} are $6.91$ and $7.02$ respectively.
\end{example}

\subsection{Exact Recovery Under Combinatorial Conditions}
\label{subsec:recovery}

We show that the minimizer of~\eqref{eq:symmetric} solves the exact recovery problem under the above combinatorial conditions.
\begin{theorem}\label{thm:recovery}
Assume that $d,D \in \nats$, $d<D$, $\sX$ is a data set in $\reals^D$ and $\rmL^*\in\GDd$. If conditions
\eqref{eq:condition1}, \eqref{eq:condition2} and~\eqref{eq:condition3} hold (w.r.t.~$\sX$ and~$\rmL^*$),
then any minimizer of~\eqref{eq:symmetric}, $\hat{\bQ}$, recovers the subspace $\rmL^*$ in the following way:
$\ker(\hat{\bQ})= \rmL^*$. If only \eqref{eq:condition1} and \eqref{eq:condition2} hold, then $\ker(\hat{\bQ}) \supseteq \rmL^*$.
\end{theorem}

\subsubsection{Weaker Alternatives of Conditions \eqref{eq:condition1} and \eqref{eq:condition2}}\label{sec:verifiability}

It is sufficient to guarantee exact recovery by requiring~\eqref{eq:condition3} and that for an arbitrarily chosen solution of \eqref{eq:symmetric1}, $\hat{\bQ}_0$, the following two conditions are satisfied:
\begin{equation}\label{eq:condition1b}
\min_{\bQ\in\bbH,\bQ\bP_{\rmL^{*\perp}}=\b0}\sum_{\bx\in\sX_1}\|\bQ\bx\|> \sqrt{2}\,\left\|\sum_{\bx\in\sX_0}\hat{\bQ}_0\bx\bx^T\bP_{\rmL^{*\perp}}/\|\hat{\bQ}_0\bx\|\right\|
\end{equation}
and
\begin{equation}\label{eq:condition2b}
\min_{\bQ\in\bbH,\bQ\bP_{\rmL^{*\perp}}=\b0}\sum_{\bx\in\sX_1}\|\bQ\bx\|> \sqrt{2}\,\left\|\sum_{\bx\in\sX_0}\hat{\bQ}_0\bx\bx^T\bP_{\rmL^*}/\|\hat{\bQ}_0\bx\|\right\|.
\end{equation}
We note that condition~\eqref{eq:condition3} guarantees that $\hat{\bQ}_0\bx \neq \b0$ for all $\bx\in\sX_0$ and thus the RHS's of~\eqref{eq:condition1b} and \eqref{eq:condition2b} are well-defined. We prove this statement in~\eqref{sec:verifiability_proofs}.

We note that conditions \eqref{eq:condition1b} and \eqref{eq:condition2b} can be verified when $\sX_0$, $\sX_1$ and $\rmL^*$ are known (unlike \eqref{eq:condition1} and \eqref{eq:condition2}), where  $\hat{\bQ}_0$ can be found by Algorithm~\ref{alg:practical}.
Furthermore, \eqref{eq:condition1b} and \eqref{eq:condition2b} are weaker than \eqref{eq:condition1} and \eqref{eq:condition2}, though they are more technically involved and harder to motivate.

In order to demonstrate the near-tightness of \eqref{eq:condition1b} and \eqref{eq:condition2b}, we formulate the following necessary conditions for the recovery of $\rmL^*$ as $\ker(\hat{\bQ})$ (see the idea of their justification at the end of \S\ref{sec:verifiability_proofs}):
For an arbitrarily chosen  solution  of \eqref{eq:symmetric1}, $\hat{\bQ}_0$:
\be
\label{eq:nec_cond1}
\min_{\bQ\in\bbH,\bQ\bP_{\rmL^{*\perp}}=\b0}\sum_{\bx\in\sX_1}\|\bQ\bx\|
\geq \|\sum_{\bx\in\sX_1}\hat{\bQ}_0\bx\bx^T\bP_{\rmL^{*\perp}}/\|\hat{\bQ}_0\bx\|\|
\ee
and
\be
\label{eq:nec_cond2}
\sum_{\bx\in\sX_1}\|\bQ(\tilde{\bP}_{\rmL^*}\bx)\|\geq  \sum_{\bx\in\sX_0}\left\langle \bQ,\tilde{\bP}_{\rmL^{*\perp}}^T\hat{\bQ}_0\bx\bx^T\tilde{\bP}_{\rmL^{*}}/\|\hat{\bQ}_0\bx\|\right\rangle_F \,\,
\text{for any $\bQ\in\reals^{(D-d)\times d}$},
\ee
where for matrices $\bA$, $\bB \in \reals^{k\times l}$: $\langle \bA, \bB \rangle_F = \tr(A \, B^T)$ is the Frobenius dot product.
Indeed, conditions~\eqref{eq:nec_cond1} and~\eqref{eq:nec_cond2} are close to conditions \eqref{eq:condition1b} and \eqref{eq:condition2b}.
In particular, \eqref{eq:nec_cond1} and \eqref{eq:condition1b} are only different by the constant factor $\sqrt{2}$, that is,
\eqref{eq:condition1b} is practically tight.

\subsection{Uniqueness of the Minimizer}
\label{sec:unique}

We recall that Theorem~\ref{thm:recovery} implies that if \eqref{eq:condition1}, \eqref{eq:condition2} and~\eqref{eq:condition3} hold, then
$\ker(\hat{\bQ})$ is unique. Here we guarantee the uniqueness of $\hat{\bQ}$ (which is required in \S\ref{sec:verifiability}, \S\ref{sec:noisy}, \S\ref{sec:regularize} and \S\ref{sec:theory_alg}) independently of the exact subspace recovery problem.
\begin{theorem}\label{thm:convex}
If the following condition holds:
\begin{equation}\label{eq:convex_condition}\{\sX\cap\rmL_1\}\cup\{\sX\cap\rmL_2\}\neq\sX
\,\,\text{for all $(D-1)$-dimensional subspaces $\rmL_1,\rmL_2\subset\reals^D$,}
\end{equation}
then  $F(\bQ)$ is a strictly convex function on $\bbH$.
\end{theorem}

\subsection{Exact Recovery under Probabilistic Models}
\label{sec:prob}
We show that our conditions for exact recovery (or the main two of them) and our condition for uniqueness of the minimizer
$\hat{\bQ}$ hold with high probability under basic probabilistic models.
Such a probabilistic theory is cleaner when the outliers are sampled from a spherically symmetric distribution as we carefully
demonstrate in \S\ref{sec:sym_outliers_theory} (with two different models).
The problem is that when the outliers are spherically symmetric then various non-robust algorithms (such as PCA) can asymptotically approach exact
recovery and nearly recover the underlying subspace with sufficiently large sample.
We thus also show in \S\ref{sec:asym_outliers_theory} how the theory in \S\ref{sec:sym_outliers_theory} can be slightly modified to
establish exact recovery of the GMS algorithm in an asymmetric case, where PCA cannot even nearly recover the underlying subspace.

\subsubsection{Cases with Spherically Symmetric Distributions of Outliers}
\label{sec:sym_outliers_theory}

First we assume a more general probabilistic model.
We say that $\mu$ on $\reals^D$ is an Outliers-Inliers Mixture (OIM) measure (w.r.t.~the fixed subspace $\rmL^* \in \GDd$) if
$\mu=\alpha_0\mu_0+\alpha_1\mu_1$, where $\alpha_0$, $\alpha_1 > 0$, $\alpha_0+\alpha_1=1$, $\mu_1$ is a sub-Gaussian probability measure and $\mu_0$ is a sub-Gaussian probability measure on $\reals^D$ (representing outliers)
that can be decomposed to a product of two independent measures $\mu_0=\mu_{0,\rmL^*}\times \mu_{0,\rmL^{*\perp}}$ such that the supports of $\mu_{0,\rmL^*}$ and $\mu_{0,\rmL^{*\perp}}$ are $\rmL^{*}$ and $\rmL^{*\perp}$ respectively, and $\mu_{0,\rmL^{*\perp}}$ is spherically symmetric with respect to rotations within $\rmL^{*\perp}$.

To provide cleaner probabilistic estimates, we also invoke the needle-haystack model of \citet{LMTZ2012}. It assumes that both $\mu_0$ and $\mu_1$ are the Gaussian distributions: $\mu_0=N(\b0,\sigma_0^2\bI/D)$ and $\mu_1=N(\b0,\sigma_1^2\bP_{\rmL^*}\bP_{\rmL^*}^T/d)$ (the factors $1/D$ and $1/d$ normalize the magnitude of outliers and inliers respectively so that their norms are comparable).
While \citet{LMTZ2012} assume a fixed number of outliers and inliers independently sampled from $\mu_0$ and $\mu_1$ respectively, here we independently sample from the mixture measure $\mu=\alpha_0\mu_0+\alpha_1\mu_1$; we refer to $\mu$ as a needle-haystack mixture measure.

In order to prove exact recovery under any of these models, one needs to restrict the fraction of inliers per outliers (or equivalently, the ratio
$\alpha_1 / \alpha_0$). We refer to this ratio as SNR (signal to noise ratio) since we may view the inliers as the pure signal and the outliers as some sort of ``noise''. For the needle-haystack model we require the following SNR, which is similar to the one of \citet{LMTZ2012}:
\be\label{eq:prob_cond2a}
\frac{\alpha_1}{\alpha_0}> 4 \, \frac{\sigma_0}{\sigma_1} \, \frac{d}{\sqrt{(D-d)D}}.
\ee
We later explain how to get rid of the term ${\sigma_1}/{\sigma_0}$.
For the OIM model we assume the following more general condition:
\be\label{eq:prob_cond2}
\alpha_1\min_{\bQ\in\bbH,\bQ\bP_{\rmL^{*^\perp}}=\b0}\int\|\bQ\bx\|\di\mu_1(\bx)
>2\sqrt{2}\frac{\alpha_0}{D-d} \int\|\bP_{\rmL^{*\perp}}\bx\|\di\mu_{0}(\bx).
\ee
Under the needle-haystack model, this condition is a weaker version of~\eqref{eq:prob_cond2a}. That is,
\begin{lemma}
\label{lemma:prob_cond}
If $\mu$ is a needle-haystack mixture measure, then \eqref{eq:prob_cond2a} implies \eqref{eq:prob_cond2}.
\end{lemma}

For i.i.d.~samples from an OIM measure satisfying \eqref{eq:prob_cond2}, we can establish our modified conditions of unique exact recovery
(i.e., \eqref{eq:condition1b}, \eqref{eq:condition2b} and \eqref{eq:condition3})
with overwhelming probability in the following way (we also guarantee the uniqueness of the minimizer $\hat{\bQ}$).
\begin{theorem}\label{thm:prob}
If $\sX$ is an i.i.d.~sample from an OIM measure $\mu$ satisfying~\eqref{eq:prob_cond2},
then conditions \eqref{eq:condition1b}, \eqref{eq:condition2b}, and \eqref{eq:condition3} hold with probability $1-C\exp(-N/C)$,
where $C$ is a constant depending on $\mu$ and its parameters.
Moreover, \eqref{eq:convex_condition} holds with probability $1$ if there are at least $2D-1$ outliers
(i.e., the number of points in $\sX \setminus \rmL^*$ is at least $2D-1$).
\end{theorem}

Under the needle-haystack model, the SNR established by Theorem~\ref{thm:prob} is comparable to the best SNR
among other convex exact recovery algorithms (this is later clarified in Table~\ref{tab:SNR}).
However, the probabilistic estimate under which this SNR holds is rather loose and thus its underlying constant $C$ is not specified.
Indeed, the proof of Theorem~\ref{thm:prob} uses $\eps$-nets and union-bounds arguments, which are often not useful for deriving tight probabilistic estimates (see, e.g., \citealt[page 18]{mendelson03notes}).
One may thus view Theorem~\ref{thm:prob} as a near-asymptotic statement.

The statement of Theorem~\ref{thm:prob} does not contradict our previous observation that
the number of outliers should be larger than  at least $D-d$.
Indeed, the constant $C$ is sufficiently large so that the
corresponding probability is negative when the number of outliers is smaller than $D-d$.

In the next theorem we assume only a needle-haystack model and thus we can provide a stronger probabilistic estimate based on the concentration of measure phenomenon (our proof follows directly \citealp{LMTZ2012}). However, the SNR is worse than the one in Theorem~\ref{thm:prob} by a factor of order $\sqrt{D-d}$. This is because we are unable to estimate $\hat{\bQ}_0$ of \eqref{eq:symmetric1} by concentration of measure. Similarly,  in this theorem we do not estimate the probability of \eqref{eq:condition3} (which also involves $\hat{\bQ}_0$). Nevertheless, we observed in experiments that \eqref{eq:condition3} holds with high probability for $N_0=2(D-d)$ and the probability seems to go to $1$ as $N_0=2(D-d)$ and $D-d\rightarrow\infty$. Moreover, one of the algorithms proposed below (EGMS) does not require condition~\eqref{eq:condition3}.
\begin{theorem}\label{thm:prob_b}
If $\sX$ is an i.i.d.~sample of size $N$ from a needle-haystack mixture measure $\mu$ and if
\be\label{eq:prob_cond3}
\frac{\alpha_1}{\alpha_0} > \frac{\sigma_0}{\sigma_1} \, \frac{\sqrt{2/\pi}-1/4-1/10}{\sqrt{2/\pi}+1/4+1/10}\sqrt{\frac{d^2}{{D}}}
\ee
and
\be\label{eq:prob_cond3_assumption}
N>64\,\max(2d/\alpha_1,2d/\alpha_0,2(D-d)/\alpha_0),
\ee
then \eqref{eq:condition1} and \eqref{eq:condition2} hold with probability $1-e^{-\alpha_1^2N/2}-2e^{-\alpha_0^2N/2}-e^{-\alpha_1 N/800}
-e^{-\alpha_0 N/800}$.
\end{theorem}

In Table~\ref{tab:SNR} we present the theoretical asymptotic SNRs for exact recovery of some recent algorithms. We assume the needle-haystack model
with fixed $d$, $D$, $\alpha_0$, $\alpha_1$, $\sigma_0$ and $\sigma_1$
and  $N\rightarrow \infty$. Let us clarify these results.
We first remark that the pure SNR of the High-dimensional Robust PCA (HR-PCA) algorithm of \citet{Xu2010_highdimensional} approaches infinity
(see Remark 3 of \citealt{Xu2010_highdimensional}).
However, as we explained earlier the violation
of exact recovery does not necessarily imply non-robustness of the estimator as it may nearly recover the subspace.
Indeed, \citet{Xu2010_highdimensional} show that if (for simplicity) $\sigma_0=\sigma_1$ and the SNR is greater than 1,
then the subspace estimated by HR-PCA is a good approximation in the following sense:
there exists a constant $c>0$ such that for the inliers set $\sX_0$ and the estimated subspace $\rmL$:
$\sum_{\bx \in \sX_0}  \|\bP_{\rmL} \bx\|_2^2
> c \sum_{\bx \in \sX_0} \|\bx\|_2^2$ (see Remark 4 of \citet{Xu2010_highdimensional}). We thus use the notation: SNR(HR-PCA) ``$\gtrapprox$'' 1 (see
Table~\ref{tab:SNR} with appropriate scales of $\sigma_0$ and $\sigma_1$).
\citet{Xu2012} established the SNR for their Outlier Pursuit (OP) algorithm
(equivalently the Low-Leverage Decomposition (LLD) of \citealt{robust_mccoy}) in Theorem 1 of their work.
Their analysis assumes a deterministic condition, but it is possible to show that this condition is asymptotically valid under the needle-haystack model.
\citet{LMTZ2012} established w.h.p.~the SNR of the REAPER algorithm in Theorem 1 of their work
(for simplicity of their expressions they assumed that $d \leq (D-1)/2$).
\citet{Teng_log_rpca} established the SNR for Tyler's M-Estimator (TME) in Theorem 1 of his work. His result is deterministic,
but it is easy to show that the deterministic condition holds with probability 1 under the needle-haystack model.
\citet{moitra_pca2012} proposed randomized and deterministic robust recovery algorithms,
RF (or RandomizedFind) and DRF (or DERandomizedFind) respectively,
and proved that they obtained the same SNR as in \cite{Teng_log_rpca} under a similar (slightly weaker) combinatorial condition
(they only guarantee polynomial time, where \citealp{Teng_log_rpca} specifies a complexity similar to that of GMS).
We remark that both \citet{Teng_log_rpca} and \citet{moitra_pca2012} appeared after the submission of this manuscript.

\begin{table*}[htbp]
\centering
\begin{tabular}{@{}|@{~}c@{~}|@{~}c@{~}|@{~}c@{~}|@{~}c@{~}|@{~}c@{~}|@{}}
\hline
HR-PCA&LLD (OP) & $\hat{\rmL}:=\ker(\hat{\bQ})$ &REAPER ($d \leq (D-1)/2$) & TME \& D/RF \\
\hline
$\frac{\sigma_1\alpha_1}{\sigma_0\alpha_0}$ ``$\gtrapprox$'' $1$
&
$\frac{\alpha_1}{\alpha_0} \geq \frac{121d}{9}$
&$\frac{\alpha_1}{\alpha_0}>4 \frac{\sigma_0}{\sigma_1} \frac{d}{\sqrt{(D-d)D}}$
&$\frac{\alpha_1}{\alpha_0} > \frac{\sigma_0}{\sigma_1} \left(C_1 \frac{d}{D}- \frac{d}{C_2 \alpha_1}\right)$
&$\frac{\alpha_1}{\alpha_0}>\frac{d}{D-d}$\\
\hline
\end{tabular}
\caption{{Theoretical SNR (lowest bound on $\alpha_1/\alpha_0$) for exact recovery when $N\rightarrow\infty$}}\label{tab:SNR}
\end{table*}

The asymptotic SNR of the minimization proposed in this paper is of the same order as that of the REAPER algorithm (which was established for
$d \leq (D-1)/2$) and both of them are better than that of the HR-PCA algorithm.
The asymptotic SNRs of OP, TME, RF and DRF are independent of $\sigma_1$ and $\sigma_0$.
However, by normalizing all data points to the unit sphere, we may assume that $\sigma_1=\sigma_0$ in all other algorithms
and treat them equally  (see \citealp{LMTZ2012}).
In this case, the SNR of OP is significantly worse than that of the minimization proposed in here,
especially when $d \ll D$ (it is also worse than the weaker SNR specified in~\eqref{eq:prob_cond3}). When $d \ll D$,
the SNR of TME, RF and DRF is of
the same order as the asymptotic SNR of our formulation. However, when $d$ is very close to
$D$, the SNR of our formulation is better than the SNR of TME by a factor of $\sqrt{D}$.
We question whether a better asymptotic rate than the one of GMS and REAPER can be obtained by
a convex algorithm for robust subspace recovery for the needle-haystack model.
\cite{moitra_pca2012} showed that it is small set expansion hard for any algorithm to obtain better SNR than theirs
for all scenarios satisfying their combinatorial condition.

We note though that there are non-convex methods for removing outliers with asymptotically zero SNRs. Such SNRs are valid
only for the noiseless case and may be differently formulated for detecting the hidden low-dimensional structure among uniform outliers.
For example, \citet{Arias-Castro05connect} proved that the scan statistics may detect points sampled uniformly from a $d$-dimensional graph in $\reals^D$
of an $m$-differentiable function among uniform outliers in a cube in $\reals^D$ with SNR of order $O(N^{-m(D-d)/(d+m(D-d))})$.
\cite{higher-order} used higher order spectral clustering affinities
to remove outliers and thus detect differentiable surfaces (or certain
unions of such surfaces) among uniform outliers
with similar SNR to that of the scan statistics.
\citet{Soltanolkotabi2011} removed outliers with ``large dictionary coefficients''
and showed that this detection works well for outliers uniform in $S^{D-1}$, inliers uniform in
$S^{D-1}\cap \rmL^*$ and SNR at least $\frac{d}{D}\cdot((\frac{\alpha_1 N-1}{d})^ {\frac{cD}{d}-1}-1)^{-1}$
(where $\alpha_1$ is the fraction of inliers) as long as $N<e^{c\sqrt{D}}/D$. For fixed $D$ and $d$ and sufficiently large $N$, this SNR, which depends on $N$, can be arbitrarily small.
Furthermore,
\citet{lp_recovery_part1_11} showed that the global minimizer of \eqref{eq:geom_subs} (that we relax in this paper so that
the minimization is convex) can in theory recover the subspace with
asymptotically zero SNR. They also showed that the underlying subspace is a local minimum of \eqref{eq:geom_subs} with SNR of order $\omega(1/\sqrt{N})$.
However, these non-convex procedures do not have efficient or sufficiently fast implementations for subspace recovery. Furthermore, their impressive
theoretical estimates often break down in the presence of noise. Indeed, in the noisy case their near-recovery is not
better than the one stated for GMS in Theorem~\ref{thm:noisy_recovery} (see, e.g., (16) and (17) of \cite{higher-order} or Theorem 1.2 of \citet{lp_recovery_part1_11}).
On the other hand, in view of \cite{Coudron_Lerman2012} we may obtain significantly better asymptotic SNR for GMS when the
noise is symmetrically distributed with respect to the underlying subspace.

\subsubsection{A Special Case with Asymmetric Outliers}
\label{sec:asym_outliers_theory}

In the case of spherically symmetric outliers, PCA cannot exactly recover the underlying subspace, but
it can asymptotically recover it (see, e.g., \citealp{lp_recovery_part1_11}).
In particular, with sufficiently large sample with spherically symmetric outliers, PCA nearly recovers the underlying subspace.
We thus slightly modify the two models of \S\ref{sec:sym_outliers_theory} so that the distribution of outliers is asymmetric
and show that our combinatorial conditions for exact recovery still hold (with overwhelming probability).
On the other hand, the subspace recovered by PCA, when sampling data from these models,
is sufficiently far from the underlying subspace for any given sample size.

We first generalize Theorem~\ref{thm:prob_b} under a generalized needle-haystack model: Let $\mu=\alpha_0\mu_0+\alpha_1\mu_1$,
$\mu_0=N(\b0,\bSigma_0/D)$, where $\bSigma_0$ is an arbitrary positive definite matrix (not necessarily a
scalar matrix as before), and as before $\mu_1= N(\b0,\sigma_1^2\bP_{\rmL^*}\bP_{\rmL^*}^T/d)$.
We claim that Theorem~\ref{thm:prob_b} still holds in this case if we replace $\sigma_0$ in the RHS of \eqref{eq:prob_cond3} with
$\sqrt{\lambda_{\max}(\bSigma_0)}$, where $\lambda_{\max}(\bSigma_0)$ denotes the largest eigenvalue of $\bSigma_0$ (see justification
in \S\ref{sec:proof_thm_prob_b_asym}).

In order to generalize Theorem~\ref{thm:prob} for asymmetric outliers, we assume that the outlier component $\mu_0$
of the OIM measure $\mu$ is a sub-Gaussian distribution with an arbitrary positive definite covariance matrix
$\bSigma_0$.
Following \citet{Coudron_Lerman2012}, we define the expected version of $F$, $F_I$,
and its oracle minimizer, $\hat{\bQ}_I$, which is analogous to \eqref{eq:symmetric1} (the subscript $I$ indicates integral):
\begin{equation}\label{eq:define_F_I}
F_I(\bQ)=\int\|\bQ\bx\|\di\mu(x)
\end{equation}
and
\begin{equation}\label{eq:define_bQ_I}
\hat{\bQ}_I=\argmin_{\bQ\in\bbH, \bQ\bP_{\rmL^*}=\b0}F_I(\bQ).
\end{equation}
We assume that $\hat{\bQ}_I$ is the unique minimizer in~\eqref{eq:define_bQ_I}
(we remark that the two-subspaces criterion in \eqref{eq:convex_condition_asymptotic}
for the projection of $\mu$ onto $L^{*\perp}$ implies this assumption).
Under these assumptions Theorem~\ref{thm:prob} still holds if we multiply the RHS of
\eqref{eq:prob_cond2} by the ratio between the largest eigenvalue of $\bP_{\rmL^{*\perp}}\hat{\bQ}_I\bP_{\rmL^{*\perp}}$
and the $(D-d)$th eigenvalue of $\bP_{\rmL^{*\perp}}\hat{\bQ}_I\bP_{\rmL^{*\perp}}$ (see justification
in \S\ref{sec:proof_thm_prob_asym}).

\subsection{Near Subspace Recovery for Noisy Samples} \label{sec:noisy}

We show that in the case of sufficiently small additive noise
(i.e., the inliers do not lie exactly on the subspace $\rmL^*$ but close to it),
the GMS algorithm nearly recovers the underlying subspace.

We use the following notation: $\|\bA\|_F$ and $\|\bA\|$ denote the Frobenius
and spectral norms of $\bA \in \reals^{k\times l}$ respectively. Furthermore, $\bbH_1$ denotes the set of all positive semidefinite matrices in $\bbH$,
that is, $\bbH_1=\{\bQ\in\bbH:\bQ \psdge 0\}$.
We also define the following two constants
\begin{equation}
\label{eq:def_gamma0}
\gamma_0=\frac{1}{N}\min_{\bQ\in\bbH_1,\|\bDelta\|_F=1,\tr(\bDelta)=0} \sum_{i=1}^N\frac{\|\bDelta\bx_i\|^2\|\bQ\bx_i\|^2-(\bx_i^T\bDelta\bQ\bx_i)^2}{\|\bQ\bx_i\|^3},
\end{equation}
and
\begin{equation}
\label{eq:def_gamma0p}
\gamma_0'=\frac{1}{N}\min_{\bQ\in\bbH_1,\|\bDelta\|=1,\tr(\bDelta)=0} \sum_{i=1}^N\frac{\|\bDelta\bx_i\|^2\|\bQ\bx_i\|^2-(\bx_i^T\bDelta\bQ\bx_i)^2}{\|\bQ\bx_i\|^3}.
\end{equation}
The sum in the RHS's of \eqref{eq:def_gamma0} and \eqref{eq:def_gamma0p} is the following second directional derivative:
$\frac{\di^2}{\di t^2}F(\bQ+t \bDelta)$; when $\bQ\bx_i=0$, its $i$th term can be set to $0$.
It is interesting to note that both~\eqref{eq:def_gamma0} and~\eqref{eq:def_gamma0p} express
the Restricted Strong Convexity (RSC) parameter $\gamma_l$ of \citet[Definition 1]{Agarwal_Negahban_Wainwright_2011b}, where their notation translates into
ours as follows: $\mathcal{L}_n(\bQ):=F(\bQ)/N$, $\tau_l:=0$, $\Omega':=\bbH_1$ and $\theta-\theta':=\Delta$.
The difference between $\gamma_0$ and $\gamma_0'$ of \eqref{eq:def_gamma0}
and \eqref{eq:def_gamma0p}
is due to the choice of either the Frobenius or the spectral norms respectively for measuring the size of $\theta-\theta'$.

Using this notation, we formulate our noise perturbation result as follows.
\begin{theorem}\label{thm:noisy_recovery}
Assume that $\{\eps_i\}_{i=1}^N$ is a set of positive numbers, $\sX=\{\bx_i\}_{i=1}^N$ and $\tilde{\sX}=\{\tilde{\bx}_i\}_{i=1}^N$  are two data sets such that $\|\tilde{\bx}_i-\bx_i\|\leq \eps_i$ $\ \forall 1 \leq i \leq N$ and $\sX$ satisfies \eqref{eq:convex_condition}.
Let $F_\sX(\bQ)$ and $F_{\tilde{\sX}}(\bQ)$ denote the corresponding versions of
$F(\bQ)$ w.r.t.~the sets $\sX$ and $\tilde{\sX}$ and let $\hat{\bQ}$ and $\tilde{\bQ}$ denote their respective minimizers.
Then we have
\begin{equation}\label{eq:bQ_diff}
\|\tilde{\bQ}-\hat{\bQ}\|_F<\sqrt{2 \sum_{i=1}^{N}\eps_i/(N\gamma_0)} \ \ \text{ and } \ \ \|\tilde{\bQ}-\hat{\bQ}\|<\sqrt{2 \sum_{i=1}^{N}\eps_i/(N\gamma_0')}.
\end{equation}
Moreover, if $\tilde{\rmL}$ and $\hat{\rmL}$ are the subspaces spanned by the bottom $d$ eigenvectors of $\tilde{\bQ}$ and $\hat{\bQ}$ respectively and
$\nu_{D-d}$ is the $(D-d)$th eigengap of $\hat{\bQ}$, then
\begin{equation}
\|\bP_{\hat{\rmL}}-\bP_{\tilde{\rmL}}\|_F\leq \frac{2\sqrt{2 \sum_{i=1}^{N}\eps_i/(N\gamma_0)}}{\nu_{D-d}}
\ \ \text{ and } \ \
\|\bP_{\hat{\rmL}}-\bP_{\tilde{\rmL}}\|\leq \frac{2\sqrt{2 \sum_{i=1}^{N}\eps_i/(N\gamma_0')}}{\nu_{D-d}}.
\label{eq:subspace_diff}
\end{equation}
\end{theorem}

We note that if $\sX$ and $\tilde{\sX}$ satisfy the conditions of Theorem~\ref{thm:noisy_recovery},
then given the perturbed data set $\tilde{\sX}$ and the dimension $d$, Theorem~\ref{thm:noisy_recovery} guarantees that
GMS nearly recovers $\rmL^*$.
More interestingly, the theorem also implies that we may properly estimate the dimension of the underlying subspace in this case
(we explain this in details in \S\ref{sec:noisy_recovery_imply1}).
Such dimension estimation is demonstrated later in Figure~\ref{fig:svd}.

Theorem~\ref{thm:noisy_recovery} is a perturbation result in the spirit of the stability analysis by \citet{candes_romberg_tao_cpam06} and \citet[Theorem 2]{Xu2012}.
In order to observe that the statement of Theorem~\ref{thm:noisy_recovery} is comparable to that of Theorem 2 of \citet{Xu2012}, we note that
asymptotically the bounds on the recovery errors in \eqref{eq:bQ_diff} and \eqref{eq:subspace_diff} depend only on the empirical mean of $\{\eps_i\}_{i=1}^N$ and do not grow with $N$. To clarify this point we formulate the following proposition.
\begin{proposition}
\label{prop:noisy_recovery}
If $\sX$ is i.i.d. sampled from a bounded distribution $\mu$ and \begin{equation}
\label{eq:convex_condition_asymptotic}
\mu(\rmL_1)+\mu(\rmL_2)<1\,\,\,\,\text{for any two $D-1$-dimensional subspaces $\rmL_1$ and $\rmL_2$,}
\end{equation}
then there exist constants $c_0(\mu)>0$ and $c'_0(\mu)>0$ depending on $\mu$ such that
\begin{equation}
\text{$\liminf_{N\rightarrow\infty}{\gamma_0(\sX)}
\geq c_0(\mu)$ \ and \ $\liminf_{N\rightarrow\infty}{\gamma'_0(\sX)}
\geq c'_0(\mu)$ almost surely}.
\label{eq:asymptotic_derivative}
\end{equation}
\end{proposition}
If \eqref{eq:convex_condition_asymptotic} is strengthened so that $\mu(\rmL_1)+\mu(\rmL_2)$ is sufficiently smaller than 1, then it can be noticed empirically that $c_0(\mu)$ and $c_0'(\mu)$ are sufficiently larger than zero.

Nevertheless, the stability theory of \citet{candes_romberg_tao_cpam06}, \citet{Xu2012} and this section is not optimal.
Stronger stability results require nontrivial analysis and we leave it to a possible future work. We comment though on some of the deficiencies
of our stability theory and their possible improvements.

We first note that the bounds in Theorem~\ref{thm:noisy_recovery} are generally not optimal.
Indeed, if $\eps_i=O(\eps)$ for all $1\leq i\leq N$, then the error bounds in Theorem~\ref{thm:noisy_recovery} are
$O(\sqrt{\eps})$, whereas we empirically noticed that these error bounds are $O(\eps)$.
In \S\ref{sec:noisy_recovery_imply2} we sketch a proof for this empirical observation
when $\eps$ is sufficiently small and $\rank(\hat{\bQ})=D$.

The dependence of the error on $D$, which follows from the dependence of $\gamma_0$ and $\gamma_0'$ on $D$,
is a difficult problem and strongly depends on the underlying distribution of $\sX$ and of the noise.
For example, in the very special case where the set $\sX$ is sampled from a subspace $L_0 \subset \reals^D$ of dimension $D_0<D$, and
the noise distribution is such that $\tilde{\sX}$ also lies in $L_0$, then practically we are performing GMS over $P_{L_0}(\sX)$ and $P_{L_0}(\tilde{\sX})$, and the bound in \eqref{eq:bQ_diff} would depend on $D_0$
instead of $D$.

\citet{Coudron_Lerman2012} suggested a stronger perturbation analysis and also remarked on the dependence of the error on $D$ in a very special scenario.

\subsection{Near Subspace Recovery for Regularized Minimization}
\label{sec:regularize}

For our practical algorithm it is advantageous to regularize the function $F$ as follows (see Theorems~\ref{thm:alg_clean} and~\ref{thm:alg_noisy} below):
$$
F_\delta(\bQ):=\sum_{i=1,\|\bQ\bx_i\|\geq \delta}^N\|\bQ\bx_i\|+ \sum_{i=1,\|\bQ\bx_i\|< \delta}^N\left(\frac{\|\bQ\bx_i\|^2}{2\delta}+\frac{\delta}{2}\right)
.
$$
We remark that other convex algorithms \citep{candes_wright_robust_pca09,Xu2012,robust_mccoy} also regularize their objective function by adding the term
$\delta \| \bX-\bL-\bO\|_F^2$. However, their proofs are not formulated for this regularization.

In order to address the regularization in our case and conclude that the GMS algorithm nearly recovers $\rmL^*$
for the regularized objective function, we adopt a similar perturbation procedure as in \S\ref{sec:noisy}.
We denote by $\hat{\bQ}_\delta$ and $\hat{\bQ}$ the minimizers of $F_\delta(\bQ)$ and $F(\bQ)$ in $\bbH$ respectively.
Furthermore, let $\hat{\rmL}_\delta$ and $\hat{\rmL}$ denote the subspaces recovered by the bottom $d$ eigenvectors of
$\hat{\bQ}_\delta$ and $\hat{\bQ}$ respectively. Using the constants $\nu_{D-d}$ and $\gamma_0$ of Theorem~\ref{thm:noisy_recovery}, the difference between the two minimizers and subspaces can be controlled as follows.
\begin{theorem}
\label{thm:noisy}
If $\sX$ is a data set satisfying \eqref{eq:convex_condition}, then
$$
\|\hat{\bQ}_\delta-\hat{\bQ}\|_F<\sqrt{\delta/2\gamma_0}
$$
and
\begin{equation}
\|\bP_{\hat{\rmL}_\delta} - \bP_{\hat{\rmL}}\|_F\leq \frac{2\sqrt{\delta/2\gamma_0}}{\nu_{D-d}}.\label{eq:subspace_diff2}
\end{equation}
\end{theorem}

\section{Understanding Our M Estimator: Interpretation and Formal Similarities with Other M Estimators}
\label{sec:significance}

We highlight the formal similarity of our M-estimator with a common M-estimator and with Tyler's M-estimator in \S\ref{sec:m_estimator_whole} and \S\ref{sec:m_estimator_tyler} respectively. We also show that in view of the standard assumptions on the algorithm for computing the common M-estimator, it may fail in exactly recovering the underlying subspace (see \S\ref{sec:m_estimator}). At last, in \S\ref{sec:l2} we interpret our M-estimator as a robust inverse covariance estimator.
%NEED TEXT HERE

\subsection{Formal Similarity with the Common M-estimator for Robust Covariance Estimation}
\label{sec:m_estimator_whole}

A well-known robust M-estimator for the $\b0$-centered covariance matrix
\citep{Maronna1976, huber_book, robust_stat_book2006}
minimizes the following function
over all $D\times D$ positive definite matrices (for some choices of a function $\rho$)
\be \label{eq:mestimator_function}
L(\bA) = \sum_{i=1}^{N}\rho(\bx_i^T\bA^{-1}\bx_i) - \frac{N}{2}\log(\det(\bA^{-1})).
\ee
The image of the estimated covariance is clearly an estimator to the underlying subspace $\rmL^*$.

If we set $\rho(x)=\sqrt{x}$ and $\bA^{-1}=\bQ^2$ then the objective function $L(\bA)$ in \eqref{eq:mestimator_function} is
$\sum_{i=1}^{N}\|\bQ\bx_i\|-N\log(\det(\bQ))$.
This energy function is formally similar to our energy function.
Indeed, using Lagrangian formulation, the minimizer $\hat{\bQ}$ in \eqref{eq:symmetric} is also the minimizer of
$\sum_{i=1}^{N}\|\bQ\bx_i\|-\lambda \tr(\bQ)$ among all $D\times D$ symmetric matrices (or equivalently nonnegative symmetric matrices) for some
$\lambda>0$ (the parameter $\lambda$ only scales the minimizer and does not effect the recovered subspace).
Therefore, the two objective functions differ by their second terms. In the common M-estimator (with $\rho(x)=\sqrt{x}$ and $\bA^{-1}=\bQ^2$) it is
$\log(\det(\bQ))$, or equivalently, $\tr(\log(\bQ))$, where in our M-estimator, it is $\tr(\bQ)$.

\subsubsection{Problems with Exact Recovery by the Common M-estimator}
\label{sec:m_estimator}

The common M-estimator is designed for robust covariance estimation, however, we show here that in general it cannot exactly recover the underlying subspace.
To make this statement more precise we recall the following uniqueness and existence conditions for the minimizer of~\eqref{eq:mestimator_function}, which were established by \citet{Kent1991}: 1) $u=2\rho'$ is positive, continuous and non-increasing. 2) Condition M: $u(x)x$ is strictly increasing.
3) Condition $\text{D}_0$: For any linear subspace $\rmL$: $|\sX\cap\rmL|/N < 1 -(D-\dim(\rmL))/ \lim_{x\rightarrow\infty}xu(x)$.
The following Theorem~\ref{thm:m_estimator} shows that the  uniqueness and existence conditions of the common M-estimator are incompatible with exact recovery.

\begin{theorem}\label{thm:m_estimator}
Assume that $d,D \in \nats$, $d<D$, $\sX$ is a data set in $\reals^D$ and $\rmL^*\in\GDd$ and let $\hat{\bA}$ be the minimizer of~\eqref{eq:mestimator_function}. If conditions M and $\text{D}_0$ hold, then $\im(\hat{\bA})\neq \rmL^*$.
\end{theorem}

For symmetric outliers (as the ones of \S\ref{sec:sym_outliers_theory}) the common M-estimator can still asymptotically achieve exact recovery
(similarly to PCA).
However, for many scenarios of asymmetric outliers, in particular, the one of \S\ref{sec:asym_outliers_theory},
the subspace recovered by the common M-estimator is sufficiently far from the underlying subspace for any given sample size.

We remark that Tyler's M-estimator \citep{tyler_dist_free87} can still recover the subspace exactly. This estimator uses $\rho(x)=D\log(x)/2$ in~\eqref{eq:mestimator_function} and adds an  additional assumption
 $\tr(\bA)=1$. \citet{Teng_log_rpca} recently showed that this M-estimator satisfies $\im(\hat{\bA})=\rmL^*$.
However, it does not belong to the class of estimators of \citet{Kent1991} addressed by Theorem~\ref{thm:m_estimator} (it requires that $\tr(\bA)=1$, otherwise it has multiple minimizers; it also does not satisfy condition M).

\subsection{Formal Similarity with Tyler's M-Estimator}
\label{sec:m_estimator_tyler}

We show here that the algorithms for our estimator and Tyler's M-estimator \citep{tyler_dist_free87} are formally similar.
Following \cite{tyler_dist_free87}, we write the iterative
algorithm for the Tyler's M-estimator for robust covariance estimation as follows:
\begin{equation}
\label{eq:tyler_m_alg}
\bSigma_{n+1}=\sum_{i=1}^N
\frac{\bx_i\bx_i^T}{\bx_i^T\bSigma_n^{-1}\bx_i}\Big/\tr\left(\sum_{i=1}^N
\frac{\bx_i\bx_i^T}{\bx_i^T\bSigma_n^{-1}\bx_i}\right).
\end{equation}
The unregularized iterative algorithm for GMS is later described in~\eqref{eq:iterate_no_regular0}.
Let us formally substitute %$\bSigma=\bQ^{-1}$
$\bSigma=\bQ^{-1}/\tr(\bQ^{-1})$ in~\eqref{eq:iterate_no_regular0}; in view of the later discussion of~\ref{sec:l2}, $\bSigma$
(if exists) can be interpreted as a robust estimator for the covariance matrix (whose top $d$ eigenvectors span the estimated subspace).
Then an unregularized version for GMS can be formally written as
\begin{equation}
\label{eq:our_vs_tyler}
\bSigma_{n+1}=\sum_{i=1}^N
\frac{\bx_i\bx_i^T}{\|\bSigma_n^{-1}\bx_i\|}\Big/\tr\left(\sum_{i=1}^N
\frac{\bx_i\bx_i^T}{\|\bSigma_n^{-1}\bx_i\|}\right).
\end{equation}
Clearly, \eqref{eq:our_vs_tyler} is obtained from \eqref{eq:tyler_m_alg} by replacing $\bx_i^T \bSigma_n^{-1} \bx_i$ with
$\|\bSigma_n^{-1}\bx_i\| \equiv \sqrt{\bx_i^T \bSigma_n^{-2} \bx_i}$.

\subsection{Interpretation of $\hat{\bQ}$ as Robust Inverse Covariance Estimator}
\label{sec:l2}

The total least squares subspace approximation is practically the minimization over $\rmL \in \GDd$ of the function
\be
\label{eq:geom_subs_l2}
\sum_{i=1}^N \|\bx_i - \bP_{\rmL}\bx_i\|^2 \equiv \sum_{i=1}^N \|\bP_{\rmL^\perp}\bx_i\|^2\,.
\ee
Its solution is obtained by the span of the top $d$ right vectors of the data matrix $\bX$ (whose rows are the data points in $\sX$), or equivalently, the top $d$ eigenvectors of
the covariance matrix $\bX^T\bX$.
The convex relaxation used in \eqref{eq:geom_subs_l2} can be also applied to~\eqref{eq:geom_subs_l2} to obtain the following convex minimization problem:
\begin{equation}\label{eq:symmetric_PCA}
\hat{\bQ}_2:=\argmin_{\bQ\in\bbH}\sum_{i=1}^{N}\|\bQ\bx_i\|^2.
\end{equation}
The ``relaxed'' total least squares subspace is then obtained by the span of the bottom $d$ eigenvectors of $\hat{\bQ}$.

We show here that $\hat{\bQ}_2$ coincides with a scaled version of the empirical inverse covariance matrix.
This clearly imply that the ``relaxed'' total least squared subspace coincides with the original one
(as the bottom eigenvectors of the inverse empirical covariance are the top eigenvectors of the empirical covariance).
We require though that the data is of full rank so that the empirical inverse covariance is well-defined. This requirement does not hold
if the data points are contained within a lower-dimensional subspace, in particular, if their number is smaller than the dimension.
We can easily avoid this restriction by initial projection of the data points onto the span of eigenvectors of the covariance matrix with
nonzero eigenvalues.
That is, by projecting the data onto the lowest-dimensional subspace containing it without losing any information.

\begin{theorem}\label{thm:pca}
If $\bX$  is the data matrix, $\hat{\bQ}_2$ is the minimizer of~\eqref{eq:symmetric_PCA} %w.r.t.~the same data set
and $\rank(\bX)=D$ (equivalently the data points span $\reals^D$), then
\be \label{eq:pca_inverse_cov}\hat{\bQ}_2=(\bX^T\bX)^{-1}/\tr((\bX^T\bX)^{-1}).\ee
\end{theorem}

We view~\eqref{eq:symmetric} as a robust version of~\eqref{eq:symmetric_PCA}. Since we verified robustness of the subspace recovered by~\eqref{eq:symmetric} and also showed that~\eqref{eq:symmetric_PCA} yields the inverse covariance matrix, we sometimes refer to the solution of~\eqref{eq:symmetric} as a robust inverse covariance matrix (though we have only verified robustness to subspace recovery).
This idea helps us interpret our numerical procedure for minimizing~\eqref{eq:symmetric}, which we present in \S\ref{sec:algorithm_whole}.

\section{IRLS Algorithms for Minimizing~\eqref{eq:symmetric}}
\label{sec:algorithm_whole}

We propose a fast algorithm for computing our M-estimator by using a straightforward iterative re-weighted least squares (IRLS) strategy.
We first motivate this strategy in \S\ref{sec:heuristic} (in particular, see
\eqref{eq:iterate_no_regular0} and \eqref{eq:iterate_regular}). We then establish its linear convergence in \S\ref{sec:theory_alg}. At last, we describe its practical choices in \S\ref{sec:practical} and summarize its complexity in \S\ref{sec:complexity}.

%NEED TEXT HERE

\subsection{Heuristic Proposal for Two IRLS Algorithms} \label{sec:heuristic}
The procedure for minimizing~\eqref{eq:symmetric} formally follows from the simple fact that
the directional derivative of $F$ at $\hat{\bQ}$ in any direction $\tilde{\bQ}-\hat{\bQ}$, where
$\tilde{\bQ} \in \bbH$, is 0, that is,
\be
\label{eq:explain_alg_min}
\left\langle F'(\hat{\bQ})\Big|_{\bQ=\hat{\bQ}},\tilde{\bQ}-\hat{\bQ}\right\rangle_F=0  \ \ \text{for any} \ \ \tilde{\bQ}\in\bbH.
\ee
We remark that since $\bbH$ is an affine subspace of matrices, \eqref{eq:explain_alg_min}
holds globally in $\bbH$ and not just locally around $\hat{\bQ}$.

We formally differentiate \eqref{eq:symmetric} at $\hat{\bQ}$ as follows (see more details in \eqref{eq:derivative0}, which appears later):
\begin{equation}
\label{eq:def_F_prime}
F'(\bQ)\Big|_{\bQ=\hat{\bQ}}=\sum_{i=1}^{N}\frac{ \hat{\bQ} \bx_i\bx_i^T+\bx_i\bx_i^T\hat{\bQ}}{2\| \hat{\bQ} \bx_i\|}.
\end{equation}
Throughout the formal derivation we ignore the possibility of zero denominator in \eqref{eq:def_F_prime}, that is, we assume that $\hat{\bQ} \bx_i \neq \b0 \ \forall \ 1 \leq i \leq N$; we later address this issue.

Since $F'(\hat{\bQ})$ is symmetric and $\tilde{\bQ}-\hat{\bQ}$ can be any symmetric matrix with trace $0$, it is easy to note that~\eqref{eq:explain_alg_min} implies that
$F'(\hat{\bQ})$ is a scalar matrix (e.g., multiply it by a basis of symmetric matrices with trace 0 whose members have exactly 2 nonzero matrix elements). That is,
\begin{equation}
\label{eq:equal_to_cI}
\sum_{i=1}^{N}\frac{\hat{\bQ}\bx_i\bx_i^T+\bx_i\bx_i^T\hat{\bQ}}{2\|\hat{\bQ}\bx_i\|}=c\bI
\end{equation}
for some $c \in \reals$. This implies that
\begin{equation}
\label{eq:solution_equal_to_cI}
\hat{\bQ}=c\left(\sum_{i=1}^{N}\frac{\bx_i\bx_i^T}{\|\hat{\bQ}\bx_i\|}\right)^{-1}.
\end{equation}
Indeed, we can easily verify that \eqref{eq:solution_equal_to_cI} solves \eqref{eq:equal_to_cI},
furthermore, \eqref{eq:equal_to_cI} is a Lyapunov equation whose solution is unique (see, e.g., page 1 of \citet{Bhatia2005}).
Since $\tr(\hat{\bQ})=1$, we obtain that
\[
\hat{\bQ}=\left(\sum_{i=1}^{N}\frac{\bx_i\bx_i^T}{\|\hat{\bQ}\bx_i\|}\right)^{-1}/\tr\left(\left(\sum_{i=1}^{N}\frac{\bx_i\bx_i^T}{\|\hat{\bQ}\bx_i\|}\right)^{-1}\right),
\]
which suggests the following iterative estimate of $\hat{\bQ}$:
\be
\label{eq:iterate_no_regular0}
{\bQ_{k+1}}=
\left(\sum_{i=1}^{N}\frac{\bx_i\bx_i^T}{\|\bQ_k \bx_i\|}\right)^{-1}/\tr\left(\left(\sum_{i=1}^{N}\frac{\bx_i\bx_i^T}{\|\bQ_k \bx_i\|}\right)^{-1}\right).
\ee

Formula~\eqref{eq:iterate_no_regular0} is undefined whenever $\bQ_k \bx_i = \b0$ for some $k \in \nats$ and $1 \leq i \leq N$.
In theory, we address it as follows.
Let $I(\bQ)=\{1\leq i\leq N: \bQ\bx_i=\b0\}$, $\rmL(\bQ)=\Sp\{\bx_i\}_{i\in I(\bQ)}$ and
$$
T(\bQ)\!=\!{\bP}_{\rmL(\bQ)^\perp}\!\!\left(\!\sum_{i\notin I(\bQ)} \! \frac{\bx_i\bx_i^T}{\|\bQ\bx_i\|}\!\right)^{-1}\!\!\!\!\!\!\!{\bP}_{\rmL(\bQ)^\perp}/\tr\!\left(\!{\bP}_{\rmL(\bQ)^\perp}\!\!\left(\!\sum_{i\notin I(\bQ)} \! \frac{\bx_i\bx_i^T}{\|\bQ\bx_i\|}\!\right)^{-1}\!\!\!\!\!\!\!{\bP}_{\rmL(\bQ)^\perp}\!\!\right)\!.
$$
Using this notation, the iterative formula can be corrected as follows
\begin{equation}
\label{eq:iterate_no_regular}
\bQ_{k+1}=T(\bQ_k).
\end{equation}
In practice, we can avoid data points satisfying $\|\bQ_k \bx_i\| \leq \delta$ for a sufficiently small parameter $\delta$ (instead of $\|\bQ_k \bx_i\|=0$). We follow a similar idea by replacing $F$ with the regularized function $F_\delta$ for a regularized parameter $\delta$.
In this case, \eqref{eq:iterate_no_regular}
obtains the following form:
\be
\label{eq:iterate_regular}
{\bQ_{k+1}}=
\left(\sum_{i=1}^{N}\frac{\bx_i\bx_i^T}{\max(\|\bQ_k\bx_i\|,\delta)}\right)^{-1}/\tr\left(\left(\sum_{i=1}^{N}\frac{\bx_i\bx_i^T}
{\max(\|\bQ_k\bx_i\|,\delta)}\right)^{-1}\right).
\ee
We note that the RHS of~\eqref{eq:iterate_no_regular} is obtained as the limit of the RHS of~\eqref{eq:iterate_regular} when $\delta$ approaches 0.

The two iterative formulas, that is, \eqref{eq:iterate_no_regular} and \eqref{eq:iterate_regular}, give rise to IRLS algorithms. For simplicity of notation, we exemplify this idea with the formal expression in~\eqref{eq:iterate_no_regular0}. It iteratively finds the solution to the following weighted (with weight $1/\|\bQ_k\bx_i\|$) least squares problem:
\be
\label{eq:update_irls}
\argmin_{\bQ\in\bbH}\sum_{i=1}^{N}\frac{1}{\|\bQ_k\bx_i\|}\|\bQ\bx_i\|^2.
\ee
To show this, we note that~\eqref{eq:update_irls} is a quadratic function and any formal directional derivative at $\bQ_{k+1}$ is $0$. Indeed,
\[
\frac{\di}{\di \bQ}\sum_{i=1}^{N}\frac{1}{\|\bQ_k\bx_i\|}\|\bQ\bx_i\|^2\Big|_{\bQ=\bQ_{k+1}}\!\!\!\!=
\!\bQ_{k+1}\!\left(\sum_{i=1}^{N}\frac{\bx_i\bx_i^T}{\|\bQ_k\bx_i\|}\right)\!+\left(\sum_{i=1}^{N}\frac{\bx_i\bx_i^T}{\|\bQ_k\bx_i\|}\right)\!\bQ_{k+1}\!=c\bI
\]
for some $c\in\reals$, and $\left\langle\bI,\tilde{\bQ}-\bQ_{k+1}\right\rangle_F=0$ for any $\tilde{\bQ}\in\bbH$.
Consequently, $\bQ_{k+1}$ of~\eqref{eq:iterate_no_regular0} is the minimizer of~\eqref{eq:update_irls}.

Formula \eqref{eq:iterate_regular} (as well as \eqref{eq:iterate_no_regular}) provides another interpretation for $\hat{\bQ}$ as robust inverse covariance (in addition to the one
discussed in \S\ref{sec:l2}).
Indeed, we note for example that the RHS of~\eqref{eq:iterate_regular} is the scaled inverse of a weighted covariance matrix;  the scaling enforces the trace of the inverse to be 1 and the weights of $\bx_i\bx_i^T$ are significantly larger when $\bx_i$ is an inlier. In other words, the weights apply a shrinkage procedure for outliers.
Indeed, since $\bQ_k\bx_i$ approaches $\hat{\bQ} \bx_i$ and the underlying subspace, which contain the inliers, is recovered by $\ker(\hat{\bQ})$, for an inlier $\bx_i$ the coefficient of $\bx_i\bx_i^T$ approaches $1/\delta$, which is a very large number (in practice we use $\delta = 10^{-20}$).
On the other hand, when $\bx_i$ is sufficiently far from the underlying subspace, the coefficient of $\bx_i\bx_i^T$ is significantly smaller.

\subsection{Theory: Convergence Analysis of the IRLS Algorithms}
\label{sec:theory_alg}

The following theorem analyzes the convergence of the sequence proposed by~\eqref{eq:iterate_no_regular} to the minimizer of~\eqref{eq:symmetric}.
\begin{theorem}\label{thm:alg_clean}
Let $\sX =\{\bx_i\}_{i=1}^N$ be a data set in $\reals^D$ satisfying~\eqref{eq:convex_condition},
$\hat{\bQ}$ the minimizer of~\eqref{eq:symmetric},
$\bQ_0$ an arbitrary symmetric matrix with $\tr(\bQ_0)=1$
and $\{\bQ_i\}_{i \in \nats}$ the sequence obtained by iteratively applying~\eqref{eq:iterate_no_regular} (while initializing it with $\bQ_0$), then $\{\bQ_i\}_{i \in \nats}$ converges to a matrix $\tilde{\bQ}\in\bbH$.
If $\tilde{\bQ}\bx_i\neq \b0$ for all $1\leq i\leq N$, then $\tilde{\bQ}=\hat{\bQ}$ and furthermore, $\{F(\bQ_i)\}_{i \in \nats}$ converges linearly to
$F(\tilde{\bQ})$ and $\{\bQ_i\}_{i \in \nats}$ converges r-linearly to $\tilde{\bQ}$.
\end{theorem}

The condition for the linear convergence to $\hat{\bQ}$ in Theorem~\ref{thm:alg_clean} (i.e., $\hat{\bQ}\bx_i\neq \b0$ for all $1\leq i\leq N$) usually does not occur for noiseless data.
This condition is common in IRLS algorithms whose objective functions are $\ell_1$-type and are not twice differentiable at $\b0$. For example, Weiszfeld's Algorithm \citep{Weiszfeld1937} may not converge to the geometric median but to one of the data points \citep[\S3.4]{Kuhn73}.
On the other hand, regularized IRLS algorithms often converge linearly to the minimizer of the regularized function.
We demonstrate this principle in our case as follows.

\begin{theorem}\label{thm:alg_noisy}
Let $\sX =\{\bx_i\}_{i=1}^N$ be a data set in $\reals^D$ satisfying~\eqref{eq:convex_condition},
$\bQ_0$ an arbitrary symmetric matrix with $\tr(\bQ_0)=1$
and $\{\bQ_i\}_{i \in \nats}$ the sequence obtained by iteratively applying~\eqref{eq:iterate_regular} (while initializing it with $\bQ_0$).
Then, the sequence $\{F_\delta(\bQ_i)\}_{i \in \nats}$ converges linearly to
the unique minimum of $F_\delta(\bQ)$, and $\{\bQ_i\}_{i \in \nats}$ converges r-linearly to
the unique minimizer of $F_\delta(\bQ)$.
\end{theorem}

The convergence rate of the iterative application of~\eqref{eq:iterate_regular} depends on $\delta$. Following Theorem 6.1 of \citet{Chan99}, this rate is at most
\[
r(\delta)=
\sqrt{\max_{\bDelta=\bDelta^T,\mathrm{tr}(\bDelta)=0}\frac{\sum_{i=1,\|\bQ_*\bx_i\|>\delta}^N
\frac{(\bx_i^T\bDelta\bQ_*\bx_i)^2}{\|\bQ_*\bx_i\|^3}}{\sum_{i=1}^N\frac{\|\bDelta\bx_i\|^2}{\max(\|\bQ_*\bx_i,\delta\|)}}}.
\]
That is, $\|\bQ_k-\hat{\bQ}\|<C\cdot r(\delta)^k$ for some constant $C>0$. If \eqref{eq:convex_condition} holds, then $r(\delta)<1$ for all $\delta>0$ and $r(\delta)$ is a non-increasing function. Furthermore, if $\{\bx_i \in \sX : \|\hat{\bQ}\bx_i\|\neq 0\}$ satisfies assumption \eqref{eq:convex_condition}, then $\lim_{\delta\rightarrow 0}r(\delta)<1$.

\subsection{The Practical Choices for the IRLS Algorithm}
\label{sec:practical}

Following the theoretical discussion in \S\ref{sec:theory_alg} we prefer using the regularized version of the IRLS algorithm.
We fix the regularization parameter
to be smaller than the rounding error, that is, $\delta=10^{-20}$,
so that the regularization is very close to the original problem
(even without regularization the iterative process is stable, but may have few warnings on badly scaled or close to singular matrices).
The idea of the algorithm is to iteratively apply~\eqref{eq:iterate_regular} with an arbitrary initialization (symmetric with trace 1).
We note that in theory $\{F_\delta(\bQ_k)\}_{k \in \nats}$ is non-increasing (see, e.g., the proof of Theorem~\ref{thm:alg_noisy}).
However, empirically the sequence decreases when it is within the rounding error to the minimizer. Therefore, we check $F_\delta(\bQ_k)$ every four
iterations and stop our algorithm when we detect an increase
(we noticed empirically that checking every four iterations, instead of every iteration,
improves the accuracy of the algorithm). Algorithm~\ref{alg:practical}
summarizes our practical procedure for minimizing~\eqref{eq:symmetric}.

\begin{algorithm}[htbp]
\caption{Practical and Regularized Minimization of \eqref{eq:symmetric}} \label{alg:practical}
\begin{algorithmic}
\REQUIRE $\sX=\{\bx_1,\bx_2,\cdots,\bx_N\} \subseteq
\mathbb{R}^{D}$: data
\ENSURE $\hat{\bQ}$: a symmetric matrix in $\reals^{D\times D}$ with $\tr(\hat{\bQ})=1$.\\
\textbf{Steps}:
\STATE
     $\bullet$ $\delta=10^{-20}$\\
     $\bullet$ Arbitrarily initialize $\bQ_0$ to be a symmetric matrix with $\tr(\bQ_0)=1$\\
     $\bullet$ $k=-1$
     \REPEAT \STATE
     $\bullet$ k=k+1\\
     $\bullet$  $\bQ_{k+1}=\left(\sum_{i=1}^{N}\frac{\bx_i\bx_i^T}{\max(\|\bQ_k\bx_i\|,\delta)}\right)^{-1}/\tr\left(\left(\sum_{i=1}^{N}\frac{\bx_i\bx_i^T}{\max(\|\bQ_k\bx_i\|,\delta)}\right)^{-1}\right)$.\\
     \UNTIL $F(\bQ_{k+1}) > F(\bQ_{k-3})$ and $\mod(k+1,4)=0$\\
    $\bullet$ Output $\hat{\bQ} := \bQ_{k}$
\end{algorithmic}
\end{algorithm}

\subsection{Complexity of Algorithm~\ref{alg:practical}}
\label{sec:complexity}
Each update of Algorithm~\ref{alg:practical} requires a complexity of order $O(N\cdot D^2)$, due to the sum of $N$ $D\times D$ matrices. Therefore, for $n_s$ iterations the total running time of  Algorithm~\ref{alg:practical} is of order $O(n_s\cdot N\cdot D^2)$. In most of our numerical experiments $n_s$ was less than $40$. The storage of this algorithm is $O(N\times D)$, which amounts to storing $\sX$. Thus, Algorithm~\ref{alg:practical} has the same order of storage and complexity as PCA. In practice, it might be a bit slower due to a larger constant for the actual complexity.

\section{Subspace Recovery in Practice}
\label{sec:prac_solutions}

We view the GMS algorithm as a prototype for various subspace recovery algorithms.
We discuss here modifications and extensions of this procedure in order to make it even more practical.
Sections \ref{sec:theory_no_dim} and \ref{sec:theory_know_dim} discuss the cases where $d$ is unknown and known respectively; in particular,
\S\ref{sec:egms} proposes the EGMS algorithm when $d$ is known. At last, \S\ref{sec:egms_complexity} concludes with the computational
complexity of the GMS and EGMS algorithms.

\subsection{Subspace Recovery without Knowledge of $d$}
\label{sec:theory_no_dim}

In theory, the subspace recovery described here can work without knowing the dimension $d$.
In the noiseless case, one may use~\eqref{eq:gms_no_noise} to estimate the subspace as guaranteed by Theorem~\ref{thm:recovery}.
In the case of small noise one can estimate $d$ from the eigenvalues of $\hat{\bQ}$ and then apply the GMS algorithm.
This strategy is theoretically justified by Theorems~\ref{thm:recovery} and~\ref{thm:noisy_recovery} as well as the
discussion following~\eqref{eq:noisy_recovery}. The problem is that condition~\eqref{eq:condition3}
for guaranteeing exact recovery by GMS is restrictive;
in particular, it requires the number of outliers to be larger than at least $D-d$
(according to our numerical experiments it is safe to use the lower bound $1.5 \, (D-d)$).
For practitioners, this is a failure mode of GMS, especially when the dimension of the data set is large (for example, $D>N$).

While this seems to be a strong restriction,
we remark that the problem of exact subspace recovery without knowledge of the intrinsic dimension is rather hard and some
assumptions on data sets or some knowledge of data
parameters would be expected.
Other algorithms for this problem, such as
\citet{Chandrasekaran_Sanghavi_Parrilo_Willsky_2009}, \citet{candes_wright_robust_pca09}, \citet{Xu2010} and \citet{robust_mccoy}, require estimates of unknown
regularization parameters (which often depend on various properties of the data, in particular,
the unknown intrinsic dimension) or strong assumptions on the underlying distribution of the outliers or corrupted elements.

We first note that if only conditions~\eqref{eq:condition1} and~\eqref{eq:condition2} hold,
then Theorem~\ref{thm:recovery} still guarantees that the GMS algorithm outputs a subspace containing the underlying subspace.
Using some information on the data one may recover the underlying subspace from the outputted subspace containing it, even when
dealing with the failure mode.

In the rest of this section we describe several practical solutions for dealing with the failure mode, in particular, with
small number of outliers. We later demonstrate them
numerically  in \S\ref{sec:check_practical} for artificial data and in \S\ref{sec:yale_face} and \S\ref{sec:background_subtraction} for real data.

Our first practical solution is to reduce the ambient dimension of the data. When the reduction is not too aggressive, it can be performed via PCA. In \S\ref{sec:egms} we also propose a robust dimensionality reduction which can be used instead.
There are two problems with this strategy. First of all, the reduced dimension is another parameter that requires tuning.
Second of all, some information may be lost by the dimensionality reduction and thus exact recovery of the underlying subspace is generally impossible.

A second practical solution is to add artificial outliers. The number of added outliers should not be too large
(otherwise \eqref{eq:condition1} and \eqref{eq:condition2} will be violated),
but they should sufficiently permeate through $\reals^D$ so that~\eqref{eq:condition3} holds.
In practice, the number of outliers can be $2D$, since empirically \eqref{eq:condition3} holds with
high probability when $N_0=2(D-d)$. To overcome the possible impact of outliers with arbitrarily large magnitude,
we project the data with artificial outliers onto the sphere (following \citealt{LMTZ2012}).
Furthermore, if the original data matrix does not have full rank (in particular if $N<D$)
we reduce the dimension of the data (by PCA) to be the rank of the data matrix.
This dimensionality reduction clearly does not result in any loss of information.
We refer to the whole process of initial ``lossless dimensionality reduction'' (if necessary),
addition of $2D$ artificial Gaussian outliers, normalization onto the sphere and application of GMS
(with optional estimation of $d$ by the eigenvalues of $\hat{\bQ}$) as the GMS2 algorithm.
We believe that it is the best practical solution to avoid condition~\eqref{eq:condition3}
when $d$ is unknown.

A third solution is to regularize our M estimator, that is, to minimize the following objective function with the regularization parameter $\lambda$:
\begin{equation}\label{eq:regu}
\hat{\bQ}=\argmin_{\tr(\bQ)=1,\bQ=\bQ^T}\sum_{i=1}^N\|\bQ\bx_i\|+\lambda \|\bQ\|_F^2.
\end{equation}
The IRLS algorithm then becomes
\[
\bQ_{k+1}=\left(\sum_{i=1}^{N}\frac{\bx_i\bx_i^T}{\max(\|\bQ_k\bx_i\|,\delta)}+2\lambda\bI\right)^{-1}/\tr\left(\left(\sum_{i=1}^{N}\frac{\bx_i\bx_i^T}{\max(\|\bQ_k\bx_i\|,\delta)}\right)^{-1}+2\lambda\bI\right).
\]
We note that if $\lambda=0$ and there are only few outliers, then in the noiseless case $\dim(\ker(\hat{\bQ}))>d$ and in the small noise case the number of significantly small eigenvalues is bigger than $d$. On the other hand when $\lambda\rightarrow\infty$, $\hat{\bQ}\rightarrow \bI/D$, whose kernel is degenerate (similarly, it has no significantly small eigenvalues). Therefore, there exists an appropriate $\lambda$ for which $\dim(\ker(\hat{\bQ}))$ (or the number of significantly small eigenvalues of $\hat{\bQ}$) is $d$. This formulation transforms the estimation of $d$ into estimation of $\lambda$. This strategy is in line with other common regularized solutions to this problem (see, e.g., \citealt{Chandrasekaran_Sanghavi_Parrilo_Willsky_2009, candes_wright_robust_pca09, Xu2010, robust_mccoy}), however, we find it undesirable to estimate a regularization parameter that is hard to interpret in terms of the data.

\subsection{Subspace Recovery with Knowledge of $d$} \label{sec:theory_know_dim}

Knowledge of the intrinsic dimension $d$ can help improve the performance of GMS or suggest completely new variants (especially as GMS always finds a subspace containing the underlying subspace).
For example, knowledge of $d$ can be used to carefully estimate the parameter $\lambda$ of~\eqref{eq:regu}, for example, by finding $\lambda$ yielding exactly a $d$-dimensional subspace via a bisection procedure.

\citet{LMTZ2012} modified the strategy described in here by requiring an additional constraint on the maximal eigenvalue of $\bQ$ in \eqref{eq:mestimator_function}: $\lambda_{\max}(\bQ)\leq \frac{1}{D-d}$ (where $\lambda_{\max}(\bQ)$ is the largest eigenvalue of $\bQ$).
This approach has theoretical guarantees, but it comes with the price of additional SVD in each iteration, which makes the algorithm slightly more expensive. Besides, in practice
(i.e., noisy setting) this approach requires
tuning the upper bound on $\lambda_{\max}(\bQ)$. Indeed, the solution $\bQ'$ to their minimization problem (with $\lambda_{\max}(\bQ') \leq 1/(D-d)$ and $\tr(\bQ')=1$) satisfies that
$\dim(\ker(\bQ')$ is at most $d$ and equals $d$ when $\bQ'$ is a scaled projector operator. They proved that $\dim(\ker(\bQ') = d$ for the setting of pure inliers (lying exactly on a subspace) under some conditions avoiding the three types of enemies. However, in practice (especially in noisy cases) the actual subspace often has dimension smaller than $d$ and thus the bound on $\lambda_{\max}(\bQ)$ has to be tuned as an additional parameter. In some cases, one may take $\lambda_{\max}(\bQ)>\frac{1}{D-d}$ and find the subspace according to the bottom $d$ eigenvectors. In other cases, a bisection method on the bound of $\lambda_{\max}(\bQ)$ provide more accurate results (see related discussion in \citet[\S6.1.6]{LMTZ2012}).

\subsubsection{The EGMS Algorithm} \label{sec:egms}

We  formulate in Algorithm~\ref{alg:rmL} the Extended Geometric Median Subspace (EGMS) algorithm for subspace recovery with known intrinsic dimension.
\begin{algorithm}[htbp]
\caption{The Extended Geometric Median Subspace Algorithm} \label{alg:rmL}
\begin{algorithmic}
\REQUIRE $\sX=\{\bx_i\}_{i=1}^N \subseteq
\mathbb{R}^{D}$: data, $d$: dimension of $\rmL^*$, an algorithm for minimizing~\eqref{eq:symmetric}\\
\ENSURE $\hat{\rmL}$: a $d$-dimensional linear subspace in $\reals^D$.\\
\textbf{Steps}:
\STATE
     $\bullet$  $\hat{\rmL}=\reals^D$\\
     %$\bullet$ $k=-1$\\
    \REPEAT \STATE
     %$\bullet$ k=k+1\\
     $\bullet$  $\hat{\bQ}=\argmin_{\bQ\in\bbH,\bQ\bP_{\hat{\rmL}^\perp}=\b0}F(\bQ)$\\
     $\bullet$  $\bu=$ \ the top eigenvector of $\hat{\bQ}$\\%, and find the smallest number $i$ such that $\sigma_{i+1}(\hat{\bQ})=0$ (theoretically), or  $\sigma_{i}(\hat{\bQ})>100\sigma_{i+1}(\hat{\bQ})$ (empirically)\\
     $\bullet$ $\hat{\rmL}=\hat{\rmL}\cap \Sp(\bu^\perp)$\\
     \UNTIL $\dim(\hat{\rmL})=d$\\
\end{algorithmic}
\end{algorithm}

We justify this basic procedure in the noiseless case without requiring~\eqref{eq:condition3} as follows.
\begin{theorem}\label{thm:rmL}
Assume that $d,D \in \nats$, $d<D$, $\sX$ is a data set in $\reals^D$ and $\rmL^*\in\GDd$.
If only conditions~\eqref{eq:condition1} and \eqref{eq:condition2} hold,
then the EGMS Algorithm exactly recovers $\rmL^*$.
\end{theorem}

In \S\ref{sec:info_evd} we show how the vectors obtained by EGMS at each iteration can be used to form robust principal components (in reverse order), even when $\hat{\bQ}$ is degenerate.

\subsection{Computational Complexity of GMS and EGMS}
\label{sec:egms_complexity}

The computational complexity of GMS is of the same order as that of Algorithm~\ref{alg:practical},
that is, $O(n_s\cdot N\cdot D^2)$ (where $n_s$ is the number of required iterations for Algorithm~\ref{alg:practical}).
Indeed, after obtaining $\hat{\bQ}$, computing $\rmL^*$ by its smallest $d$ eigenvectors takes an order of $O(d\cdot D^2)$ operations.

EGMS on the other hand repeats Algorithm~\ref{alg:practical} $D-d$ times; therefore it adds an order of $O((D-d)\cdot n_s\cdot N\cdot D^2)$ operations,
where $n_s$ denotes the total number of iterations for Algorithm~\ref{alg:practical}.
In implementation, we can speed up the EGMS algorithm by excluding the span of some of the top eigenvectors
of $\hat{\bQ}$ from $\hat{\rmL}$ (instead of excluding only the top eigenvector in the third step of Algorithm~\ref{alg:rmL}).
We demonstrate this modified procedure on artificial setting in \S\ref{sec:check_practical}.

\section{Numerical Experiments}\label{sec:numerics}

We compare our proposed estimator to other algorithms, while using both synthetic and real data. We also demonstrate the effectiveness of some of our
practical proposals.
In \S\ref{sec:synthetic} we describe a model for generating synthetic data. Using this model, we respectively demonstrate in
\S\ref{sec:check_practical}-\S\ref{sec:test_delta} the effectiveness of the following strategies: the practical solutions of \S\ref{sec:theory_no_dim} and \S\ref{sec:theory_know_dim}, our estimation of the subspace dimension, and our regularization (more precisely, its effect on the recovery error).
In \S\ref{sec:info_evd} we demonstrate the use of our M estimator for robust estimation of eigenvectors of the covariance (or the inverse covariance) matrix.
At last, actual comparisons are demonstrated in \S\ref{sec:detailed_synthetic}-\S\ref{sec:background_subtraction} for synthetic data, face data and video surveillance data respectively.

%NEED TEXT HERE

\subsection{Model for Synthetic Data}\label{sec:synthetic}
In \S\ref{sec:check_practical}-\S\ref{sec:test_delta} and \S\ref{sec:detailed_synthetic} we generate data from the following model.
We randomly choose $\rmL^* \in \GDd$, sample $N_1$ inliers from the $d$-dimensional Multivariate Normal
distribution $N(\b0,\bI_{d \times d})$ on $\rmL^*$ and add $N_0$ outliers sampled from a uniform distribution on $[0,1]^D$.
The outliers are strongly asymmetric around the subspace to make the subspace recovery problem more difficult \citep{lp_recovery_part1_11}.
In some experiments below additional Gaussian noise is considered.
When referring to this synthetic data we only need to specify its parameters $N_1$, $N_0$, $D$, $d$
and possibly the standard deviation for the additive noise.
For any subspace recovery  algorithm (or heuristics), we denote by  $\tilde{\rmL}$ its output
(i.e., the estimator for $\rmL^*$) and measure the corresponding recovery error by $e_{\tilde{\rmL}} = \|\bP_{\tilde{\rmL}}-\bP_{\rmL^*}\|_F$.

\subsection{Demonstration of Practical Solutions of \S\ref{sec:theory_no_dim} and \S\ref{sec:theory_know_dim}}
\label{sec:check_practical}

We present two different artificial cases, where
in one of them condition~\eqref{eq:condition3} holds and in the other one it does not hold and test the
practical solutions of \S\ref{sec:theory_no_dim} and \S\ref{sec:theory_know_dim}
in the second case.

The two cases are the following instances
of the synthetic model of \S\ref{sec:synthetic}: (a) $(N_1,N_0,D,d)$ $=(100$ , $100$, $100$, $20)$ and (b) $(N_1,N_0,D,d)=(100,20,100,20)$.
The GMS algorithm estimates the underlying subspace $\rmL^*$ given $d=20$ with recovery errors $2.1\times 10^{-10}$ and $3.4$ in cases (a) and (b) respectively. In case (a) there are sufficiently many outliers (with respect to $D-d$) and the GMS algorithm is successful.
We later show in \S\ref{sec:dim} that the underlying dimension ($d=20$) can be easily estimated by the eigenvalues of $\hat{\bQ}$.
In case (b) $N_0 = 0.25*(D-d)$, therefore, condition~\eqref{eq:condition3} is violated
and the GMS algorithm completely fails.

We demonstrate the success of the practical solutions of \S\ref{sec:theory_no_dim} and \S\ref{sec:theory_know_dim}
in case (b). We assume that the dimension $d$ is known, though in \S\ref{sec:dim} we estimate $d$ correctly for the non-regularized solutions of \S\ref{sec:theory_no_dim}. Therefore, these solutions can be also applied without knowing the dimension.
If we reduce the dimension of the data set in case (b) from $D=100$ to $D=35$ (via PCA; though one can also use EGMS), then GMS (with $d=20$) achieves a recovery error of $0.23$, which indicates that GMS almost recovers the subspace correctly.
We remark though that if we reduce the dimension to, for example, $D=55$, then the GMS algorithm will still fail. We also note that the recovery error is not as attractive as the ones below;
this observation probably indicates that some information was lost during the dimension reduction.

The GMS2 algorithm with $d=20$ recovers the underlying subspace in case (b) with error $1.2\times 10^{-10}$.
This is the method we advocated for when possibly not knowing the intrinsic dimension.

The regularized minimization of~\eqref{eq:regu} with $\lambda=100$ works well for case (b).
In fact, it recovers the subspace as $\ker{\hat{\bQ}}$ (without using its underlying dimension)
with error $3.3\times 10^{-13}$.
The only issue is how to determine the value of $\lambda$.
We claimed in \S\ref{sec:theory_know_dim} that if $d$ is known, then $\lambda$ can be carefully estimated by the bisection method.
This is true for this example, in fact, we initially chose $\lambda$ this way.

We remark that the REAPER algorithm of \citet{LMTZ2012} did not perform well for this particular data, though in general it is a very successful solution.
The recovery error of the direct REAPER algorithm was 3.725 (and 3.394 for S-REAPER) and the error for its modified version via bisection
(relaxing the bound on the largest eigenvalue so that $\dim(\ker(\hat{\bQ}))=20$) was 3.734 (and 3.175 for S-REAPER).

At last we demonstrate the performance of EGMS and its faster heuristic with $d=20$. The recovery error of the original EGMS for case (b) is only $0.095$. We suggested in \S\ref{sec:egms_complexity} a faster heuristic for EGMS, which can be reformulated as follows: In the third step of Algorithm~\ref{alg:rmL}, we replace $\bu$ (the top eigenvector of $\hat{\bQ}$) with $\bU$, the subspace spanned by several top eigenvectors. In the noiseless case, we could let $\bU$ be the span of the nonzero eigenvectors of $\hat{\bQ}$. This modification of EGMS (for the noiseless case) required only two repetitions of
Algorithm~\ref{alg:practical} and its recovery error was $2.2\times 10^{-13}$. In real data sets with noise we need to determine the number of top eigenvectors spanning $\bU$, which makes this modification of EGMS less automatic.

\subsection{Demonstration of Dimension Estimation}
\label{sec:dim}

We test dimension estimation by eigenvalues of $\hat{\bQ}$ for cases (a) and (b) of \S\ref{sec:check_practical}.
The eigenvalues of $\hat{\bQ}$ obtained by Algorithm~\ref{alg:practical} for the two cases are shown in Figure~\ref{fig:svd}.
\begin{figure}
\begin{center}
\includegraphics[width=.45\textwidth,height=.3\textwidth]{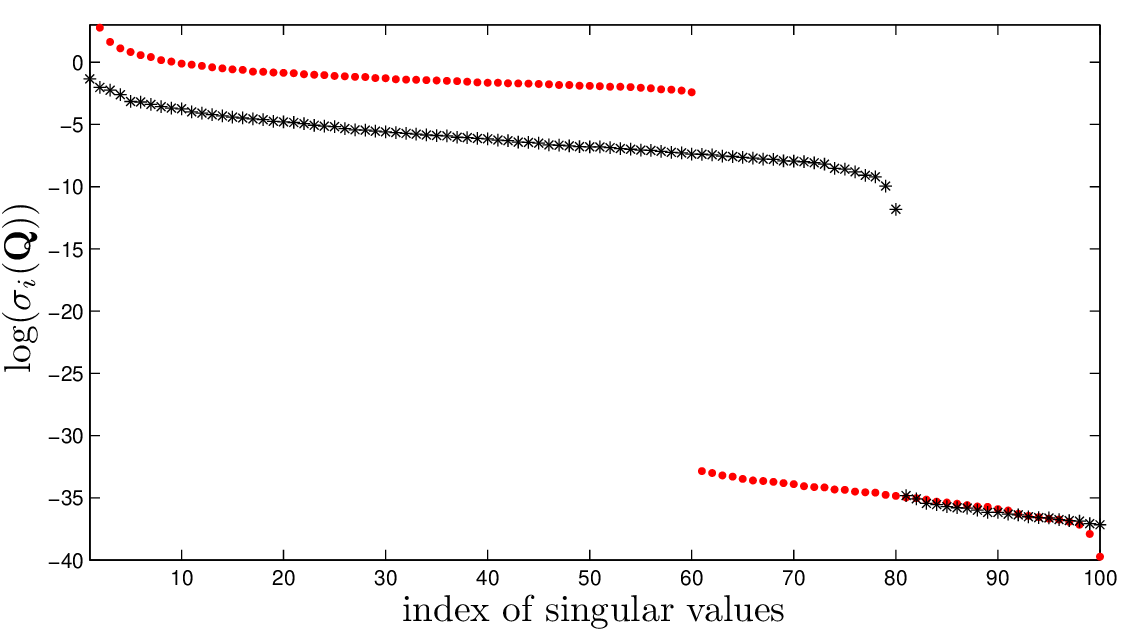}
\includegraphics[width=.45\textwidth,height=.3\textwidth]{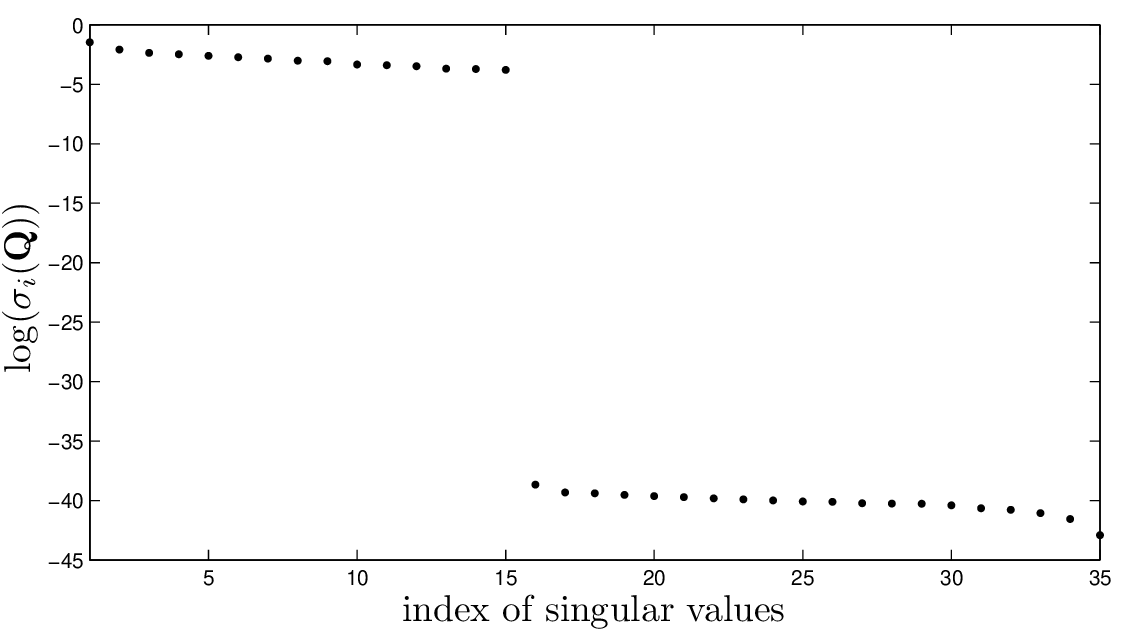}
\caption{Dimension estimation: In the left figure, the starred points and the dotted point represent log-scaled eigenvalues of the output of Algorithm~\ref{alg:practical} for cases (a) and (b) respectively (see \S\ref{sec:dim}). The right figure corresponds to case (b) with dimension reduced to $35$. \label{fig:svd}}
\end{center}
\end{figure}
In case (a), the largest logarithmic eigengap (i.e., the largest gap in logarithms of eigenvalues) occurs at $80$, so we can correctly estimate that $d= D-80 = 20$ (the eigenvalues are not zero since Algorithm~\ref{alg:practical} uses the $\delta$-regularized objective function). However, in case (b) the largest eigengap occurs at $60$ and thus mistakenly predicts $d=40$.

As we discussed in \S\ref{sec:check_practical}, the dimension estimation
fails here since condition~\eqref{eq:condition3} is not satisfied. However, we have verified that if we try any of the solutions
proposed in \S\ref{sec:theory_no_dim} then we can correctly recover that $d=20$ by the logarithmic eigengap. For example, in Figure~\ref{fig:svd} we demonstrate the logarithms of eigenvalues of $\hat{\bQ}$ in case (b) after dimensionality reduction (via PCA) onto dimension $D=35$ and it is clear that the largest gap is at $d=20$ (or $D-d=80$). We obtained similar graphs when using $2 D$ artificial outliers (more precisely, the GMS2 algorithm without the final application of the GMS algorithm) or the regularization of \eqref{eq:regu} with $\lambda = 100$.

\subsection{The Effect of the Regularization Parameter $\delta$}
\label{sec:test_delta}
We assume a synthetic data set sampled according to the model of \S\ref{sec:synthetic} with $(N_1,N_0,D,d)=(250,250,100,10)$.
We use the GMS algorithm with $d=10$ and different values of the regularization parameter $\delta$ and record the recovery error in Figure~\ref{fig:F_delta}.
For $10^{-14}\leq \delta \leq 10^{-2}$, $\log(\text{error})-\log(\delta)$ is constant. We thus empirically obtain that the error is of order
$O(\delta)$ in this range.
On the other hand, \eqref{eq:subspace_diff2} only obtained an order of $O(\sqrt{\delta})$. It is possible that methods similar to those of
\citet{Coudron_Lerman2012} can obtain sharper error bounds.
We also expect that for $\delta$ sufficiently small (here smaller than $10^{-14}$), the rounding error becomes dominant.
On the other hand, perturbation results are often not valid for sufficiently large $\delta$ (here this is the case for $\delta\ > 10^{-2}$).

\begin{figure}[htbp]
\begin{center}
\includegraphics[width=.4\textwidth,height=.3\textwidth]{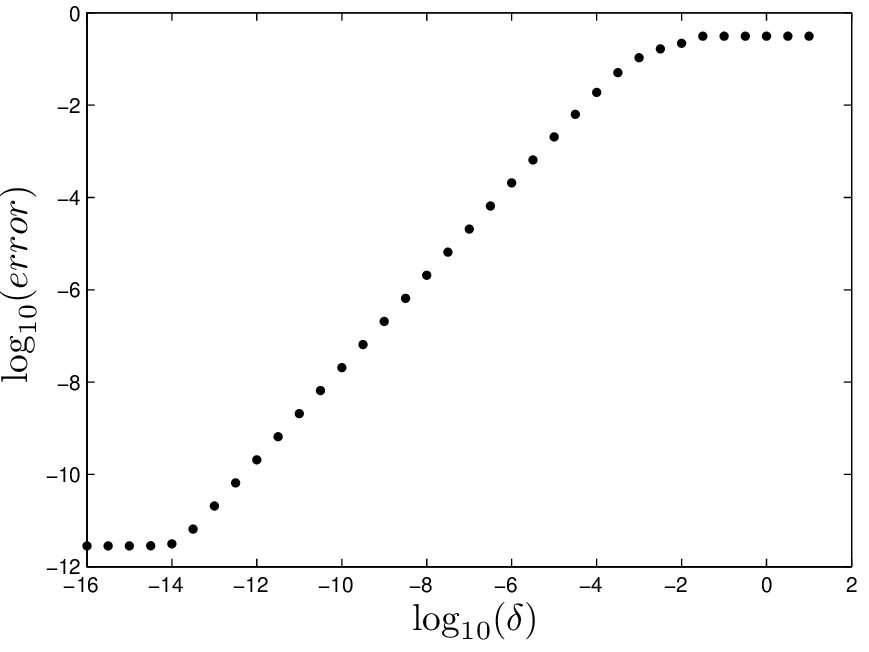}
\caption{The recovery errors and the regularization parameters $\delta$\label{fig:F_delta}}
\end{center}
\end{figure}

\subsection{Information Obtained from Eigenvectors}
\label{sec:info_evd}
Throughout the paper we emphasized the subspace recovery problem, but did not discuss at all the information that can be inferred from the eigenvectors of our robust PCA strategy. Since in standard PCA these vectors have significant importance, we exemplify the information obtained from our robust PCA and compare it to that obtained from PCA and some other robust PCA algorithms.

We create a sample from a mixture of two Gaussian distributions with the same mean and same eigenvalues of the covariance matrices, but different eigenvectors of the covariance matrices. The mixture percentages are $25\%$ and $75\%$. We expect the eigenvectors of any good robust PCA algorithm
(robust to outliers as perceived in this paper) to be close to that of the covariance of the main component (with $75\%$).

More precisely, we sample $300$ points from $N(\b0,\bSigma_1)$, where $\bSigma_1$ is a $10 \times 10$ diagonal matrix with elements $1,2^{-1},2^{-2},\cdots,2^{-9}$ and $100$ points from $N(\b0,\bSigma_2)$, where $\bSigma_2 = \bU \bSigma_1 \bU^T$, where $\bU$ is randomly chosen from the set of all orthogonal matrices in $\mathbb{R}^{10\times 10}$. The goal is to estimate the eigenvectors of $\bSigma_1$
(i.e., the standard basis vectors in $\reals^{10}$) in the presence of $25 \%$ ``outliers''. Unlike the subspace recovery problem, where we can expect to exactly recover a linear structure among many outliers, here the covariance structure is more complex and we cannot exactly recover it with $25 \%$ outliers.

We estimated the eigenvectors of $\bSigma_1$ by the
the eigenvectors of $\hat{\bQ}$ of Algorithm~\ref{alg:practical} in reverse order (recall that $\hat{\bQ}$  is a scaled and robust version of the inverse covariance). We refer to this procedure as ``EVs (eigenvectors) of $\hat{\bQ}^{-1}$''. We also estimated these eigenvectors by standard PCA, LLD \citep{robust_mccoy} with $\lambda = 0.8\sqrt{D/N}$  and PCP \citep{candes_wright_robust_pca09} with $\lambda = 1/\sqrt{\max(D,N)}$.
We repeated the random simulation (with different samples for the random orthogonal matrix $\bU$) 100 times and reported  in Table~\ref{tab:eigen1} the average angles between the estimated and actual top two eigenvectors of $\bSigma_1$ according to the different methods. We note that the ``EVs  of $\hat{\bQ}^{-1}$'' outperforms PCA, LLD (or OP) and PCP in terms of estimation of the top two eigenvectors of $\bSigma_1$. We remark though that PCP does not suit for robust estimation of the empirical covariance and thus the comparison is unfair for PCP.
\begin{table*}[htbp]
\centering
\begin{tabular}{|c|c|c|c|c|}
\hline
&EVs of $\hat{\bQ}^{-1}$ &LLD & PCP & PCA\\
\hline
Eigenvector 1& $3.0^\circ$ &$5.5^\circ$&$45.7^\circ$&$14.8^\circ$\\
Eigenvector 2& $3.0^\circ$ &$5.5^\circ$&$47.4^\circ$&$40.3^\circ$\\
\hline
\end{tabular}
\caption{{Angles (in degrees) between the estimated and actual top two eigenvectors of $\bSigma_1$. \label{tab:eigen1}}}
\end{table*}

When the covariance matrix $\bSigma_1$ (and consequently also $\bSigma_2$) is degenerate, $\hat{\bQ}$ might be singular and therefore $\hat{\bQ}$ cannot be directly used to robustly estimate eigenvectors of the covariance matrix.
For this case, EGMS (Algorithm~\ref{alg:rmL}) can be used, where the vector $\bu$ obtained in the $i$th iteration of Algorithm~\ref{alg:rmL} can be considered as the $(D-i+1)$st robust eigenvector (that is, we reverse the order again).
To test the performance of this method, we modify $\bSigma_1$ in the above model as follows:
$\bSigma_1$=$\diag(1,0.5,0.25,0,0,\cdots,0)$.
We repeated the random simulations of this modified model 100 times and reported  in Table~\ref{tab:eigen1} the average angles between the estimated and actual top two eigenvectors of $\bSigma_1$ according to the different methods.
Here LLD did slightly better than EGMS and they both outperformed PCA (and PCP).
\begin{table*}[htbp]
\centering
\begin{tabular}{|c|c|c|c|c|}
\hline
&EGMS&LLD  & PCP & PCA\\
\hline
Eigenvector 1& $5.2^\circ$ &$3.4^\circ$&$42.6^\circ$&$8.2^\circ$\\
Eigenvector 2& $5.2^\circ$ &$3.4^\circ$&$47.3^\circ$&$16.1^\circ$\\
\hline

\end{tabular}
\caption{{Angles (in degrees) between the estimated and actual top two eigenvectors of $\bSigma_1$. \label{tab:eigen2}}}
\end{table*}

\subsection{Detailed Comparison with Other Algorithms for Synthetic Data}\label{sec:detailed_synthetic}

Using the synthetic data of \S\ref{sec:synthetic}, we compared the GMS algorithm with the following algorithms:
MDR (Mean Absolute Deviation Rounding) of \citet{robust_mccoy}, LLD (Low-Leverage Decomposition) of \citet{robust_mccoy},
OP (Outlier Pursuit) of  \citet{Xu2010}, PCP (Principal Component Pursuit) of \citet{candes_wright_robust_pca09},
MKF (Median $K$-flats with $K=1$) of  \citet{MKF_workshop09}, HR-PCA (High-dimensional Robust PCA)
of \citet{Xu2010_highdimensional}, a common M-estimator \citep[see, e.g.,]{huber_book} and $R_1$-PCA of \citet{Ding+06}.
The codes of OP and HR-PCA were obtained from \url{http://guppy.mpe.nus.edu.sg/~mpexuh},
the code of MKF from \url{http://www.math.umn.edu/~zhang620/mkf}, the code of PCP
from \url{http://perception.csl.illinois.edu/matrix-rank/sample_code.html}
with the Accelerated Proximal Gradient and full SVD version, the codes of MDR and LLD from
\url{http://www.acm.caltech.edu/~mccoy/code/} and the codes of the common M-estimator, $R_1$-PCA and GMS will appear in a supplemental webpage.
We also record the output of standard PCA, where we recover the subspace by the span of the top $d$ eigenvectors.
We ran the experiments on a computer with Intel Core 2 CPU at 2.66GHz and 2 GB
memory.

We remark that since the basic GMS algorithm already performed very well on these artificial instances, we did not test its extensions and modifications described in \S\ref{sec:prac_solutions} (e.g., GMS2 and EGMS).

For all of our experiments with synthetic data,
we could correctly estimate $d$ by the largest logarithmic eigengap of the output of Algorithm~\ref{alg:practical}.
Nevertheless, we used the knowledge of $d$ for all algorithms for the sake of fair comparison.

For LLD, OP and PCP we estimated $\rmL^*$ by the span of the top $d$ eigenvectors of the low-rank matrix.
Similarly, for the common M-estimator we used the span of the top $d$ eigenvectors of the estimated covariance $\bA$.
For the HR-PCA algorithm we also used the true percentage of outliers ($50\%$ in our experiments).
For LLD, OP and PCP we set the mixture parameter $\lambda$ as $0.8\sqrt{D/N}, 0.8\sqrt{D/N}, 1/\sqrt{\max(D,N)}$ respectively (following the suggestions
of \citet{robust_mccoy} for LLD/OP and \citet{candes_wright_robust_pca09} for PCP).
These choices of parameters are also used in experiments with real data sets in \S\ref{sec:yale_face} and \S\ref{sec:background_subtraction}.

For the common M-estimator, we used $u(x)=2\max(\ln(x)/x,10^{30})$ and the algorithm discussed by  \citet{Kent1991}. Considering the conditions in \S\ref{sec:m_estimator}, we also tried other functions:  $u(x)=\max(x^{-0.5},10^{30})$ had a significantly larger recovery error and $u(x)=\max(x^{-0.9},10^{30})$ resulted in a similar recovery error as $\max(\ln(x)/x,10^{30})$ but a double running time.

We used the syntectic data with different values of $(N_1,N_0,D,d)$. In some instances we also add noise from the Gaussian distribution $N(0,\eta^2\bI)$ with $\eta=0.1$ or $0.01$.  We repeated each experiment 20 times (due to the random generation of data).
We record in Table~\ref{tab:synthetic} the mean running time, the mean recovery error and their standard deviations.

We remark that PCP is designed for uniformly corrupted coordinates of data, instead of corrupted data points (i.e., outliers), therefore,
the comparison with PCP is somewhat unfair for this kind of data. On the other hand, the applications in~\S\ref{sec:yale_face} and~\S\ref{sec:background_subtraction} are tailored for the PCP model (though the other algorithms still apply successfully to them).

From Table~\ref{tab:synthetic} we can see that GMS is the fastest robust algorithm. Indeed, its running time is comparable to that of PCA. We note that this is due to its linear convergence rate (usually it converges in less than $40$ iterations). The common M-estimator is the closest algorithm in terms of running time to GMS, since it also has the linear convergence rate. In contrast, PCP, OP and LLD need a longer running time since their convergence rates are much slower.
Overall, GMS performs best in terms of exact recovery. The PCP, OP and LLD algorithms cannot approach exact recovery even by tuning the parameter $\lambda$.
For example, in the case where $(N_1,N_0,D,d) = (125,125,10,5)$ with $\eta=0$, we checked a geometric sequence of $101$ $\lambda$ values from $0.01$ to $1$, and the smallest recovery errors for LLD, OP and PCP are $0.17$, $0.16$ and $0.22$ respectively.
The common M-estimator performed very well for many cases (sometimes slightly better than GMS), but its performance deteriorates as the density of outliers increases (e.g.,  poor performance for the case where $(N_1,N_0,D,d) = (125,125,10,5)$). Indeed, Theorem~\ref{thm:m_estimator} indicates problems with the exact recovery of the common M-estimator.

At last, we note that the empirical recovery error of the GMS algorithm for noisy data sets is in the order of $\sqrt{\eta}$, where $\eta$ is the size of noise.

\begin{table*}[htbp]
\centering
\footnotesize
\resizebox{!}{3.6in}{
\begin{tabular}{|@{}c@{}|@{}c@{}||@{}c@{}|@{}c@{}|@{}c@{}|@{}c@{}|@{}c@{}|@{}c@{}|@{}c@{}|@{}c@{}|@{}c@{}|@{}c@{}||}
\hline
{$(N_1,N_0,D,d)$ }& & \,\,GMS\,\, & MDR & LLD &  OP & PCP & HR-PCA&\,\,MKF\,\,&\,\,PCA\,\,&\,M-est.\,&$R_1$-PCA  \\ \hline

{}& $e$ &6e-11&0.275&1.277&0.880&0.605&0.210&0.054&0.193&0.102&0.121\\
{$(125,125,10,5)$}& $std.e$&4e-11&0.052&0.344&0.561&0.106&0.049&0.030&0.050&0.037&0.048\\
{$\eta=0$}& $t(s)$ & 0.008&0.371&0.052&0.300&0.056&0.378&0.514&0.001&0.035&0.020\\
&$std.t $ &0.002&0.120&0.005&0.054&0.002&0.001&0.262&8e-06&4e-04&0.014\\
\hline

{}& $e$ &0.011&0.292&1.260&1.061&0.567&0.233&0.069&0.213&0.115&0.139\\
{$(125,125,10,5)$}& $std.e$&0.004&0.063&0.316&0.491&0.127&0.075&0.036&0.073&0.054&0.073\\
{$\eta=0.01$}& $t(s)$ & 0.008&0.340&0.053&0.287&0.056&0.380&0.722&0.001&0.035&0.052\\
&$std.t $ &0.001&0.075&0.007&0.033&0.001&0.009&0.364&1e-05&4e-04&0.069\\
\hline

{}& $e$ &0.076&0.264&1.352&0.719&0.549&0.200&0.099&0.185&0.122&0.128\\
{$(125,125,10,5)$}& $std.e$&0.023&0.035&0.161&0.522&0.102&0.051&0.033&0.048&0.041&0.050\\
{$\eta=0.1$}& $t(s)$ & 0.007&0.332&0.055&0.301&0.056&0.378&0.614&0.001&0.035&0.032\\
&$std.t $ &0.001&0.083&0.004&0.044&0.001&0.001&0.349&7e-06&4e-04&0.037\\
\hline
\hline
{}& $e$ &2e-11&0.652&0.258&0.256&0.261&0.350&0.175&0.350&1e-12&0.307\\
{$(125,125,50,5)$}& $std.e$&3e-11&0.042&0.030&0.032&0.033&0.023&0.028&0.025&5e-12&0.029\\
{$\eta=0$}& $t(s)$ & 0.015&0.420&0.780&1.180&3.164&0.503&0.719&0.006&0.204&0.020\\
&$std.t $ &0.001&0.128&0.978&0.047&0.008&0.055&0.356&9e-05&0.001&0.011\\
\hline

{}& $e$ &0.061&0.655&0.274&0.271&0.273&0.355&0.196&0.359&0.007&0.321\\
{$(125,125,50,5)$}& $std.e$&0.009&0.027&0.039&0.038&0.040&0.038&0.038&0.033&0.001&0.038\\
{$\eta=0.01$}& $t(s)$ & 0.023&0.401&4.155&1.506&0.499&0.653&0.656&0.006&0.191&0.028\\
&$std.t $ &0.002&0.079&0.065&0.197&0.006&0.044&0.377&8e-05&0.001&0.022\\
\hline

{}& $e$ &0.252&0.658&0.292&0.290&0.296&0.358&0.264&0.363&0.106&0.326\\
{$(125,125,50,5)$}& $std.e$&0.027&0.033&0.032&0.032&0.033&0.027&0.031&0.032&0.014&0.032\\
{$\eta=0.1$}& $t(s)$ & 0.021&0.363&0.923&1.726&0.501&0.638&0.641&0.006&0.191&0.025\\
&$std.t $ &0.001&0.063&0.033&0.470&0.009&0.051&0.240&1e-04&0.001&0.012\\
\hline
\hline
{}& $e$ &3e-12&0.880&0.214&0.214&0.215&0.332&0.161&0.330&2e-12&0.259\\
{$(250,250,100,10)$}& $std.e$&2e-12&0.018&0.019&0.019&0.019&0.014&0.024&0.012&9e-12&0.016\\
{$\eta=0$}& $t(s)$ & 0.062&1.902&3.143&7.740&2.882&1.780&1.509&0.039&0.819&1.344\\
&$std.t $ &0.006&0.354&4.300&0.038&0.014&0.041&1.041&3e-04&0.023&0.708\\
\hline

{}& $e$ &0.077&0.885&0.217&0.216&0.219&0.334&0.164&0.335&0.009&0.263\\
{$(250,250,100,10)$}& $std.e$&0.006&0.031&0.019&0.018&0.020&0.019&0.019&0.017&3e-04&0.018\\
{$\eta=0.01$}& $t(s)$ & 0.084&1.907&21.768&11.319&2.923&1.785&1.412&0.039&0.400&1.086\\
&$std.t $ &0.010&0.266&0.261&0.291&0.014&0.041&0.988&3e-04&0.002&0.738\\
\hline

{}& $e$ &0.225&0.888&0.238&0.237&0.262&0.342&0.231&0.345&0.136&0.276\\
{$(250,250,100,10)$}& $std.e$&0.016&0.020&0.019&0.019&0.019&0.019&0.018&0.015&0.010&0.019\\
{$\eta=0.1$}& $t(s)$ & 0.076&1.917&4.430&16.649&2.876&1.781&1.555&0.039&0.413&1.135\\
&$std.t $ &0.007&0.299&0.069&1.184&0.014&0.025&0.756&4e-04&0.011&0.817\\
\hline
\hline
{}& $e$ &4e-11&1.246&0.162&0.164&0.167&0.381&0.136&0.381&3e-13&0.239\\
{$(500,500,200,20)$}& $std.e$&1e-10&0.018&0.011&0.011&0.011&0.010&0.009&0.008&6e-14&0.009\\
{$\eta=0$}& $t(s)$ & 0.464&23.332&16.778&89.090&16.604&8.602&5.557&0.347&6.517&15.300\\
&$std.t $ &0.024&2.991&0.878&1.836&0.100&0.216&4.810&0.009&0.126&3.509\\
\hline

{}& $e$ &0.082&1.247&0.160&0.162&0.166&0.374&0.139&0.378&0.012&0.236\\
{$(500,500,200,20)$}& $std.e$&0.003&0.018&0.007&0.007&0.008&0.011&0.010&0.006&2e-04&0.007\\
{$\eta=0.01$}& $t(s)$ &0.592&23.214&128.51&122.61&16.823&8.541&6.134&0.354&2.361&15.165\\
&$std.t $ &0.060&3.679&1.155&6.500&0.036&0.219&4.318&0.019&0.064&3.485\\
\hline
{}& $e$ &0.203&1.262&0.204&0.204&0.250&0.391&0.275&0.398&0.166&0.270\\
{$(500,500,200,20)$}& $std.e$&0.007&0.012&0.007&0.007&0.007&0.012&0.272&0.009&0.005&0.008\\
{$\eta=0.1$}& $t(s)$ &0.563&24.112&24.312&202.22&16.473&8.552&8.745&0.348&2.192&15.150\\
&$std.t $ &0.061&2.362&0.226&8.362&0.050&0.155&3.408&0.010&0.064&3.420\\
\hline
\hline

\end{tabular}
}
\caption{{Mean running times, recovery errors and their standard deviations for synthetic data.}}\label{tab:synthetic}
\end{table*}

\subsection{Yale Face data}\label{sec:yale_face}
Following \citet{candes_wright_robust_pca09}, we apply our algorithm to face images.
It has been shown that face images from the same person lie in a low-dimensional linear subspace of dimension at most $9$ \citep{Basri03}. However, cast shadows, specular reflections and saturations could possibly distort this low-rank modeling. Therefore, one can use a good robust PCA algorithm to remove these errors if one has many images from the same face.

We used the images of the first two persons in the extended Yale face database B \citep{KCLee05}, where each of them has $65$ images of size $192\times 168$ under different illumination conditions. Therefore we represent each person by $65$ vectors of length $32256$.
Following \citet{Basri03} we applied GMS, GMS2 and EGMS with $d=9$ and we also reduced the $65\times 32256$ matrix to $65 \times 65$ (in fact, we only reduced the representation of the column space) by rejecting left vectors with zero singular values.
We also applied the GMS algorithm after initial dimensionality reduction (via PCA) to $D=20$. The running times of EGMS and GMS (without dimensionality reduction) are $13$ and $0.16$ seconds respectively on average for each face (we used the same computer as in \S\ref{sec:detailed_synthetic}).
On the other hand, the running times of PCP and LLD are $193$ and $2.7$ seconds respectively. Moreover, OP ran out of memory.  The recovered images are shown in Figure~\ref{fig:yale_face}, where the shadow of the nose and the parallel lines were removed best by EGMS. The GMS algorithm without dimension reduction did not perform well, due to the difficulty explained in \S\ref{sec:prac_solutions} and demonstrated in \S\ref{sec:check_practical}.
The GMS2 algorithm turns out to work well, except for the second image of face 2. However, other algorithms such as PCP and GMS with dimension reduction ($D = 20$) performed even worse on this image and LLD did not remove any shadow at all; the only good algorithm for this image is EGMS.

\begin{figure}[htbp]
\begin{center}
\includegraphics[width=.7\textwidth,height=.35\textwidth]{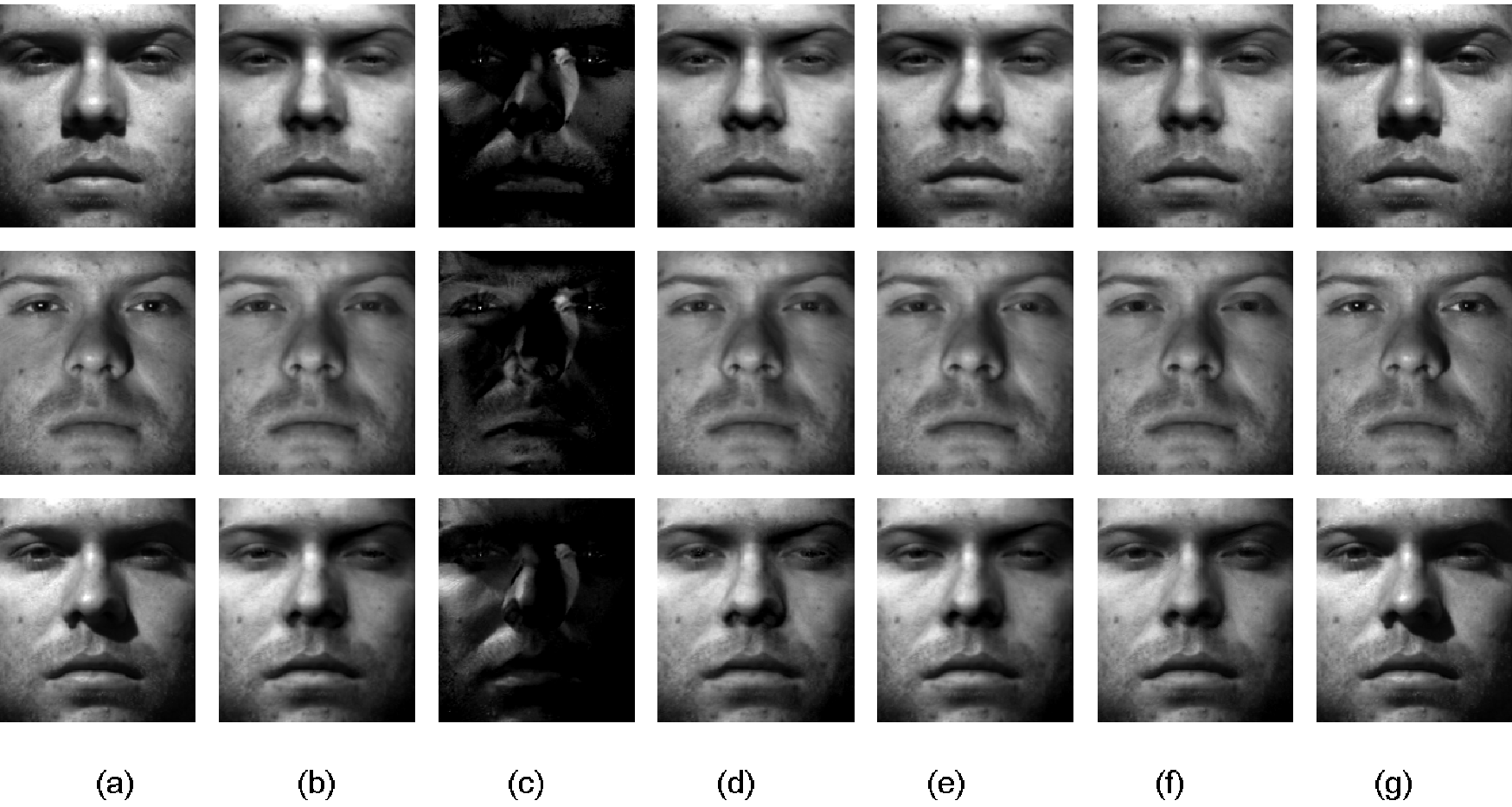}
\\
\includegraphics[width=.7\textwidth,height=.35\textwidth]{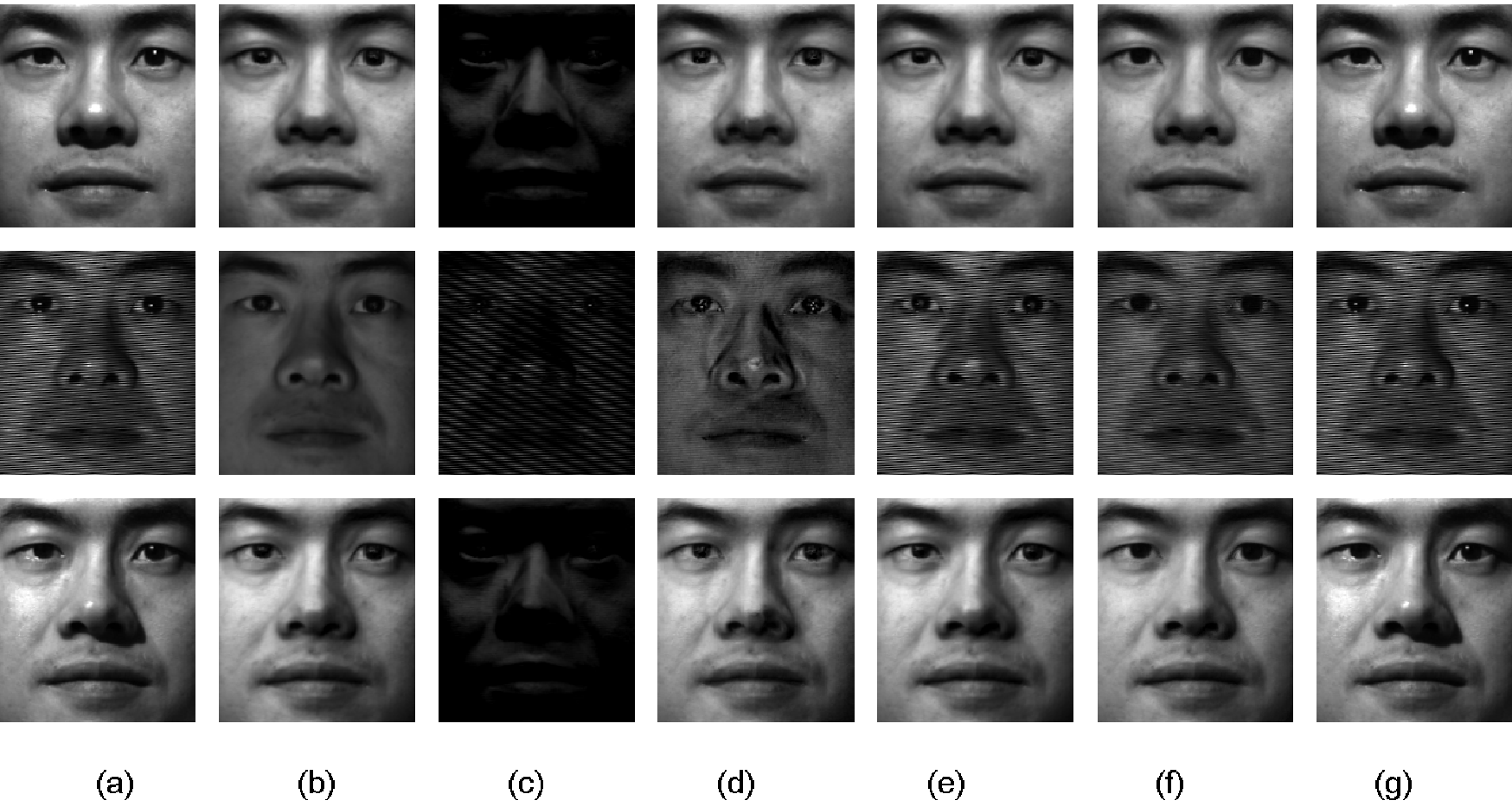}
\caption{Recovering faces: (a) given images, (b)-(f) the recovered images by EGMS, GMS without dimension reduction, GMS2, GMS with dimension reduced to $20$, PCP and LLD respectively\label{fig:yale_face}}
\end{center}
\end{figure}

\subsection{Video Surveillance}\label{sec:background_subtraction}

For background subtraction in surveillance videos \citep{Li_backgroundsubtraction}, we consider the following two videos used by \citet{candes_wright_robust_pca09}: ``Lobby in an office building with switching on / off lights'' and
``Shopping center'' from \url{http://perception.i2r.a-star.edu.sg/bk_model/bk_index.html}. In the first video, the resolution is $160 \times 128$ and we used $1546$ frames from `SwitchLight1000.bmp' to `SwitchLight2545.bmp'. In the second video, the resolution is $320 \times 256$ and we use $1000$ frames from `ShoppingMall1001.bmp' to `ShoppingMall2000.bmp'. Therefore, the data matrices are of size $1546\times 20480$ and $1001\times 81920$. We used a computer with Intel Core 2 Quad Q6600 2.4GHz and 8 GB memory due to the large size of these data.

We applied GMS, GMS2 and EGMS with $d=3$ and with initial dimensionality reduction to $200$ to reduce running time. %(EGMS also worked well with the negligible reduction to $D'$ such that $\rank(\bX)=D'$, but $D=200$ reduced its running time).
For this data we are unaware of a standard choice of $d$; though we noticed empirically that the outputs of our algorithms as well as other algorithms
are very stable to changes in $d$ within the range $2\leq d\leq 5$.
We obtain the foreground by the orthogonal projection to the recovered $3$-dimensional subspace.
Figure~\ref{fig:lobby} demonstrates foregrounds detected by EGMS, GMS, GMS2, PCP and LLD, where PCP and LLD used  $\lambda=1/\sqrt{\max(D,N)},0.8\sqrt{D/N}$. We remark that OP ran out of memory. Using truth labels provided in the data, we also form ROC curves for GMS, GMS2, EGMS and PCP  in Figure~\ref{fig:lobby_error} (LLD is not included since it performed poorly for any value of $\lambda$ we tried).
We note that PCP performs better than both GMS and EGMS in the `Shoppingmall' video, whereas the latter algorithms perform better than PCP in the `SwitchLight' video. Furthermore, GMS is significantly faster than EGMS and PCP. Indeed, the running times (on average) of GMS, EGMS and PCP are 91.2, 1018.8 and 1209.4 seconds respectively.
\begin{figure}[htbp]
\begin{centering}
\includegraphics[width=.85\textwidth,height=.4\textwidth]{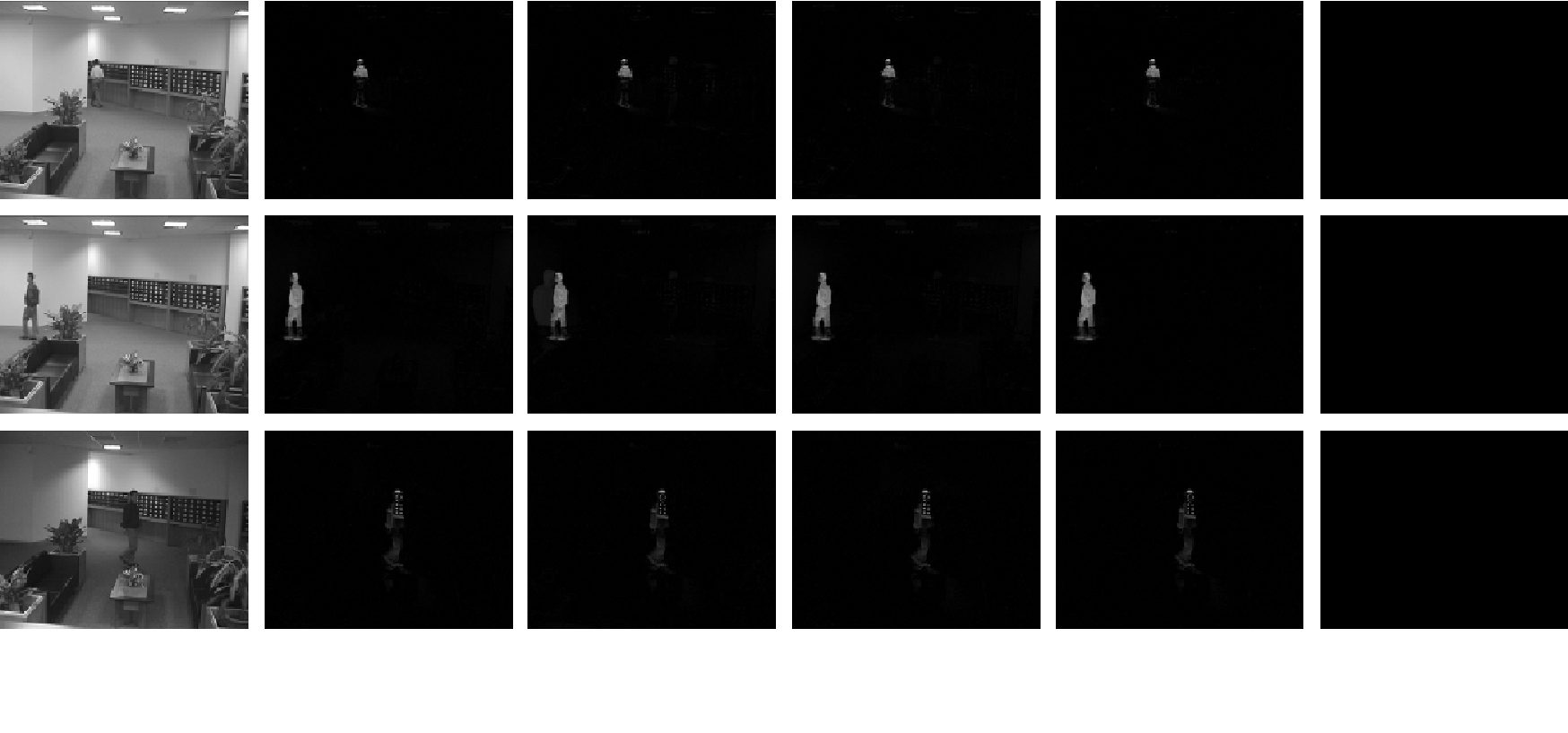}
\includegraphics[width=.85\textwidth,height=.4\textwidth]{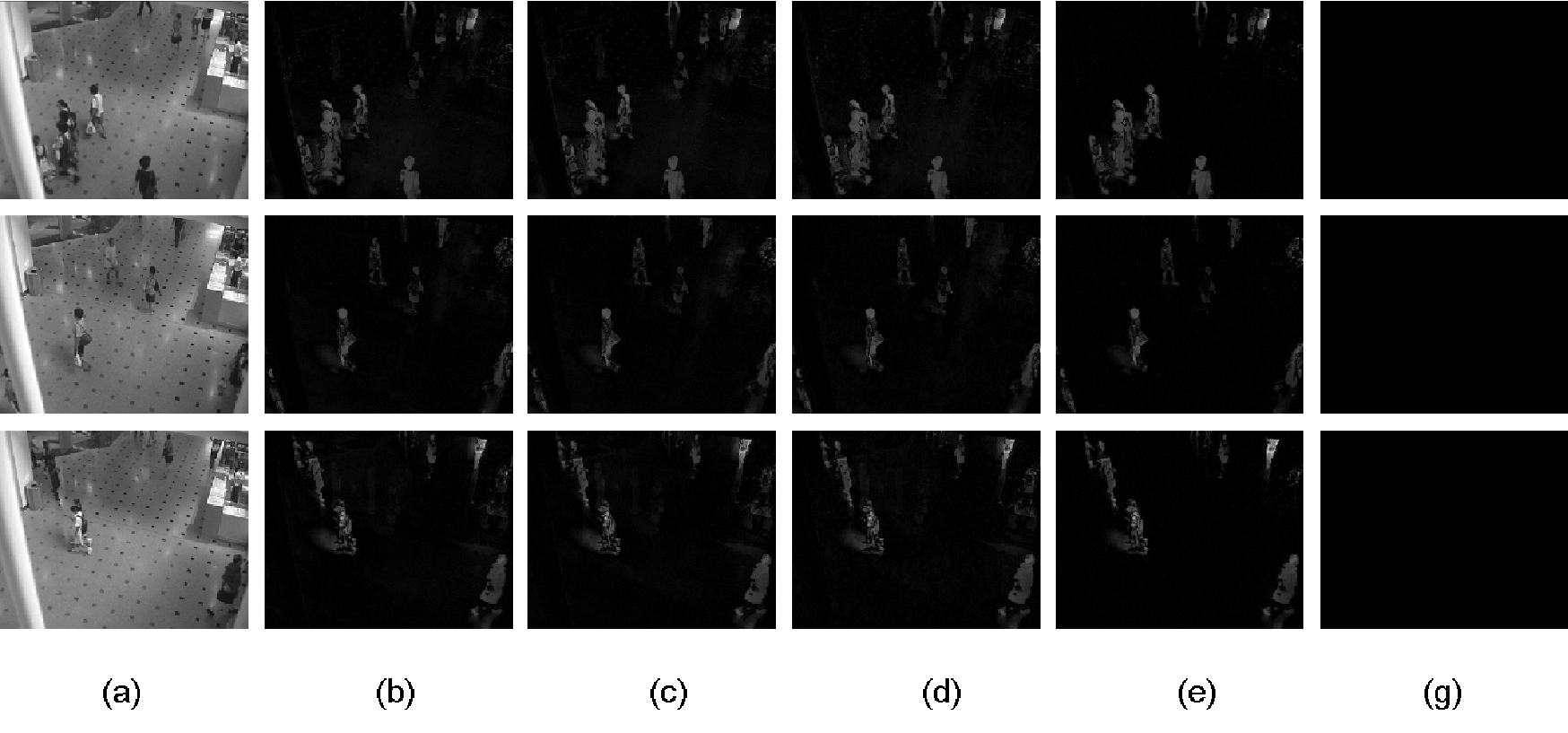}
\caption{Video surveillance: (a) the given frames (b)-(e) the detected foreground by EGMS, GMS, GMS2, PCP, LLD respectively\label{fig:lobby}}
\end{centering}
\end{figure}

\begin{figure}[htbp]
\begin{center}
\includegraphics[width=.4\textwidth,height=.3\textwidth]{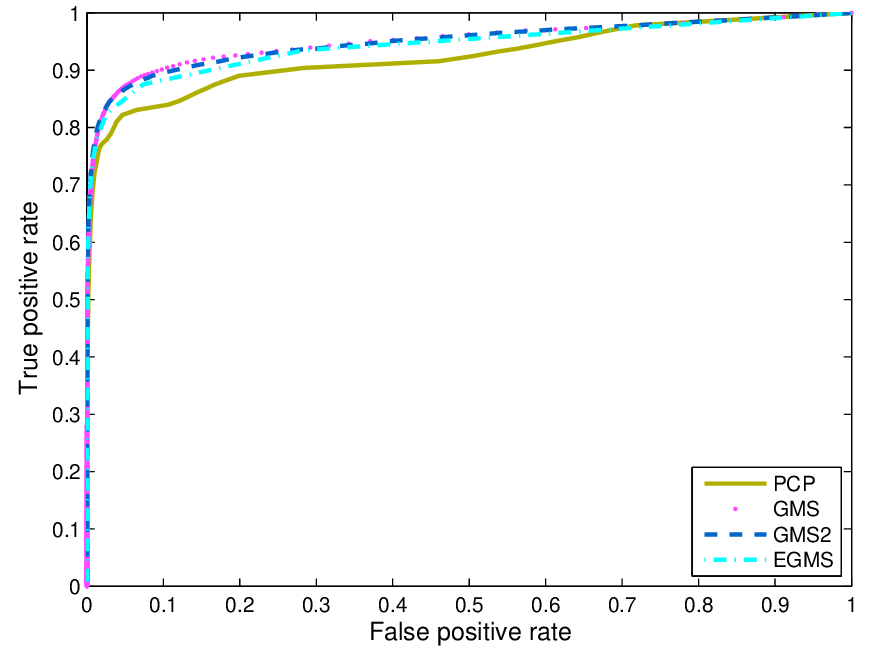}
\includegraphics[width=.4\textwidth,height=.3\textwidth]{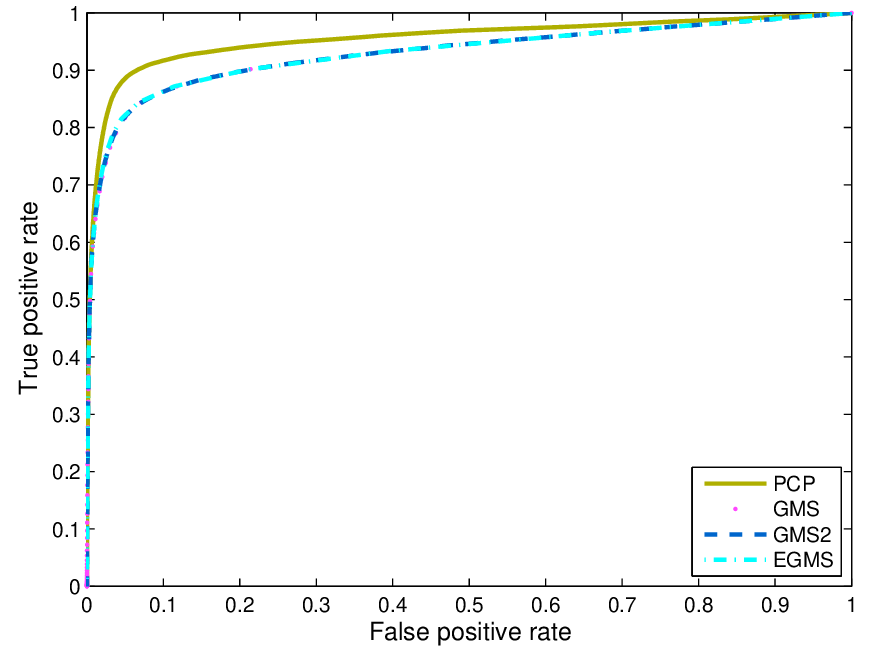}
\caption{ROC curves for EGMS, GMS, GMS2 and PCP in the 'SwitchLight' video (the left figure) and the 'Shoppingmall' video (the right figure)\label{fig:lobby_error}}
\end{center}
\end{figure}
\section{Proofs of Theorems}
\label{sec:proofs}

We present the technical proofs of the theoretical statements of this paper according to their order of appearance.

%NEED TEXT HERE

\subsection{Proof of Theorem~\ref{thm:recovery}}

We will prove that if conditions \eqref{eq:condition1} and \eqref{eq:condition2} hold, then the set of all minimizers satisfying~\eqref{eq:symmetric} coincides with the set of all minimizers satisfying~\eqref{eq:symmetric1}. This clearly implies that if conditions \eqref{eq:condition1} and \eqref{eq:condition2} hold, then any minimizer $\hat{\bQ}$ of~\eqref{eq:symmetric} satisfies  $\ker(\hat{\bQ}) \supseteq \rmL^*$ (indeed, this condition is equivalent with the condition
$\bQ \bP_{\rmL^*}=\b0$, which appears in the formulation of~\eqref{eq:symmetric1}). If condition~\eqref{eq:condition3} also holds, then $\ker(\hat{\bQ})= \rmL^*$ and the theorem is concluded.

We assume that conditions \eqref{eq:condition1} and \eqref{eq:condition2} hold and arbitrarily fix a minimizer $\hat{\bQ}_0$ of the oracle problem~\eqref{eq:symmetric1}. We claim that in order to establish the equivalence of the sets of solutions of~\eqref{eq:symmetric} and~\eqref{eq:symmetric1}, it is sufficient to prove that
\begin{equation}
\label{eq:symmetric2}
\text{$F(\hat{\bQ}_0+\bDelta)-F(\hat{\bQ}_0)> 0$ for any symmetric $\bDelta$ with $\tr(\bDelta)=0$ and $\bDelta\bP_{\rmL^*}\neq\b0$.}
\end{equation}
Indeed, we first note that~\eqref{eq:symmetric2} implies that $\hat{\bQ}_0$ is also a minimizer of~\eqref{eq:symmetric}.
This observation follows from combining~\eqref{eq:symmetric2} with the following equation:
$$
\text{$F(\hat{\bQ}_0+\bDelta)-F(\hat{\bQ}_0) \geq 0$ for any symmetric $\bDelta$ with $\tr(\bDelta)=0$ and $\bDelta\bP_{\rmL^*}=\b0$,}
$$
which is an immediate consequence of the definition of~\eqref{eq:symmetric1}.
To conclude the equivalence, we assume on the contrary that there exists $\hat{\bQ}_0$, which is a minimizer of~\eqref{eq:symmetric1} but not a minimizer of~\eqref{eq:symmetric}.
We denote by $\hat{\bQ}'_0$ a minimizer of~\eqref{eq:symmetric1}, which is also a minimizer of~\eqref{eq:symmetric} and let $\bDelta := \hat{\bQ}'_0-\hat{\bQ}_0$. Then by the definitions of $\hat{\bQ}_0$, $\hat{\bQ}'_0$ and $\bDelta$:
$\tr(\bDelta)=0$, $\bDelta\bP_{\rmL^*}\neq\b0$ and $F(\hat{\bQ}'_0)=F(\hat{\bQ}_0)$. This contradicts
\eqref{eq:symmetric2} and thus concludes the proof.

In order to conclude~\eqref{eq:symmetric2} (and thus the theorem) we first differentiate $\|\bQ\bx\|$ at $\bQ=\bQ_0$ when $\bx \in \ker(\bQ_0)^\perp$ as follows:
\begin{equation}\label{eq:derivative0}
\frac{\di}{\di \bQ}\|\bQ\bx\|\Big|_{\bQ=\bQ_0}=\frac{\di}{\di \bQ}\sqrt{\|\bQ\bx\|^2}\Big|_{\bQ=\bQ_0}=\frac{\di}{\di \bQ}\frac{\bQ\bx\bx^T\bQ^T}{2\|\bQ_0\bx\|}\Big|_{\bQ=\bQ_0}=\frac{\bQ_0\bx\bx^T+\bx\bx^T\bQ_0}{2\|\bQ_0\bx\|}.
\end{equation}
We note that for any $\bx \in \reals^D \setminus \{\b0\}$ satisfying
 $\hat{\bQ}_0\bx \neq \b0$ and $\bDelta \in \reals^{D \times D}$ symmetric:
\begin{equation}\label{eq:derivative1}
\| (\hat{\bQ}_0+\bDelta) \bx \| -  \|\hat{\bQ}_0 \bx\| \geq  \left\langle\bDelta,(\hat{\bQ}_0\bx\bx^T+\bx\bx^T\hat{\bQ}_0) / 2\| \hat{\bQ}_0 \bx \|\right\rangle_F=\left\langle\bDelta,\hat{\bQ}_0\bx\bx^T / \| \hat{\bQ}_0 \bx \|\right\rangle_F.
\end{equation}
Indeed, the first equality follows from \eqref{eq:derivative0} and the convexity of $\|\bQ \bx\|$ in $\bQ$ and the second equality follows from
the symmetry of $\bDelta$ and $\hat{\bQ}_0$ as well as the definition of the Frobenius dot product.

If on the other hand $\hat{\bQ}_0\bx = \b0$, then clearly
\begin{equation}\label{eq:derivative2p}
\|(\hat{\bQ}_0+\bDelta) \bx \| -  \|\hat{\bQ}_0 \bx\| = \|\bDelta\bx\|.
\end{equation}

For simplicity of our presentation, we use~\eqref{eq:derivative2p} only for  $\bx\in\sX_1$
(where obviously $\hat{\bQ}_0\bx = \b0$ since $\hat{\bQ}_0\bP_{\rmL^*}=0$).
On the other hand, we use~\eqref{eq:derivative1}
for all $\bx\in\sX_0$. One can easily check that if $\bx\in\sX_0$ and $\hat{\bQ}_0\bx = \b0$,
then replacing~\eqref{eq:derivative1}
with~\eqref{eq:derivative2p} does not change the analysis below.
Using these observations we note that
\begin{equation}
F(\hat{\bQ}_0+\bDelta)-F(\hat{\bQ}_0) \geq \sum_{\bx\in\sX_1} \|\bDelta\bx\|+\sum_{\bx\in\sX_0}\left\langle\bDelta,\hat{\bQ}_0\bx\bx^T/\|\hat{\bQ}_0\bx\|\right\rangle_F.
\label{eq:derivative}
\end{equation}

We assume first that $\bDelta\bP_{\rmL^*}=\b0$. In this case, $\hat{\bQ}_0+\bDelta\in\bbH$ and $(\hat{\bQ}_0+\bDelta)\bP_{\rmL^*}=\b0$.
Since $\hat{\bQ}_0$ is the minimizer of \eqref{eq:symmetric1}, %(we note that the domain of the function in \eqref{eq:symmetric1}  only contains one element when $d=D-1$ and therefore we exclude this case)
we obtain the following identity (which is analogous to \eqref{eq:explain_alg_min}):
\begin{equation}\label{eq:bDelta}
\sum_{\bx\in\sX_0} \left\langle\bDelta,\frac{\hat{\bQ}_0\bx \bx^T}{\|\hat{\bQ}_0\bx\|}\right\rangle_F
\geq 0\, \ \ \forall \ \bDelta \in \reals^{D \times D} \text{ s.t.~} \tr(\bDelta)=0 \ , \bDelta\bP_{\rmL^*}=\b0.
\end{equation}

We will prove \eqref{eq:symmetric2} by showing that the RHS of~\eqref{eq:derivative} is positive
for any symmetric $\bDelta$ with $\tr(\bDelta)=0$ and $\bDelta\bP_{\rmL^*}\neq\b0$.
Using~\eqref{eq:derivative} and the facts that $\sX_1\subset\rmL^*$ and $\hat{\bQ}_0=\bP_{\rmL^{*\perp}}\hat{\bQ}_0$ (since $\bP_{\rmL^*} \hat{\bQ}_0 = \hat{\bQ}_0 \bP_{\rmL^*}=\b0$), we establish the following inequality:
\begin{align}
&F(\hat{\bQ}_0+\bDelta)-F(\hat{\bQ}_0) \geq \sum_{\bx\in\sX_1} \|\bDelta\bx\|+\sum_{\bx\in\sX_0}\left\langle\bDelta,\hat{\bQ}_0\bx\bx^T/\|\hat{\bQ}_0\bx\|\right\rangle_F
\nonumber\\=& \sum_{\bx\in\sX_1} \|\bDelta\bP_{\rmL^*}\bx\|+\sum_{\bx\in\sX_0}\left\langle(\bDelta\bP_{\rmL^*}+\bDelta\bP_{\rmL^{*\perp}}),\bP_{\rmL^{*\perp}}\hat{\bQ}_0\bx\bx^T/\|\hat{\bQ}_0\bx\|\right\rangle_F
\nonumber\\\geq &\sum_{\bx\in\sX_1} \left(\|\bP_{\rmL^*}\bDelta\bP_{\rmL^*}\bx\|+\|\bP_{\rmL^{*\perp}}\bDelta\bP_{\rmL^*}\bx\|\right)/\sqrt{2}
\nonumber\\+&\sum_{\bx\in\sX_0}\left\langle(\bP_{\rmL^{*\perp}}\bDelta\bP_{\rmL^*}+\bP_{\rmL^{*\perp}}\bDelta\bP_{\rmL^{*\perp}}),\hat{\bQ}_0\bx\bx^T/\|\hat{\bQ}_0\bx\|\right\rangle_F.
\label{eq:derivative2}\end{align}

For ease of notation we denote $\bDelta_0=\tr(\bP_{\rmL^*}\bDelta\bP_{\rmL^*})\bv_0\bv_0^T$, where $\bv_0$ is the minimizer of the RHS of \eqref{eq:condition1}. Combining the following two facts: $\tr(\bDelta_0)-\tr(\bP_{\rmL^*}\bDelta\bP_{\rmL^*})=0$ and $\tr(\bP_{\rmL^*}\bDelta\bP_{\rmL^*})+\tr(\bP_{\rmL^{*\perp}}\bDelta\bP_{\rmL^{*\perp}})=\tr(\bDelta)=0$, we obtain that
$$
\tr(\bDelta_0+\bP_{\rmL^{*\perp}}\bDelta\bP_{\rmL^{*\perp}})=0.
$$
Further application of~\eqref{eq:bDelta} implies that
\begin{equation}\label{eq:bDelta1}
\sum_{\bx\in\sX_0} \left\langle\bDelta_0+\bP_{\rmL^{*\perp}}\bDelta\bP_{\rmL^{*\perp}},\hat{\bQ}_0\bx\bx^T/\|\hat{\bQ}_0\bx\|\right\rangle_F\geq 0.
\end{equation}
We note that
\begin{align}
\nonumber&
\left\langle\bP_{\rmL^{*\perp}}\bDelta\bP_{\rmL^{*\perp}},\hat{\bQ}_0\bx\bx^T/\|\hat{\bQ}_0\bx\|\right\rangle_F=
\left\langle\bP_{\rmL^{*\perp}}\bDelta\bP_{\rmL^{*\perp}}\bP_{\rmL^{*\perp}},\hat{\bQ}_0\bx\bx^T/\|\hat{\bQ}_0\bx\|\right\rangle_F\\
=&\left\langle\bP_{\rmL^{*\perp}}\bDelta\bP_{\rmL^{*\perp}},\hat{\bQ}_0\bx\bx^T\bP_{\rmL^{*\perp}}/\|\hat{\bQ}_0\bx\|\right\rangle_F.
\label{eq:derivative3}
\end{align}
Combining~\eqref{eq:bDelta1} and~\eqref{eq:derivative3} we conclude that
\begin{align}
\nonumber&
- \sum_{\bx\in\sX_0} \left\langle\bP_{\rmL^{*\perp}}\bDelta\bP_{\rmL^{*\perp}},\hat{\bQ}_0\bx\bx^T/\|\hat{\bQ}_0\bx\|\right\rangle_F\leq
\sum_{\bx\in\sX_0} \left\langle\bDelta_0,\hat{\bQ}_0\bx\bx^T\bP_{\rmL^{*\perp}}/\|\hat{\bQ}_0\bx\|\right\rangle_F \nonumber\\=& \sum_{\bx\in\sX_0} \tr(\bP_{\rmL^*}\bDelta\bP_{\rmL^*}) (\bv_0^T\hat{\bQ}_0\bx/\|\hat{\bQ}_0\bx\|) (\bv_0^T\bP_{\rmL^{*\perp}}\bx)
\leq |\tr(\bP_{\rmL^*}\bDelta\bP_{\rmL^*})| \sum_{\bx\in\sX_0}|\bv_0^T\bx|.
\label{eq:derivative4}
\end{align}
We apply~\eqref{eq:derivative4} and then use~\eqref{eq:condition1} with $\bQ=\bP_{\rmL^*}\bDelta\bP_{\rmL^*}/\tr(\bP_{\rmL^*}\bDelta\bP_{\rmL^*})$ to obtain the inequality:
\begin{align}
&\sum_{\bx\in\sX_1} \|\bP_{\rmL^*}\bDelta\bP_{\rmL^*}\bx\|/\sqrt{2}
+\sum_{\bx\in\sX_0}\left\langle\bP_{\rmL^{*\perp}}\bDelta\bP_{\rmL^{*\perp}},\hat{\bQ}_0\bx\bx^T/\|\hat{\bQ}_0\bx\|\right\rangle_F
\nonumber\\\geq &
\sum_{\bx\in\sX_1} \|\bP_{\rmL^*}\bDelta\bP_{\rmL^*}\bx\|/\sqrt{2}-|\tr(\bP_{\rmL^*}\bDelta\bP_{\rmL^*})|\sum_{\bx\in\sX_0}|\bv_0^T\bx|> 0.\label{eq:result1}
\end{align}

We define $\bbH_1 = \{\bQ\in\bbH: \ \bQ\bP_{\rmL^{*\perp}}=\b0\}$ and claim that~\eqref{eq:condition2} leads to the following inequality:
\begin{equation}\label{eq:condition2_old}
\sum_{\bx\in\sX_1}\|\bQ({\bP}_{\rmL^*}\bx)\|> \sqrt{2}\,\sum_{\bx\in\sX_0}\|\bQ({\bP}_{\rmL^*}\bx)\|\,\,
\ \forall \bQ\in\bbH_1.
\end{equation}
Indeed, since the RHS of \eqref{eq:condition2_old} is a convex function of $\bQ$, its maximum is achieved at the set of all extreme points of $\bbH_1$, which is $\{\bQ\in\reals^{D\times D}: \bQ=\bv\bv^T, \text{where } \bv\in\rmL^*, \|\bv\|=1\}$. Therefore the maximum of the RHS of \eqref{eq:condition2_old} is the RHS of \eqref{eq:condition2}. Since the minimum of the LHS of \eqref{eq:condition2_old} is also the LHS of \eqref{eq:condition2}, \eqref{eq:condition2_old} is proved.

We also claim that \eqref{eq:condition2_old} can be extended from $\bbH_1$ to all $\bQ\in\reals^{D\times D}$ such that $\bQ\bP_{\rmL^{*\perp}}=\b0$.
%Indeed, we assign for
%any such $\bQ$ ($\bQ\in\reals^{D\times D}$ satisfying $\bQ\bP_{\rmL^{*\perp}}=\b0$) with SVD $\bQ=\bU\bSigma\bV^T$, the matrix $\bQ'=\bV\bSigma\bV^T/\tr(\bV\bSigma\bV^T)$ and note that $\bQ' \in \bbH_1$.
Indeed, for any $\bQ\in\reals^{D\times D}$ satisfying $\bQ\bP_{\rmL^{*\perp}}=\b0$ and having the SVD decomposition $\bQ=\bU\bSigma\bV^T$, we can assign the following
matrix $\bQ' = \bQ'(\bQ) \in \bbH_1$: $\bQ' := \bV\bSigma\bV^T/\tr(\bV\bSigma\bV^T)$.
It is not hard to note that the inequality in \eqref{eq:condition2_old} holds for $\bQ$ if and only if it holds for $\bQ'$.

By first applying Cauchy's inequality, then using the defining property of projections and at last applying~\eqref{eq:condition2_old}
with $\bQ={\bP}_{\rmL^{*\perp}}\bDelta{\bP}_{\rmL^*}$ (while using its latter extension beyond  $\bbH_1$), we obtain the inequality:
\begin{align}
&\sum_{\bx\in\sX_1} \|\bP_{\rmL^{*\perp}}\bDelta\bP_{\rmL^*}\bx\|/\sqrt{2}
+\sum_{\bx\in\sX_0}\left\langle\bP_{\rmL^{*\perp}}\bDelta \bP_{\rmL{^*}},
\hat{\bQ}_0\bx\bx^T/\|\hat{\bQ}_0\bx\|\right\rangle_F
\nonumber\\\geq &
\sum_{\bx\in\sX_1} \|{\bP}_{\rmL^{*\perp}}\bDelta\bP_{\rmL^*}\bx\|/\sqrt{2}
-\sum_{\bx\in\sX_0}\|{\bP}_{\rmL^{*\perp}}\bDelta\bP_{\rmL^*}\bx\|
\nonumber\\=&
\sum_{\bx\in\sX_1} \|{\bP}_{\rmL^{*\perp}}\bDelta{\bP}_{\rmL^*}({\bP}_{\rmL^*}\bx)\|/\sqrt{2}
-\sum_{\bx\in\sX_0}\|{\bP}_{\rmL^{*\perp}}\bDelta{\bP}_{\rmL^*}({\bP}_{\rmL^*}\bx)\|
> 0 \label{eq:result2}.
\end{align}
Finally, we combine \eqref{eq:result1} and \eqref{eq:result2}
and conclude that the RHS of~\eqref{eq:derivative2} is nonnegative and consequently~\eqref{eq:symmetric2} holds.% This and the fact that the minimizer is unique (which follows from~\eqref{eq:convex_condition}) imply the theorem.

\subsection{Proof of Theorem~\ref{thm:convex}}

Assume on the contrary that $F$ is not strictly convex, in particular, there exists $0<t_0<1$ such that
$$
t_0 \cdot F(\bQ_1)+(1-t_0) \cdot F(\bQ_2)= F(t_0 \cdot \bQ_1+ (1-t_0) \cdot \bQ_2) \ \text{ for } \ \bQ_1\neq \bQ_2,
$$
or equivalently,
\begin{equation}
t_0 \cdot \sum_{i=1}^{N}\|\bQ_1\bx_i\|+ (1-t_0) \cdot \sum_{i=1}^{N}\|\bQ_2\bx_i\|
=\sum_{i=1}^{N}\|(t_0 \cdot \bQ_1+ (1-t_0) \cdot \bQ_2)\bx_i\|.\label{eq:convex_contradiction1}
\end{equation}
Combining \eqref{eq:convex_contradiction1} with the fact that $\|\bQ_1\bx_i\|+\|\bQ_2\bx_i\|\geq \|(\bQ_1+\bQ_2)\bx_i\|$, we obtain that
$t_0 \cdot \|\bQ_1\bx_i\|+ (1-t_0) \cdot \|\bQ_2\bx_i\|=
\|(t_0 \cdot \bQ_1+ (1-t_0) \cdot \bQ_2)\bx_i\|$ for any $1\leq i \leq N$ and therefore there exists a sequence $\{c_i\}_{i=1}^{N}\subset\reals$ such that
\begin{equation}
\text{$\bQ_2\bx_i=\b0$ \ or \ $\bQ_1\bx_i=c_i\,\bQ_2\bx_i$ \ for all $1\leq i\leq N$.}\label{eq:convex_proof}
\end{equation}

We conclude Theorem~\ref{thm:convex} by considering two different cases.
We first assume that $\ker(\bQ_1)=\ker(\bQ_2)$. We denote
\[\text{$\tilde{\bQ}_1=\bP_{\ker(\bQ_1)^\perp}\bQ_1\bP_{\ker(\bQ_1)^\perp}$ and $\tilde{\bQ}_2=\bP_{\ker(\bQ_1)^\perp}\bQ_2\bP_{\ker(\bQ_1)^\perp}$.}\]
It follows from~\eqref{eq:convex_proof} that
\[\tilde{\bQ}_1(\bP_{\ker(\bQ_1)^\perp}\bx_i)=c_i\,\tilde{\bQ}_2(\bP_{\ker(\bQ_1)^\perp}\bx_i)\]
and consequently that $\bP_{\ker(\bQ_1)^\perp}\bx_i$ lies in one of the eigenspaces of $\tilde{\bQ}_1^{-1}\tilde{\bQ}_2$.
We claim that $\tilde{\bQ}_1^{-1}\tilde{\bQ}_2$ is a scalar matrix. Indeed, if on the contrary $\tilde{\bQ}_1^{-1}\tilde{\bQ}_2$ is not a scalar matrix, then  $\{\bP_{\ker(\bQ_1)^\perp}\bx_i\}_{i=1}^{N}$ lies in a union of several eigenspaces with dimensions summing to $\dim(\bP_{\ker(\bQ_1)^\perp})$ and this contradicts~\eqref{eq:convex_condition}.
In view of this property of $\tilde{\bQ}_1^{-1}\tilde{\bQ}_2$ and the fact that $\tr(\tilde{\bQ}_1)=\tr(\hat{\bQ}_1)=1$ we have that $\tilde{\bQ}_1=\tilde{\bQ}_2$ and $\bQ_1=\bQ_2$, which contradicts our current assumption.

Next, assume that $\ker(\bQ_1)\neq\ker(\bQ_2)$. We will first show that if $1 \leq i \leq N$ is arbitrarily fixed, then
$\bx_i \in \ker(\bQ_2) \cup \ker(\bP_{\ker(\bQ_1)}\bQ_2)$. Indeed, if $\bx_i\notin\ker(\bQ_2)$, then using \eqref{eq:convex_proof} we have $\bQ_1\bx_i=c_i\,\bQ_2\bx_i$. This implies that $c_i\bP_{\ker(\bQ_1)}\bQ_2\bx_i=\bP_{\ker(\bQ_1)}\bQ_1\bx_i=\b0$ and thus $\bx_i\in\ker(\bP_{\ker(\bQ_1)}\bQ_2)$.
That is, $\sX$ is contained in the union of the 2 subspaces $\ker(\bQ_2)$ and $\ker(\bP_{\ker(\bQ_1)}\bQ_2)$. The dimensions of both spaces are less than $D$. This obvious for $\ker(\bQ_2)$, since $\tr(\bQ_2)=1$. For $\ker(\bP_{\ker(\bQ_1)}\bQ_2)$ it follows from the fact that $\ker(\bQ_1)\neq\ker(\bQ_2)$ and thus $\bP_{\ker(\bQ_1)}\bQ_2\neq\b0$. We thus obtained a contradiction to~\eqref{eq:convex_condition}.

\subsection{Verification of \eqref{eq:condition1b} and \eqref{eq:condition2b} as Sufficient Conditions and \eqref{eq:nec_cond1} and \eqref{eq:nec_cond2} as Necessary Ones}
\label{sec:verifiability_proofs}

We revisit the proof of Theorem~\ref{thm:recovery} and first show that \eqref{eq:condition1b} and \eqref{eq:condition2b} can replace
\eqref{eq:condition1} and \eqref{eq:condition2} in the first part of Theorem~\ref{thm:recovery}.
We only deal with the first part of Theorem~\ref{thm:recovery}, which assumes that \eqref{eq:condition3} holds,
since \eqref{eq:condition3} guarantees that \eqref{eq:condition1b} and \eqref{eq:condition2b} are well-defined (see the discussion in \S\ref{sec:verifiability}).

To show that \eqref{eq:condition2b} can replace \eqref{eq:condition2}, we prove the inequality in \eqref{eq:result2} using \eqref{eq:condition2b} as follows. Assuming that the SVD of $\bP_{\rmL^{*\perp}}\bDelta\bP_{\rmL^*}$ is $\bU\bSigma\bV^T$, then $\bQ':=\bV\bSigma\bV^T/\tr(\bSigma)$ satisfies $\bQ'\in\bbH, \bQ'\bP_{\rmL^{*\perp}}=\b0$ and $\|\bQ'\bx\|=\|\bP_{\rmL^{*\perp}}\bDelta\bP_{\rmL^*}\bx\|/\tr(\bSigma)=\|\bP_{\rmL^{*\perp}}\bDelta\bP_{\rmL^*}\bx\|/\|\bP_{\rmL^{*\perp}}\bDelta\bP_{\rmL^*}\|_*.$
Using this fact, we obtain that
\begin{align}
\label{eq:aux_alt_cond1}
&\sum_{\bx\in\sX_1} \|\bP_{\rmL^{*\perp}}\bDelta\bP_{\rmL^*}\bx\|
\geq \|\bP_{\rmL^{*\perp}}\bDelta\bP_{\rmL^*}\|_* \min_{\bQ'\in\bbH,\bQ'\bP_{\rmL^{*\perp}}=\b0}\sum_{\bx\in\sX_1}\|\bQ'\bx\|.
\end{align}
We also note that
\begin{align}
&\sum_{\bx\in\sX_0}\!\!\left\langle\bP_{\rmL^{*\perp}}\bDelta \bP_{\rmL{^*}},
\hat{\bQ}_0\bx\bx^T/\|\hat{\bQ}_0\bx\|\right\rangle_F
=\sum_{\bx\in\sX_0}\!\!\left\langle\bP_{\rmL^{*\perp}}\bDelta \bP_{\rmL{^*}},
\hat{\bQ}_0\bx\bx^T\bP_{\rmL{^*}}/\|\hat{\bQ}_0\bx\|\right\rangle_F
\nonumber\\\label{eq:aux_alt_cond2}\geq& - \|\bP_{\rmL^{*\perp}}\bDelta\bP_{\rmL^*}\|_* \left\|\sum_{\bx\in\sX_0}\!\!\hat{\bQ}_0\bx\bx^T\bP_{\rmL^*}/\|\hat{\bQ}_0\bx\|\right\|.
\end{align}
Therefore  \eqref{eq:result2} follows from \eqref{eq:condition2b}, \eqref{eq:aux_alt_cond1} and \eqref{eq:aux_alt_cond2}.
Similarly, one can show that \eqref{eq:condition1b} may replace \eqref{eq:condition1}.

One can also verify that \eqref{eq:nec_cond1} and \eqref{eq:nec_cond2} are necessary conditions for exact recovery
by revisiting the proof of Theorem~\ref{thm:recovery} and reversing inequalities.

\subsection{Proof of Lemma~\ref{lemma:prob_cond}}
We first note by symmetry that the minimizer of the LHS of~\eqref{eq:prob_cond2} for the needle-haystack model is $\bQ = \bP_{\rmL^*}/d$.
We can thus rewrite~\eqref{eq:prob_cond2} in this case as $\alpha_1 \Expect r_1/d > 2\sqrt{2} \alpha_0 \Expect r_0 /(D-d)$, where the ``radii'' $r_1$ and $r_0$ are the norms of the normal distributions with covariances $\sigma_1^2 d^{-1} \bP_{\rmL^*}$ and $\sigma_0^2 D^{-1} \bP_{\rmL^{*\perp}}$  respectively. Let $\tilde{r}_1$ and $\tilde{r}_2$ be the $\chi$-distributed random variables with $d$ and $D-d$ degrees of freedoms, then~\eqref{eq:prob_cond2} obtains the form
$$ \frac{\alpha_1 \sigma_1}{d\sqrt{d}} \Expect \tilde{r}_1 > \frac{2\sqrt{2} \alpha_0 \sigma_0}{(D-d)\sqrt{D}} \Expect \tilde{r}_0. $$

Applying (B.7) of \citet{LMTZ2012}, $\Expect \tilde{r}_1\geq \sqrt{d/2}$ and $\Expect \tilde{r}_0\leq \sqrt{D-d}$. Therefore ~\eqref{eq:prob_cond2} follows from \eqref{eq:prob_cond2a}.

\subsection{Proof of Theorem~\ref{thm:prob}}
\label{sec:proof_thm_prob}
For simplicity of the proof we first assume that the supports of $\mu_0$ and $\mu_1$ are contained in a ball centered at the origin of radius $M$.

We start with the proof of \eqref{eq:condition3} ``in expectation'' and then extend it to hold with high probability.
We use the notation $F_I(\bQ)$ and $\hat{\bQ}_I$ defined in \eqref{eq:define_F_I}
and \eqref{eq:define_bQ_I} respectively.
The spherical symmetry of $\mu_{0,\rmL^{*\perp}}$ implies that
\begin{equation}
\label{eq:bqi_sym}
\hat{\bQ}_I=\frac{1}{D-d}\bP_{\rmL^{*\perp}}\bP_{\rmL^{*\perp}}^T
\end{equation}
is the unique
minimizer of \eqref{eq:define_bQ_I}. To see this formally, we first note that $\mu_{0,\rmL^{*\perp}}$ satisfies the two-subspaces
criterion of \citet{Coudron_Lerman2012} for any $0<\gamma\leq 1$ (this criterion generalizes~\eqref{eq:convex_condition} of this paper
to continuous measures)
and thus by Theorem 2.1 of \citet{Coudron_Lerman2012} (whose proof follows directly the one of Theorem 2 here)
the solution of this minimization must be unique.
On the other hand,  any application of an arbitrary rotation of $\rmL^*$ (within $\reals^D$) to
the minimizer expressed in the RHS of~\eqref{eq:define_bQ_I} should also be a minimizer of the RHS of~\eqref{eq:define_bQ_I}. We note
that $\frac{1}{D-d}\bP_{\rmL^{*\perp}}\bP_{\rmL^{*\perp}}^T$ is the only element in the domain of this minimization that is preserved under any
rotation of $\rmL^*$. Therefore, due to uniqueness, this can be the only solution of this minimization problem.

Let
\begin{equation}
\label{eq:bbH2}
\bbH_2=\{\bQ\in\bbH  :  \ \bQ\bP_{\rmL^*}=\b0, \ \bQ \succeq \b0  \text{ and }
\text{cond}(\bP_{\rmL^{*\perp}} \bQ \bP_{\rmL^{*\perp}}) \geq 2 \},
\end{equation}
where $\bQ \succeq \b0$ denotes the positive semidefiniteness of $\bQ$
and $\text{cond}(\bP_{\rmL^{*\perp}} \bQ \bP_{\rmL^{*\perp}})$ denotes the condition number of this matrix,
that is, the ratio between the largest and lowest eigenvalues of
$\bP_{\rmL^{*\perp}} \bQ \bP_{\rmL^{*\perp}}$, or equivalently, the ratio
between the top eigenvalue and the $(D-d)$th eigenvalue of $\bQ$.
Since $\hat{\bQ}_I$ is the unique minimizer of \eqref{eq:define_bQ_I} and $\hat{\bQ}_I \nin \bbH_2$, then
\begin{equation}
\label{eq:c1}
c_1:=\min_{\bQ\in\bbH_2}
(F_I(\bQ) - F_I(\hat{\bQ}_I))>0.%=\alpha_0\sigma_0 (\frac{1}{D-d-1}\int_{\bx}\|\bx\|\di g_{D-d-1}-
\end{equation}
We note that if $\bx$ is a random variable sampled from $\mu$ and $\bQ \in \bbH$ (so that $\|\bQ\| \leq \|\bQ\|_*=1$), then $\|\bQ\bx\|\leq M$. Applying this fact, \eqref{eq:c1} and Hoeffding's inequality,
we conclude that for any fixed $\bQ \in \bbH_2$
\begin{equation}
\label{eq:hoeff_1}
F(\bQ)-F(\hat{\bQ}_I)>c_1 N /2 \ \text{ w.p.~}1-\exp(-c_1^2N/2M^2).
\end{equation}

We also observe that
\begin{equation}\label{eq:eps_net}F(\bQ_1)-F(\bQ_2)\leq \|\bQ_1-\bQ_2\|\sum_{i=1}^N\|\bx\|\leq \|\bQ_1-\bQ_2\|N\,M\,.
\end{equation}
Combining~\eqref{eq:hoeff_1} and~\eqref{eq:eps_net}, we obtain that
for all $\bQ$ in a ball of radius $r_1:=c_1/2M$ centered around a fixed element in $\bbH_2$: $F(\bQ)-F(\hat{\bQ}_I)>0$
w.p.~$1-\exp(-c_1^2N/2M^2)$.

We thus cover the compact space ${\bbH}_2$ by an $r_1$-net. Denoting the corresponding covering number by $N({\bbH}_2,r_1)$ and using the above observation we note that w.p. \linebreak[4] $1-N({\bbH}_2,r_1)\exp(-c_1^2N/2M^2)$
\begin{equation}\label{eq:bQ_I}
F(\bQ)-F(\hat{\bQ}_I)> 0\,\,\,\,\text{ for all $\bQ \in \bbH_2$.}
\end{equation}

The definition of $\hat{\bQ}_0$ (that is, \eqref{eq:symmetric1}) implies that $F(\hat{\bQ}_0) \leq F(\hat{\bQ}_I)$. Combining this observation with~\eqref{eq:bQ_I}, we conclude that
w.h.p.~$\hat{\bQ}_0 \nin \bbH_2$. We also claim that $\hat{\bQ}_0 \succeq \b0$ (see, e.g., the proof of Lemma~\ref{lemma:nonnegative}, which appears later). Since $\hat{\bQ}_0 \nin \bbH_2$ and $\hat{\bQ}_0 \succeq \b0$, $\hat{\bQ}_0$ satisfies the following property w.h.p.:
\begin{equation}
\label{eq:help1}
\text{cond}(\bP_{\rmL^{*\perp}}^T\hat{\bQ}_0\bP_{\rmL^{*\perp}}^T)<2.
\end{equation}
Consequently, \eqref{eq:condition3} holds w.h.p.~(more precisely, w.p.~$1-N({\bbH}_2,c_1/2M)\exp(-c_1^2N/2M^2)$).

Next, we verify \eqref{eq:condition1b} w.h.p. as follows.
Since $\hat{\bQ}_0$ is symmetric and $\hat{\bQ}_0 \bP_{\rmL^*}=\b0$ (see \eqref{eq:symmetric1}), then
\begin{equation}
\label{eq:help2}
\hat{\bQ}_0 = \bP_{\rmL^{*\perp}} \hat{\bQ}_0 \bP_{\rmL^{*\perp}}.
\end{equation}
Applying \eqref{eq:help2}, basic inequalities of operators' norms and \eqref{eq:help1},
we bound the RHS of \eqref{eq:condition1b} from above as follows:
\begin{align}
\nonumber
&\sqrt{2}\,\left\|\sum_{\bx\in\sX_0}\hat{\bQ}_0\bx\bx^T\bP_{\rmL^{*\perp}}/\|\hat{\bQ}_0\bx\|\right\|
=
\sqrt{2}\,\left\|\bP_{\rmL^{*\perp}}\hat{\bQ}_0\bP_{\rmL^{*\perp}}\cdot\sum_{\bx\in\sX_0}\bP_{\rmL^{*\perp}}\bx\bx^T\bP_{\rmL^{*\perp}}/\|\hat{\bQ}_0\bx\|\right\|
\\
\label{eq:help3}
\leq& \sqrt{2} \cdot \left\|\bP_{\rmL^{*\perp}}\hat{\bQ}_0\bP_{\rmL^{*\perp}}\right\|\cdot\left\|\sum_{\bx\in\sX_0}\bP_{\rmL^{*\perp}}\bx\bx^T\bP_{\rmL^{*\perp}}/\|\hat{\bQ}_0\bx\|\right\|
\\
\nonumber
\leq &\sqrt{2}
\cdot \lambda_{\max}(\bP_{\rmL^{*\perp}}\hat{\bQ}_0\bP_{\rmL^{*\perp}})\cdot \|\sum_{\bx\in\sX_0}\bP_{\rmL^{*\perp}}\bx\bx^T\bP_{\rmL^{*\perp}}/
\|\lambda_{\min}(\bP_{\rmL^{*\perp}}\hat{\bQ}_0\bP_{\rmL^{*\perp}})\bP_{\rmL^{*\perp}}\bx\|\|
\\
\nonumber < &\sqrt{8}\Big\|\sum_{\bx\in\sX_0}\bP_{\rmL^{*\perp}}\bx\bx^T\bP_{\rmL^{*\perp}}/\|\bP_{\rmL^{*\perp}}\bx\|\Big\|=\max_{\bu\in S^{D-1}\cap \rmL^{*\perp}} \sqrt{8} \bu^T( \sum_{\bx\in\sX_0}\bP_{\rmL^{*\perp}}\bx\bx^T\bP_{\rmL^{*\perp}}/\|\bP_{\rmL^{*\perp}}\bx\|)\bu.
\end{align}
Therefore to prove \eqref{eq:condition1b}, we only need to prove that with high probability
\begin{equation}\label{eq:condition1b2}
\min_{\bQ\in\bbH,\bQ\bP_{\rmL^{*\perp}}=\b0}\sum_{\bx\in\sX_1}\|\bQ\bx\| >
\max_{\bu\in S^{D-1}\cap \rmL^{*\perp}} \sqrt{8} \bu^T( \sum_{\bx\in\sX_0}\bP_{\rmL^{*\perp}}\bx\bx^T\bP_{\rmL^{*\perp}}/\|\bP_{\rmL^{*\perp}}\bx\|)\bu.
\end{equation}

We will prove that the
LHS and RHS of \eqref{eq:condition1b2} concentrates w.h.p.~around the LHS and RHS of \eqref{eq:prob_cond2} respectively and consequently verify \eqref{eq:condition1b2} w.h.p.
Let $\eps_1$ be the difference between the RHS and LHS of \eqref{eq:condition1b2}.
Theorem 1 of \citet{Coudron_Lerman2012} implies that the LHS of \eqref{eq:condition1b2} is within distance $\eps_1/4$ to the RHS of \eqref{eq:prob_cond2}
with probability $1-C \exp(-N/C)$ (where $C$ is a constant depending on $\eps_1$, $\mu$ and its parameters).

The concentration of the RHS of \eqref{eq:prob_cond2} can be concluded as follows.
The spherical symmetry of $\mu_{0,\rmL^{*\perp}}$ implies that the expectation (w.r.t.~$\mu_0$) of
$\sum_{\bx\in\sX_0}\bP_{\rmL^{*\perp}}\bx\bx^T\bP_{\rmL^{*\perp}}/\|\bP_{\rmL^{*\perp}}\bx\|$
is a scalar matrix within $\rmL^{*\perp}$, that is, it equals
$\rho_{\mu} \, {\bP_{\rmL^{*\perp}}\bx\bx^T\bP_{\rmL^{*\perp}}/\|\bP_{\rmL^{*\perp}}\bx\|}$ for some $\rho_{\mu} \in \reals$.
We observe that
$$
\Expect_{\mu_0}{\tr(\bP_{\rmL^{*\perp}}\bx\bx^T\bP_{\rmL^{*\perp}}/\|\bP_{\rmL^{*\perp}}\bx\|)}
=\Expect_{\mu_{0}}{\|\bP_{\rmL^{*\perp}}\bx\|}
$$
and thus conclude that $\rho_{\mu}=\Expect_{\mu_{0}}{\|\bP_{\rmL^{*\perp}}\bx\|}/(D-d)$.
Therefore, for any $\bu\in S^{D-1}\cap \rmL^{*\perp}$
\begin{equation}
\label{eq:expected_u}
\Expect_{\mu_0}{\bu^T(\bP_{\rmL^{*\perp}}\bx\bx^T\bP_{\rmL^{*\perp}}/\|\bP_{\rmL^{*\perp}}\bx\|)\bu}
=\Expect_{\mu_{0}}{\|\bP_{\rmL^{*\perp}}\bx\|}/(D-d)=\int\|\bP_{\rmL^{*\perp}}\bx\|\di\mu_{0}(\bx)/(D-d).
\end{equation}

We thus conclude from \eqref{eq:expected_u} and Hoeffding's inequality that for any fixed $\bu\in S^{D-1}\cap \rmL^{*\perp}$ the function
$\sqrt{8} \bu^T(\sum_{\bx\in\sX_0}\bP_{\rmL^{*\perp}}\bx\bx^T\bP_{\rmL^{*\perp}}/\|\bP_{\rmL^{*\perp}}\bx\|)\bu$
is within distance $\eps_1/4$ to the RHS of \eqref{eq:prob_cond2}
with probability $1-C \exp(-N/C)$ (where $C$ is a constant depending on $\eps_1$, $\mu$ and its parameters).
Furthermore, applying $\eps$-nets and covering (i.e., union bounds) arguments with regards to $S^{D-1}\cap\rmL^{*\perp}$, we obtain that
for all $\bu \in S^{D-1}\cap\rmL^{*\perp}$, \linebreak[4] $\sqrt{8} \bu^T(\sum_{\bx\in\sX_0}\bP_{\rmL^{*\perp}}\bx\bx^T\bP_{\rmL^{*\perp}}/\|\bP_{\rmL^{*\perp}}\bx\|)\bu$
is within distance $\eps_1/2$ to the RHS of \eqref{eq:prob_cond2}
with probability $1-C \exp(-N/C)$ (where $C$ is a constant depending on $\eps_1$, $\mu$ and its parameters).
In particular, the RHS of  \eqref{eq:condition1b2} is within distance $\eps_1/2$ to the RHS of \eqref{eq:prob_cond2}
with the same probability. We thus conclude \eqref{eq:condition1b2} with probability $1-C' \exp(-N/C')$.

Similarly we can also prove \eqref{eq:condition2b}, noting that the expectation (w.r.t.~$\mu_0$) of \linebreak[4] $\hat{\bQ}_0\bx\bx^T\bP_{\rmL^*}/\|\hat{\bQ}_0\bx\|$ is $\b0$, since $\hat{\bQ}_0\bx/\|\hat{\bQ}_0\bx\|$ and $\bx^T\bP_{\rmL^*}$ are independent when $\bx$ is restricted to lie in the complement of $\rmL^*$ (that is, $ \bx \in\sX_0$).

If we remove the assumption of bounded supports (with radius $M$), then we need to replace Hoeffding's inequality with the Hoeffding-type inequality for sub-Gaussian measures of Proposition 5.10 of \citet{vershynin_book}, where in this proposition $a_i=1$ for all $1\leq i\leq n$.

We emphasize that our probabilistic estimates are rather loose and can be interpreted as near-asymptotic;
we thus did not fully specify their constants.
We clarify this point for the probability estimate we have for \eqref{eq:condition3}, that is,
$1-N({\bbH}_2,c_1/2M)\exp(-c_1^2N/2M^2)$.
Its constant $N(\bbH_2,r_1)$ can be bounded from above by
the covering number $N(\bbH_0,r_1)$ of the larger set
$\bbH_0=\{\bQ\in\reals^{D\times D}: |\bQ_{i,i}|\leq 1\}$, which is bounded from above by
$(8/r_1)^{D(D-1)/2}$ (see, e.g., Lemma 5.2 of \citealp{vershynin_book}).
This is clearly a very loose estimate that cannot reveal interesting information, such as, the right dependence of $N$ on $D$ and $d$ in order to obtain a sufficiently small probability.

At last, we explain why \eqref{eq:convex_condition} holds with probability $1$ if there are at least $2D-1$ outliers. %or at least $D$ outliers and $D-1$ inliers,
We denote the set of outliers by $\{\by_1,\by_2,\cdots,\by_{N_0}\}$, where $N_0 \geq 2D-1$,
and assume on the contrary that \eqref{eq:convex_condition} holds with probability smaller than $1$.
Then, there exists a sequence $\{i_j\}_{j=1}^{D-1}\subset\{1,2,3,\cdots,N_0\}$ such that the subspace spanned by the $D-1$ points $\by_{i_1},\by_{i_2},\cdots,\by_{i_{D-1}}$ contains another outlier with positive probability. However, this is not true for haystack model and thus our claim is proved.

\subsubsection{Proof of the Extension of Theorem~\ref{thm:prob} to the Asymmetric Case}
\label{sec:proof_thm_prob_asym}
We recall our assumptions that $\mu_0$ is a sub-Gaussian distribution with covariance
$\bSigma_0$ and that $\hat{\bQ}_I$ is unique.
We follow the proof of Theorem~\ref{thm:prob} in \S\ref{sec:proof_thm_prob} with the following changes.
First of all,
we replace the requirement
\begin{equation}
\label{eq:sym_cond}
\text{cond}(\bP_{\rmL^{*\perp}} \bQ \bP_{\rmL^{*\perp}}) \geq 2.
\end{equation}
in \eqref{eq:bbH2}
with the following one:
\begin{equation}
\label{eq:asym_cond}
\text{cond}(\bP_{\rmL^{*\perp}} \bQ \bP_{\rmL^{*\perp}}) \geq 2 \cdot \text{cond}(\bP_{\rmL^{*\perp}} \hat{\bQ}_I
\bP_{\rmL^{*\perp}}).
\end{equation}
We note that \eqref{eq:sym_cond} follows from \eqref{eq:asym_cond} in the symmetric case.
Indeed, in this case the expression of $\hat{\bQ}_I$ in \eqref{eq:bqi_sym} implies that
the RHS of \eqref{eq:asym_cond} is 2.
Similarly, instead of \eqref{eq:help1} we prove that
$$
\text{cond}(\bP_{\rmL^{*\perp}}^T\hat{\bQ}_0\bP_{\rmL^{*\perp}}^T)<
2 \cdot \text{cond}(\bP_{\rmL^{*\perp}}
\hat{\bQ}_I \bP_{\rmL^{*\perp}}).
$$
Second of all,
in the third inequality of \eqref{eq:help3}
the term
$$\sqrt{2}
\, \lambda_{\max}(\bP_{\rmL^{*\perp}}\hat{\bQ}_0\bP_{\rmL^{*\perp}})/\lambda_{\max}(\bP_{\rmL^{*\perp}}\hat{\bQ}_0\bP_{\rmL^{*\perp}})
$$
needs to be bounded above by
$\sqrt{8}\,\text{cond}(\bP_{\rmL^{*\perp}} \hat{\bQ}_I \bP_{\rmL^{*\perp}})$, instead of $\sqrt{8}$.
We can thus conclude the revised theorem, in particular, the last modification in the proof clarifies why we need to multiply the RHS
of \eqref{eq:prob_cond2} by $\text{cond}(\bP_{\rmL^{*\perp}} \hat{\bQ}_I \bP_{\rmL^{*\perp}})$,
which is the ratio between the largest eigenvalue
of $\bP_{\rmL^{*\perp}}\hat{\bQ}_I\bP_{\rmL^{*\perp}}$ and the $(D-d)$th eigenvalue of
$\bP_{\rmL^{*\perp}}\hat{\bQ}_I\bP_{\rmL^{*\perp}}$.

\subsection{Proof of Theorem~\ref{thm:prob_b}}
\label{sec:proof_thm_prob_b}

This proof follows ideas of \citet{LMTZ2012}.
We bound from below the LHS of~\eqref{eq:condition2} by applying (A.15) of \citet{LMTZ2012} as follows
\be\label{eq:condition2c}
\min_{\bQ\in\bbH,\bQ\bP_{\rmL^{*\perp}}=\b0}\sum_{\bx\in\sX_1}\|\bQ\bx\|
\geq \frac{1}{\sqrt{d}} \min_{\bv\in\rmL^*,\|\bv\|=1}\sum_{\bx\in\sX_1}|\bv^T\bx|.
\ee
We denote the number of inliers sampled from $\mu_1$ by $N_1$ and the number of outliers sampled from $\mu_0$ by $N_0 (=N-N_1)$.
We bound from below w.h.p.~the RHS of \eqref{eq:condition2c} by applying Lemma B.2 of \citet{LMTZ2012} in the following way:
\be\label{eq:condition2d}
\frac{1}{\sqrt{d}}\min_{\bv\in\rmL^*,\|\bv\|=1}\sum_{\bx\in\sX_1}|\bv^T\bx|
\geq \frac{\sigma_1}{{d}}\Big(\sqrt{2/\pi} N_1 -2\sqrt{N_1d} -t\sqrt{N_1}\Big)\,\,\text{w.p. $1-e^{-t^2/2}$}.
\ee
By following the proof of Lemma B.2 of \citet{LMTZ2012}
we bound from above w.h.p.~the RHS of \eqref{eq:condition2} as follows
\be\label{eq:condition2e}
\max_{\bv\in\rmL^*,\|\bv\|=1}\sum_{\bx\in\sX_0}|\bv^T\bx|
\leq \frac{\sigma_0}{\sqrt{D}}\Big(\sqrt{2/\pi} N_0 +2\sqrt{N_0d} +t\sqrt{N_0}\Big)\,\,\text{w.p. $1-e^{-t^2/2}$}.
\ee

We need to show w.h.p.~that the RHS of \eqref{eq:condition2e} is strictly less than the RHS of~\eqref{eq:condition2d}.
We note that Hoeffding's inequality implies that
\begin{equation}
\label{eq:cond_n1_n0}
N_1>\alpha_1 N/2 \ \text{ w.p. } \ 1-e^{-\alpha_1^2N/2} \ \text{ and } \ |N_0- \alpha_0 N|<\,\alpha_0 N/2 \ \text{ w.p. } \ 1-2e^{-\alpha_0^2N/2}.
\end{equation}
Furthermore, \eqref{eq:prob_cond3_assumption} and \eqref{eq:cond_n1_n0} imply that
\begin{equation}
\label{eq:cond_d_above_below}
d<N_1/4 \ \text{ w.p. } \ 1-e^{-\alpha_1^2N/2} \ \text{ and } \ d<N_0/4 \ \text{ w.p. } \ 1-e^{-\alpha_0^2N/2}.
\end{equation}

Substituting $t=\sqrt{N_1}/10$ $(>\sqrt{\alpha_1 N}/20$ w.p.~$1-e^{-\alpha_1^2N/2}$) in \eqref{eq:condition2d} and
$t=\sqrt{N_0}/10$ $(>\sqrt{\alpha_0 N}/20$ w.p.~$1-2e^{-\alpha_0^2N/2}$) in \eqref{eq:condition2e}
and combining \eqref{eq:prob_cond3} and \eqref{eq:condition2c}-\eqref{eq:cond_d_above_below},
we obtain that \eqref{eq:condition2} holds w.p.
$1-e^{-\alpha_1^2N/2}-2e^{-\alpha_0^2N/2}-e^{-\alpha_1 N/800}-e^{-\alpha_0 N/800}$.
We can similarly obtain that \eqref{eq:condition1} holds with the same probability.

\subsubsection{Proof of the Extension of Theorem~\ref{thm:prob_b} to the Asymmetric Case}
\label{sec:proof_thm_prob_b_asym}

We assume the generalized needle-haystack model of \S\ref{sec:asym_outliers_theory}.
The proof of Theorem~\ref{thm:prob_b} in \S\ref{sec:proof_thm_prob_b} immediately extends to this model,
where $\sigma_0$ in the RHS of \eqref{eq:condition2e} needs to be replaced with $\sqrt{\lambda_{\max}(\bSigma_0)}$
(recall that $\lambda_{\max}(\bSigma_0)$ denotes the largest eigenvalue of $\bSigma_0$).
Consequently, Theorem~\ref{thm:prob_b} still holds in this case when replacing $\sigma_0$ in the RHS of \eqref{eq:prob_cond3} with
$\sqrt{\lambda_{\max}(\bSigma_0)}$.

\subsection{Proof of Theorem~\ref{thm:noisy_recovery}}\label{sec:proof_thm_noisy_recovery}
We first establish the following lemma.
\begin{lemma}
\label{lemma:nonnegative}
The minimizer of $F(\bQ)$, $\hat{\bQ}$, is a semi-definite positive matrix.
\end{lemma}
\begin{proof}
We assume that $\hat{\bQ}$ has some negative eigenvalues and
show that this assumption contradicts the defining property of $\hat{\bQ}$, that is, being the minimizer of $F(\bQ)$.
We denote the eigenvalue decomposition of $\hat{\bQ}$ by
$\hat{\bQ}=\bV_{\hat{\bQ}}\bSigma_{\hat{\bQ}}\bV_{\hat{\bQ}}^T$ and
define $\bSigma_{\hat{\bQ}}^+=\max(\bSigma_{\hat{\bQ}},0)$ and
$\hat{\bQ}^+=\bV_{\hat{\bQ}}\bSigma_{\hat{\bQ}}^+\bV_{\hat{\bQ}}^T/\tr(\bSigma_{\hat{\bQ}}^+)\in\bbH$.
Then $\tr(\bSigma_{\hat{\bQ}}^+)>\tr(\bSigma_{\hat{\bQ}})=\tr(\hat{\bQ})=1$
and for any $\bx\in\reals^D$ we have
\begin{align*}&\|\hat{\bQ}^+\bx\|<\tr(\bSigma_{\hat{\bQ}}^+)
\|\hat{\bQ}^+\bx\|=\|\bSigma_{\hat{\bQ}}^+(\bV_{\hat{\bQ}}^T\bx)\|
\leq \|\bSigma_{\hat{\bQ}}(\bV_{\hat{\bQ}}^T\bx)\| =\|\hat{\bQ}\bx\|.
\end{align*}
Summing it over all $\bx \in \sX$, we conclude the contradiction
$F(\hat{\bQ}^+)<F(\hat{\bQ})$.

\end{proof}

In order to prove Theorem~\ref{thm:noisy_recovery} we first notice that by definition and the connection of $\gamma_0$, $\gamma_0$ with second derivative of $F(\bQ)$
\be\label{eq:bQ_norm1}
F_{\sX}(\tilde{\bQ})-F_{\sX}(\tilde{\bQ})\geq N\gamma_0 \|\tilde{\bQ}-\hat{\bQ}\|_F^2,
\ee
and
\be\label{eq:bQ_norm2}
F_{\sX}(\tilde{\bQ})-F_{\sX}(\tilde{\bQ})\geq N\gamma_0' \|\tilde{\bQ}-\hat{\bQ}\|^2.
\ee

Next, we observe that
$$
|F_\sX(\!\hat{\bQ}\!)-F_{\tilde{\sX}}(\!\hat{\bQ}\!)|\!\leq\! \sum_{i=1}^{N}\left|\!\|\hat{\bQ}\tilde{\bx}_i\|\!-\!\|\hat{\bQ}\bx_i\|\!\right|\!\leq\! \sum_{i=1}^{N}\|\hat{\bQ}(\tilde{\bx}_i-\bx_i)\|\!\leq\! \sum_{i=1}^{N}\|\tilde{\bx}_i-\bx_i\|\!\leq\! \sum_{i=1}^{N}\eps_i
$$
and similarly $|F_\sX(\tilde{\bQ})-F_{\tilde{\sX}}(\tilde{\bQ})|\leq \sum_{i=1}^{N}\eps_i.$
Therefore,
\begin{align}
&F_{\sX}(\tilde{\bQ})-F_{\sX}(\hat{\bQ})= (F_{\tilde{\sX}}(\tilde{\bQ})-F_{\tilde{\sX}}(\hat{\bQ}))
+(F_{\sX}(\tilde{\bQ})-F_{\tilde{\sX}}(\tilde{\bQ})) +
(F_{\tilde{\sX}}(\hat{\bQ})\nonumber\\-&F_{\sX}(\hat{\bQ}))\leq0+ |F_{\sX}(\tilde{\bQ})-F_{\tilde{\sX}}(\tilde{\bQ})|
+|F_{\tilde{\sX}}(\hat{\bQ})-F_{\sX}(\hat{\bQ})|
\leq 2 \sum_{i=1}^{N}\eps_i.\label{eq:diff2}
\end{align}
Therefore~\eqref{eq:bQ_diff}
follows from \eqref{eq:bQ_norm1}, \eqref{eq:bQ_norm2} and \eqref{eq:diff2}.
Applying the Davis-Kahan perturbation Theorem \citep{Davis1970} to \eqref{eq:bQ_diff}, we conclude~\eqref{eq:subspace_diff}.

\subsubsection{Implication of Theorem~\ref{thm:noisy_recovery} to Dimension Estimation}\label{sec:noisy_recovery_imply1}

Theorem~\ref{thm:noisy_recovery} implies that we may properly estimate the dimension of the underlying subspace for low-dimensional data with sufficiently
small perturbation. We make this statement more precise by assuming the setting of Theorem~\ref{thm:noisy_recovery} and further
assuming that $\hat{\bQ}$ is a low-rank matrix with $\ker(\hat{\bQ})=\rmL^*$. We note that the $(D-d+1)$st eigenvalue of $\hat{\bQ}$ is $0$.
Thus applying the following eigenvalue stability inequality \citep[(1.63)]{tao_ran_mat_book}:
\be
\label{eq:noisy_recovery}
|\lambda_i(\mathbf{A+B})-\lambda_i(\mathbf{A})|\leq \|\mathbf{B}\|,
\ee
we obtain that the $(D-d+1)$st eigenvalue of $\tilde{\bQ}$ is smaller than $\sqrt{2 \sum_{i=1}^{N}\eps_i/\gamma_0}$,
and the $(D-d)$th eigengap of $\tilde{\bQ}$ is larger than $\nu_{D-d}-2\sqrt{2 \sum_{i=1}^{N}\eps_i/\gamma_0}$
(recall that $\nu_{D-d}$ is the $(D-d)$th eigengap of $\hat{\bQ}$).
This means that when the noise is small and the conditions of Theorem~\ref{thm:recovery} hold,
then we can estimate the dimension of the underlying subspace for $\tilde{\sX}$
from the number of small eigenvalues.

\subsubsection{Improved Bounds in a Restricted Setting}\label{sec:noisy_recovery_imply2}

We assume that $\eps_i=O(\eps)$ for all $1\leq i\leq N$, where
$\eps$ is sufficiently small, and further assume that $\rank(\hat{\bQ})=D$.
We show that in this special case the norm of $\hat{\bQ} - \tilde{\bQ}$
is of order $O(\eps)$ instead of order $O(\sqrt{\eps})$ that is specified in Theorem~\ref{thm:noisy_recovery}.

We note that since $\hat{\bQ}$ is of full rank, then the first and second directional derivative of $F$ are well-defined
in a sufficiently small neighborhood around $\hat{\bQ}$. Therefore, if $\bDelta \in \reals^{D \times D}$ and $\|\bDelta\|$ is sufficiently
small then
\begin{equation}
\label{eq:imply2_1}
F'_{\sX}(\hat{\bQ})-F'_{\sX}(\hat{\bQ}+\bDelta) = O(\|\bDelta\|).
\end{equation}
Furthermore, we note by basic calculations that
\begin{equation}
\label{eq:imply2_2}
F'_{\sX}(\bQ)-F'_{\tilde{\sX}}(\bQ)=O(\eps).
\end{equation}
Combining~\eqref{eq:imply2_2} with the following facts: $F'_{\sX}(\hat{\bQ})=0$ and $F'_{\tilde{\sX}}(\tilde{\bQ})=0$, we obtain that
\begin{equation}
\label{eq:imply2_3}
F'_{\sX}(\hat{\bQ})-F'_{\sX}(\tilde{\bQ})=F'_{\tilde{\sX}}(\tilde{\bQ})
-F'_{\sX}(\tilde{\bQ})=O(\eps).
\end{equation}
At last, the combination of~\eqref{eq:imply2_1} and~\eqref{eq:imply2_3}  implies that $\|\hat{\bQ}-\tilde{\bQ}\|=O(\eps)$.
Clearly, the spectral norm of $\hat{\bQ}-\tilde{\bQ}$ can be replaced with any other norm, in particular, the Frobenius norm.

\subsection{Proof of Proposition~\ref{prop:noisy_recovery}}

We recall the function $F_I$, which was defined in
\eqref{eq:define_F_I}, and the notation ${F_{I,1}}''(\bQ,\bDelta)$ should be clear, where now $F_I$ replaces $F$.

The law of large numbers implies that
${F_{1}}''(\bQ,\bDelta)/N \rightarrow {F_{I}}''(\bQ,\bDelta)$
almost surely for any $\bDelta$ and $\bQ$
(see also related bounds in \citealt{Coudron_Lerman2012}). Since $\bQ$ and $\bDelta$ lie in compact space, we conclude
\eqref{eq:asymptotic_derivative} for $\gamma_0$ and $c_0$; the proof is identical for $\gamma'_0$ and $c'_0$.

\subsection{Proof of Theorem~\ref{thm:noisy}}

The theorem follows from the observation that $0\leq F(\bQ)-F_\delta(\bQ)\leq N\delta/2$ for all $\bQ\in\bbH$ and the proof of Theorem~\ref{thm:noisy_recovery}.

\subsection{Proof of Theorem~\ref{thm:m_estimator}}

It is sufficient to verify that
\be\text{If $\tilde{\bA} \in \reals^{D \times D}$ with $\im(\tilde{\bA})=\rmL^*$, then $L(\tilde{\bA}+\eta\bI)\rightarrow\infty$ as $\eta\rightarrow 0$.}\label{eq:tildebA}\ee
Indeed, since $L(\bA)$ is a continuous function, \eqref{eq:tildebA} implies that $L(\tilde{\bA})$ is undefined (or infinite) and therefore $\tilde{\bA}$ is not the minimizer of \eqref{eq:mestimator_function} as stated in Theorem~\ref{thm:m_estimator}.

We fix $a_1<\lim_{x\rightarrow\infty}xu(x)$ and note that Condition $\text{D}_0$ (w.r.t.~$\rmL^*$) implies that
\be\text{$|\sX_0|/N>(D-d)/a_1$.}\label{eq:tildebA1}\ee
Condition M implies that there exists $x_1$ such that for any $x>x_1$: $xu(x)\geq a_1$  and therefore (recalling that $u=\rho'$) $\rho(x)\geq a_1\ln(x-x_1)/2 + u(x_1)/2$. Thus for any $\bx_i\in\sX_0$, we have
\be \text{$\rho(\bx_i^T(\tilde{\bA}+\eta\bI)^{-1}\bx_i)\geq a_1\ln(1/\eta-x_1)/2+C_i$ for some constant $C_i \equiv C_i(\bx_i, \tilde{\bA})$}\label{eq:tildebA2}\ee
and
\be\label{eq:tildebA3}
\frac{N}{2}\log(\det(\bA))\leq N C_0 + (D-d)/2\ln(\eta) \,\,\text{for some $C_0 \equiv C_0(\tilde{\bA})$}.
\ee
Equation~\eqref{eq:tildebA} thus follows from~\eqref{eq:tildebA1}-\eqref{eq:tildebA3} and the theorem is concluded.

\subsection{Proof of Theorem~\ref{thm:pca}}
The derivative of the energy function in the RHS of \eqref{eq:symmetric_PCA} is %$\sum_{i=1}^N\bQ\bx_i \bx_i^T$ $=$
$\bQ\bX^T\bX+\bX^T\bX\bQ$. Using the argument establishing \eqref{eq:equal_to_cI} and the fact that $\hat{\bQ}_2$ is the minimizer of \eqref{eq:symmetric_PCA}, we conclude that $\bQ\bX^T\bX+\bX^T\bX\bQ$
is a scalar matrix. We then conclude \eqref{eq:pca_inverse_cov} by using the argument establishing \eqref{eq:solution_equal_to_cI} as well as the following two facts: $\tr(\hat{\bQ}_2)=1$ and $\bX$ is full rank (so the inverse of $\bX^T \bX$ exists).

\subsection{Proof of Theorem~\ref{thm:alg_clean}}
\label{sec:proof_alg_clean}

We frequently use here some of the notation introduced in \S\ref{sec:heuristic}, in particular, $I(\bQ)$,
$\rmL(\bQ)$ and $T(\bQ)$.
We will first prove
that $F(\bQ_k)\geq F(\bQ_{k+1})$ for all $k\geq 1$. For this purpose, we use the convex quadratic function: \[G(\bQ,\bQ^*)=\frac{1}{2}
\sum_{\latop{i=1}{i\notin I(\bQ^*)}}^{N}\left(\|\bQ\bx_i\|^2/\|\bQ^*\bx_i\|+\|\bQ^*\bx_i\|\right).\]
Following the same derivation of~\eqref{eq:derivative0} and~\eqref{eq:equal_to_cI}, we obtain that
\[
\frac{\di}{\di \bQ}G(\bQ,\bQ_k)\big|_{\bQ=\bQ_{k+1}}=\left(\bQ_{k+1}\left(\sum_{\latop{i=1}{i\notin I(\bQ_k)}}^{N}\frac{\bx_i\bx_i^T}{\|\bQ_k\bx_i\|}\right)+\left(\sum_{\latop{i=1}{i\notin I(\bQ_k)}}^{N}\frac{\bx_i\bx_i^T}{\|\bQ_k\bx_i\|}\right)\bQ_{k+1}\right)/2.
\]
We let $\bA_k= \sum_{i=1,\,i\notin I(\bQ_k)}^{N}\frac{\bx_i\bx_i^T}{\|\bQ_k\bx_i\|}$, $c_k = {\bP}_{\rmL(\bQ_k)^\perp}\bA_k^{-1}{\bP}_{\rmL(\bQ_k)^\perp}$ and for any symmetric $\bDelta\in\reals^{D\times D}$ with $\tr(\bDelta)=0$ and $\bP_{\rmL(\bQ_k)}\bDelta=\b0$ we let $\bDelta_0=\tilde{\bP}_{\rmL(\bQ_k)^\perp}^T\bDelta\tilde{\bP}_{\rmL(\bQ_k)^\perp}$. We note  that
\begin{align*}&\tr(\bDelta_0)=\left\langle\bDelta_0,\bI\right\rangle_F=
\left\langle\tilde{\bP}_{\rmL(\bQ_k)^\perp}^T\bDelta\tilde{\bP}_{\rmL(\bQ_k)^\perp},\bI\right\rangle_F
=
\left\langle\bDelta,\tilde{\bP}_{\rmL(\bQ_k)^\perp}\tilde{\bP}_{\rmL(\bQ_k)^\perp}^T\right\rangle_F
\\=&
\left\langle\bDelta,\bI-{\bP}_{\rmL(\bQ_k)}\right\rangle_F
=\left\langle\bDelta,\bI\right\rangle_F
-\left\langle\bDelta,{\bP}_{\rmL(\bQ_k)}\right\rangle_F
=\left\langle\bDelta,\bI\right\rangle_F
=\tr(\bDelta)=0.\end{align*}
Consequently, we establish that the derivative of $G(\bQ,\bQ_k)$ at $\bQ_{k+1}$ in the direction $\bDelta$ is zero as follows.
\begin{align*}
&\left\langle (\bQ_{k+1}\bA_k+\bA_k\bQ_{k+1})/2,\bDelta\right\rangle_F=\left\langle \bQ_{k+1}\bA_k,\bDelta\right\rangle_F
=c_k\left\langle{\bP}_{\rmL(\bQ_k)^\perp}\bA_k^{-1}{\bP}_{\rmL(\bQ_k)^\perp}\bA_k,\bDelta\right\rangle_F\\
=&c_k\left\langle{\bP}_{\rmL(\bQ_k)^\perp}\bA_k^{-1}{\bP}_{\rmL(\bQ_k)^\perp}\bA_k,\tilde{\bP}_{\rmL(\bQ_k)^\perp}\bDelta_0\tilde{\bP}_{\rmL(\bQ_k)^\perp}^T\right\rangle_F
\\
=&c_k\left\langle(\tilde{\bP}_{\rmL(\bQ_k)^\perp}^T\bA_k^{-1}\tilde{\bP}_{\rmL(\bQ_k)^\perp})(\tilde{\bP}_{\rmL(\bQ_k)^\perp}^T\bA_k\tilde{\bP}_{\rmL(\bQ_k)^\perp}),\bDelta_0\right\rangle_F
=c_k\left\langle\bI,\bDelta_0\right\rangle_F=0\,.
\end{align*}
This and the strict convexity of $G(\bQ,\bQ_k)$ (which follows from $\Sp(\{\bx_i\}_{i\notin I(\bQ_k)})=\reals^D$ using  \eqref{eq:convex_condition})
imply that
$\bQ_{k+1}$ is the unique minimizer of $G(\bQ,\bQ_k)$ among all $\bQ\in\bbH$ such that $\bP_{\rmL(\bQ_k)}\bQ=\b0$.

Combining this with the following two facts: $\bQ_{k+1}\bx_i=0$ for any $i\in I(\bQ_k)$ and $G(\bQ_k,\bQ_k)=F(\bQ_k)$, we conclude that
\begin{align}
&F(\bQ_{k+1})=\sum_{i\notin I(\bQ_k)}\|\bQ_{k+1}\bx_i\|=\sum_{i\notin I(\bQ_k)}\frac{\|\bQ_{k+1}\bx_i\| \|\bQ_k\bx_i\|}{\|\bQ_k\bx_i\|}\nonumber\\\leq & \sum_{i\notin I(\bQ_k)}\frac{\|\bQ_{k+1}\bx_i\|^2+ \|\bQ_k\bx_i\|^2}{2\|\bQ_k\bx_i\|}=G(\bQ_{k+1},\bQ_k)\leq G(\bQ_k,\bQ_k)=F(\bQ_k)\label{eq:upper_bound}.
\end{align}

Since $F$ is positive, $F(\bQ_k)$ converges and \begin{equation}\text{$F(\bQ_k)-F(\bQ_{k+1})\rightarrow 0$ as $k\rightarrow\infty$.}\label{eq:upper_bound1}\end{equation}
Applying \eqref{eq:upper_bound} we also have that
\begin{equation}\label{eq:lower_bound}
F(\!\bQ_k\!)\!-\!F(\!\bQ_{k+1}\!)\!\geq\!  G(\bQ_k,\!\bQ_k)-G(\bQ_{k+1},\!\bQ_k)\!=\!\frac{1}{2}\sum_{i\notin I(\bQ_k)}\|(\bQ_k-\bQ_{k+1})\bx_i\|^2/\|{\bQ}_k \bx_i\|.
\end{equation}

We note that if $\bQ_k\neq\bQ_{k+1}$, then $\Sp(\{\bx_i\}_{i\notin I(\bQ_k)})=\reals^D\supset \ker(\bQ_k-\bQ_{k+1})$ and $1/\|{\bQ}_k \bx_i\|\geq 1/\max_{i}\|\bx_i\|$. Combining this observation with~\eqref{eq:upper_bound1} and~\eqref{eq:lower_bound} we obtain that \begin{equation}\|\bQ_k-\bQ_{k+1}\|_2 \rightarrow 0\,\,\,\, \text{as $k\rightarrow\infty$.}\label{eq:converge1}\end{equation}

Since for all $k \in \nats$, $\bQ_k$ is nonnegative (this follows from its defining formula~\eqref{eq:iterate_no_regular})
and $\tr(\bQ_k)=1$, the sequence $\{\bQ_k\}_{k \in \nats}$ lies in a compact space (of nonnegative matrices) and it thus has a converging subsequence. Assume a subsequence of $\{\bQ_k\}_{k \in \nats}$, which converges to $\tilde{\bQ}$.
We claim the following property of $\tilde{\bQ}$: \begin{equation}\text{$\tilde{\bQ}=\argmin_{\bQ\in\bbH_0}F(\bQ)$, where  $\bbH_0:=\{\bQ\in\bbH: \ker{\bQ}\supseteq\rmL(\tilde{\bQ})\}$.}\label{eq:finite_bQ}\end{equation}

In order to prove~\eqref{eq:finite_bQ}, we note that
\eqref{eq:upper_bound} and the convergence of the subsequence imply that $F(\tilde{\bQ})=F(T(\tilde{\bQ}))$.
Combining this with \eqref{eq:upper_bound} (though replacing $\bQ_k$ and $\bQ_{k+1}$ in \eqref{eq:upper_bound} with $\tilde{\bQ}$ and $T(\tilde{\bQ})$ respectively) we get that $G(T(\tilde{\bQ}),\tilde{\bQ})=G(\tilde{\bQ},\tilde{\bQ})$.
We conclude that $T(\tilde{\bQ})=\tilde{\bQ}$ from this observation and the following three facts: 1) $\bQ = \tilde{\bQ}$ is the unique minimizer of $G(\bQ,\tilde{\bQ})$ among all $\bQ\in\bbH$, 2) $\bP_{\rmL(\tilde{\bQ})}\tilde{\bQ}=\b0$, 3) $\bQ = T(\tilde{\bQ})$ is the unique minimizer of $G(\bQ,\tilde{\bQ})$ among all $\bQ\in\bbH$ such that $\bP_{\rmL(\tilde{\bQ})}\bQ=\b0$ (we remark that $F(\bQ)$ is strictly convex in $\bbH$ and consequently also in $\bbH_0$ by Theorem~\ref{thm:convex}).
Therefore, for any symmetric $\bDelta\in\reals^{D\times D}$ with $\tr(\bDelta)=0$ and $\bP_{\rmL(\tilde{\bQ})}\bDelta=\b0$, the directional derivative at $\tilde{\bQ}$ is $0$: \begin{equation}0=\left\langle\bDelta,\frac{\di}{\di \bQ}G(\bQ,\tilde{\bQ})\big|_{\bQ=\tilde{\bQ}}\right\rangle_F=\left\langle\bDelta,
\tilde{\bQ}\sum_{i\notin I(\tilde{\bQ})}\frac{\bx_i\bx_i^T}{\|\tilde{\bQ}\bx_i\|}\right\rangle_F.
\label{eq:derivative_convergence}\end{equation}
We note that~\eqref{eq:derivative_convergence} is the corresponding directional derivative of $F(\bQ)$ when restricted to $\bQ\in\bbH_0$ and we thus conclude~\eqref{eq:finite_bQ}.

Next, we will prove that $\{\bQ_k\}_{k \in \nats}$ converge to $\tilde{\bQ}$ by proving that there are only finite choices for $\tilde{\bQ}$.
In view of~\eqref{eq:finite_bQ} and the strict convexity of $F(\bQ)$ in $\bbH_0$, any limit $\tilde{\bQ}$ (of a subsequence as above) is uniquely determined by $I(\tilde{\bQ})$.
Since the number of choices for $I(\tilde{\bQ})$ is finite (independently of $\tilde{\bQ}$), the number of choices for $\tilde{\bQ}$ is finite. That is, $\sY:=\{\bQ\in\bbH:F({\bQ})=F(T({\bQ}))\}$ is a finite set.
Combining this with \eqref{eq:converge1} and the convergence analysis of the sequence $\{\bQ_k\}_{k \in \nats}$ \citep[see][Theorem 28.1]{Ostrowski66}, we conclude that $\{\bQ_k\}_{k \in \nats}$ converges to $\tilde{\bQ}$.

At last, we assume that $\tilde{\bQ}\bx_i\neq \b0$ for all $1\leq i\leq N$. We note that $I(\tilde{\bQ})=\emptyset$ and thus $\tilde{\bQ}=\hat{\bQ}$ by \eqref{eq:finite_bQ}.
The proof for the rate of convergence follows the analysis of generalized Weiszfeld's method by \citet{Chan99} (in particular see \S6 of that work).
We practically need to verify Hypotheses 4.1 and 4.2 (see \S4 of that work) and replace the functions $F$ and $G$ in that work by $F(\bQ)$ and
$$
\tilde{G}(\bQ,\bQ^*)=\sum_{i=1}^{N}\left(\|\bQ\bx_i\|^2/\|\bQ^*\bx_i\|+\|\bQ^*\bx_i\|\right)
$$
respectively.
We note that the functions $\tilde{G}$ and $G$ (defined earlier in this work) coincide in the following way: $\tilde{G}(\bQ,\bQ_k)={G}(\bQ,\bQ_k)$  for any $k\in\nats$ (this follows from the fact that $\bQ_k\bx_i\neq\b0$ for all $k\in\nats$ and $1\leq i\leq N$; indeed, otherwise for some $i$, $\bQ_j\bx_i = \b0$ for $j \geq k$ by~\eqref{eq:iterate_no_regular} and this leads to the contradiction $\hat{\bQ}\bx_i = \b0$).
We remark that even though \citet{Chan99} consider vector-valued functions,
their proof generalizes to matrix-valued functions as here.
Furthermore, we can replace the global properties of Hypotheses 4.1 and 4.2 of \citet{Chan99} by the local properties in $B(\hat{\bQ},\delta_0)$ for any $\delta_0>0$, since the convergence of $\bQ_k$ implies the existence of $K_0>0$ such that $\bQ_k\in B(\hat{\bQ},\delta_0)$ for all $k>K_0$. In particular, there is no need to check condition 2 in Hypothesis 4.1. Condition 1 in Hypothesis 4.1 holds since $F(\bQ)$ is twice differentiable in $B(\hat{\bQ},\delta_0)$ (which follows from the assumption on the limit $\tilde{\bQ} \equiv \hat{\bQ}$ and the continuity of the derivative). Conditions 1-3 in Hypothesis 4.2 are verified by the fact that $C$ of Hypothesis 4.2 satisfies $C(\bQ^*)=\sum_{i=1}^{N}\bx_i\bx_i^T/\|\bQ^*\bx_i\|$ and $\bQ^*\bx_i\neq\b0$ when $\bQ^*\in B(\hat{\bQ},\delta_0)$. Condition 3 in Hypothesis 3.1 and condition 4 in Hypothesis 4.2 are easy to check.

\subsection{Proof of Theorem~\ref{thm:alg_noisy}}

The proof follows from the second part of the proof of  Theorem~\ref{thm:alg_clean}, while using instead of
$\tilde{G}(\bQ,\bQ^*)$ the function
 \[
 G_\delta(\bQ,\bQ^*)=\frac{1}{2}\!\!\sum_{i=1,\|\bQ^*\bx_i\|\geq\delta}^{N}
 \!\!\!\left(\|\bQ\bx_i\|^2/\|\bQ^*\bx_i\|+\|\bQ^*\bx_i\|\right)
 +\!\!\!\!\sum_{i=1,\|\bQ^*\bx_i\|<\delta}^{N}\!\!\!\!(\|\bQ\bx_i\|^2/2\delta+\delta/2).
 \]

\subsection{Proof of Theorem~\ref{thm:rmL}}

We note that the minimization of $F(\bQ)$ over all $\bQ \in \bbH$ such that $\bQ\bP_{\hat{\rmL}^\perp}=\b0$ in Algorithm~\ref{alg:rmL} can be performed at each iteration with respect to the projected data:  $\tilde{\bP}_{\hat{\rmL}}(\sX)=\{\tilde{\bP}_{\hat{\rmL}}\bx_1,\tilde{\bP}_{\hat{\rmL}}\bx_2,\cdots,\tilde{\bP}_{\hat{\rmL}}\bx_N\}$.

We note that conditions \eqref{eq:condition1} and \eqref{eq:condition2} hold
for $\tilde{\bP}_{\hat{\rmL}}(\sX)$ with any $\hat{\rmL}\supseteq\rmL^*$.
Therefore, Theorem~\ref{thm:recovery} implies that $\bu\perp\rmL^*$ and $\hat{\rmL}\supseteq\rmL^*$ in each iteration. Since $\dim(\hat{\rmL})$ decreases by one in each iteration,  $\dim(\hat{\rmL})=d$ in $D-d$ iterations and thus $\hat{\rmL}=\rmL^*$.

\section{Conclusion}\label{sec:discussion}
We proposed an M-estimator for the problems of exact and near subspace
recovery.
Substantial theory has been developed to quantify the recovery obtained
by this estimator as well as its numerical approximation.
Numerical experiments demonstrated state-of-the-art speed and accuracy for our corresponding implementation on both synthetic and real data sets.

This work broadens the perspective of two recent ground-breaking theoretical
works for subspace recovery by
\citet{candes_wright_robust_pca09} and \citet{Xu2012}. We
hope that it will motivate additional approaches to this problem.

There are many interesting open problems that stem from our work.
We believe that by modifying or extending the framework described in here, one can even yield better results in various scenarios.
For example, we have discussed in \S\ref{sec:this_work} the modification by \cite{LMTZ2012} suggesting tighter convex relaxation
of orthogonal projectors when $d$ is known.
We also discussed in \S\ref{sec:this_work} adaptation by \citet{wang_singer_2013} of the basic ideas in here to the different synchronization problem.
Another direction was recently followed up by \cite{Coudron_Lerman2012},
where they established exact asymptotic subspace recovery under
specific sampling assumptions, which may allow relatively large magnitude of noise.
It is interesting to follow this direction and establish exact recovery when using in theory a sequence of IRLS regularization parameters $\{\delta_i\}_{i \in \nats}$ approaching zero
(in analogy to the work of \citealt{Daubechies_iterativelyreweighted}).

An interesting generalization that was not pursued so far is robust data modeling by multiple
subspaces or by locally-linear structures.
It is also interesting to know whether one can adapt the current framework so that it can detect linear
structure in the presence of both sparse elementwise corruption
(as in \citealt{candes_wright_robust_pca09}) and the type of outliers addressed in here.

\acks{This work was supported by NSF grants DMS-09-15064 and DMS-09-56072, GL was also partially supported by the IMA (during 2010-2012).
Arthur Szlam has inspired our extended research on robust $\ell_1$-type subspace recovery.
We thank John Wright for referring us to \citet{Xu2010} shortly after it appeared online and for some guidance with the real data sets. GL thanks Emmanuel Cand{\`e}s for inviting him to visit Stanford university in May 2010 and for his constructive criticism on the lack of a theoretically guaranteed algorithm for the $\ell_1$ subspace recovery of \citet{lp_recovery_part1_11}.

Supp. webpage: \url{http://www.math.umn.edu/~lerman/gms}.
}

\bibliography{refs_10_31_13}

\begin{thebibliography}{78}
\providecommand{\natexlab}[1]{#1}
\providecommand{\url}[1]{\texttt{#1}}
\expandafter\ifx\csname urlstyle\endcsname\relax
  \providecommand{\doi}[1]{doi: #1}\else
  \providecommand{\doi}{doi: \begingroup \urlstyle{rm}\Url}\fi

\bibitem[Agarwal et~al.(2012{\natexlab{a}})Agarwal, Negahban, and
  Wainwright]{Agarwal_Negahban_Wainwright_2011}
A.~Agarwal, S.~Negahban, and M.~J. Wainwright.
\newblock Noisy matrix decomposition via convex relaxation: optimal rates in
  high dimensions.
\newblock \emph{Ann. Statist.}, 40\penalty0 (2):\penalty0 1171--1197,
  2012{\natexlab{a}}.
\newblock ISSN 0090-5364.
\newblock \doi{10.1214/12-AOS1000}.

\bibitem[Agarwal et~al.(2012{\natexlab{b}})Agarwal, Negahban, and
  Wainwright]{Agarwal_Negahban_Wainwright_2011b}
A.~Agarwal, S.~Negahban, and M.~J. Wainwright.
\newblock Fast global convergence of gradient methods for high-dimensional
  statistical recovery.
\newblock \emph{The Annals of Statistics}, 40\penalty0 (5):\penalty0
  2452--2482, 2012{\natexlab{b}}.

\bibitem[Ammann(1993)]{Ammann1993}
L.~P. Ammann.
\newblock Robust singular value decompositions: A new approach to projection
  pursuit.
\newblock \emph{Journal of the American Statistical Association}, 88\penalty0
  (422):\penalty0 pp. 505--514, 1993.
\newblock ISSN 01621459.

\bibitem[Arias-Castro et~al.(2005)Arias-Castro, Donoho, Huo, and
  Tovey]{Arias-Castro05connect}
E.~Arias-Castro, D.~L. Donoho, X.~Huo, and C.~A. Tovey.
\newblock Connect the dots: how many random points can a regular curve pass
  through?
\newblock \emph{Adv. in Appl. Probab.}, 37\penalty0 (3):\penalty0 571--603,
  2005.

\bibitem[Arias-Castro et~al.(2011)Arias-Castro, Chen, and Lerman]{higher-order}
E.~Arias-Castro, G.~Chen, and G.~Lerman.
\newblock Spectral clustering based on local linear approximations.
\newblock \emph{Electron. J. Statist.}, 5:\penalty0 1537--1587, 2011.

\bibitem[Arslan(2004)]{Arslan2004}
O.~Arslan.
\newblock Convergence behavior of an iterative reweighting algorithm to compute
  multivariate {M}-estimates for location and scatter.
\newblock \emph{Journal of Statistical Planning and Inference}, 118\penalty0
  (1-2):\penalty0 115 -- 128, 2004.
\newblock ISSN 0378-3758.
\newblock \doi{10.1016/S0378-3758(02)00402-0}.

\bibitem[Bargiela and Hartley(1993)]{Bargiela_Hartley93}
A.~Bargiela and J.~K. Hartley.
\newblock Orthogonal linear regression algorithm based on augmented matrix
  formulation.
\newblock \emph{Comput. Oper. Res.}, 20:\penalty0 829--836, October 1993.
\newblock ISSN 0305-0548.
\newblock \doi{10.1016/0305-0548(93)90104-Q}.

\bibitem[Basri and Jacobs(2003)]{Basri03}
R.~Basri and D.~Jacobs.
\newblock Lambertian reflectance and linear subspaces.
\newblock \emph{IEEE Transactions on Pattern Analysis and Machine
  Intelligence}, 25\penalty0 (2):\penalty0 218--233, February 2003.

\bibitem[Basri et~al.(2011)Basri, Hassner, and
  Zelnik-Manor]{Basri_nearest_subspace11}
R.~Basri, T.~Hassner, and L.~Zelnik-Manor.
\newblock Approximate nearest subspace search.
\newblock \emph{IEEE Transactions on Pattern Analysis and Machine
  Intelligence}, 33\penalty0 (2):\penalty0 266--278, 2011.
\newblock ISSN 0162-8828.
\newblock \doi{http://doi.ieeecomputersociety.org/10.1109/TPAMI.2010.110}.

\bibitem[Bhatia and Drissi(2005)]{Bhatia2005}
R.~Bhatia and D.~Drissi.
\newblock Generalized {Lyapunov} equations and positive definite functions.
\newblock \emph{SIAM J. Matrix Anal. Appl.}, 27\penalty0 (1):\penalty0
  103--114, May 2005.
\newblock ISSN 0895-4798.
\newblock \doi{10.1137/040608970}.

\bibitem[Bradley and Mangasarian(2000)]{Bradley00kplanes}
P.~Bradley and O.~Mangasarian.
\newblock k-plane clustering.
\newblock \emph{J. Global optim.}, 16\penalty0 (1):\penalty0 23--32, 2000.

\bibitem[Brubaker(2009)]{Brubaker2009}
S.~C. Brubaker.
\newblock Robust {PCA} and clustering in noisy mixtures.
\newblock In \emph{Proceedings of the Twentieth Annual ACM-SIAM Symposium on
  Discrete Algorithms}, SODA '09, pages 1078--1087, Philadelphia, PA, USA,
  2009. Society for Industrial and Applied Mathematics.

\bibitem[Cand\`{e}s et~al.(2006)Cand\`{e}s, Romberg, and
  Tao]{candes_romberg_tao_cpam06}
E.~J. Cand\`{e}s, J.~Romberg, and T.~Tao.
\newblock Stable signal recovery from incomplete and inaccurate measurements.
\newblock \emph{Communications on Pure and Applied Mathematics}, 59\penalty0
  (8):\penalty0 1207--1223, 2006.
\newblock \doi{10.1002/cpa.20124}.

\bibitem[Cand{\`e}s et~al.(2011)Cand{\`e}s, Li, Ma, and
  Wright]{candes_wright_robust_pca09}
E.~J. Cand{\`e}s, X.~Li, Y.~Ma, and J.~Wright.
\newblock Robust principal component analysis?
\newblock \emph{Journal of the ACM (JACM)}, 58\penalty0 (3):\penalty0 11, 2011.

\bibitem[Chan and Mulet(1999)]{Chan99}
T.~F. Chan and P.~Mulet.
\newblock On the convergence of the lagged diffusivity fixed point method in
  total variation image restoration.
\newblock \emph{SIAM J. Numer. Anal.}, 36:\penalty0 354--367, 1999.
\newblock ISSN 0036-1429.
\newblock \doi{10.1137/S0036142997327075}.

\bibitem[Chandrasekaran et~al.(2011)Chandrasekaran, Sanghavi, Parrilo, and
  Willsky]{Chandrasekaran_Sanghavi_Parrilo_Willsky_2009}
V.~Chandrasekaran, S.~Sanghavi, P.~A. Parrilo, and A.~S. Willsky.
\newblock Rank-sparsity incoherence for matrix decomposition.
\newblock \emph{SIAM J. Optim.}, 21\penalty0 (2):\penalty0 572--596, 2011.
\newblock ISSN 1052-6234.
\newblock \doi{10.1137/090761793}.

\bibitem[Chin et~al.(2012)Chin, Yu, and Suter]{Chin+_hypotheis2012}
T.-J. Chin, J.~Yu, and D.~Suter.
\newblock Accelerated hypothesis generation for multistructure data via
  preference analysis.
\newblock \emph{IEEE Trans. Pattern Anal. Mach. Intell.}, 34\penalty0
  (4):\penalty0 625--638, April 2012.
\newblock ISSN 0162-8828.
\newblock \doi{10.1109/TPAMI.2011.169}.

\bibitem[Cline(1972)]{Cline1972}
A.~K. Cline.
\newblock Rate of convergence of {L}awson's algorithm.
\newblock \emph{Mathematics of Computation}, 26\penalty0 (117):\penalty0 pp.
  167--176, 1972.
\newblock ISSN 00255718.

\bibitem[Coudron and Lerman(2012)]{Coudron_Lerman2012}
M.~Coudron and G.~Lerman.
\newblock On the sample complexity of robust {PCA}.
\newblock In \emph{NIPS}, pages 3230--3238, 2012.

\bibitem[Croux and Haesbroeck(2000)]{Croux00principalcomponent}
C.~Croux and G.~Haesbroeck.
\newblock Principal component analysis based on robust estimators of the
  covariance or correlation matrix: Influence functions and efficiencies.
\newblock \emph{Biometrika}, 87:\penalty0 603--618, 2000.

\bibitem[Croux et~al.(2007)Croux, Filzmoser, and Oliveira]{Croux2007}
C.~Croux, P.~Filzmoser, and M.~Oliveira.
\newblock Algorithms for projectionc pursuit robust principal component
  analysis.
\newblock \emph{Chemometrics and Intelligent Laboratory Systems}, 87\penalty0
  (2):\penalty0 218--225, 2007.

\bibitem[d'Aspremont et~al.(2007)d'Aspremont, Ghaoui, Jordan, and
  Lanckriet]{DAspremont07}
A.~d'Aspremont, L.~El Ghaoui, M.~Jordan, and G.~Lanckriet.
\newblock A direct formulation for sparse {PCA} using semidefinite programming.
\newblock \emph{SIAM Review}, 49\penalty0 (3):\penalty0 434--448, 2007.
\newblock \doi{10.1137/050645506}.

\bibitem[Daubechies et~al.(2010)Daubechies, DeVore, Fornasier, and
  Gunturk]{Daubechies_iterativelyreweighted}
I.~Daubechies, R.~DeVore, M.~Fornasier, and C.~S. Gunturk.
\newblock Iteratively reweighted least squares minimization for sparse
  recovery.
\newblock \emph{Communications on Pure and Applied Mathematics}, 63:\penalty0
  1--38, 2010.
\newblock \doi{10.1002/cpa.20303}.

\bibitem[David and Semmes(1991)]{DS91}
G.~David and S.~Semmes.
\newblock Singular integrals and rectifiable sets in {$\reals^n${\rm:}}
  au-del\`a des graphes {L}ipschitziens.
\newblock \emph{Ast\'erisque}, 193:\penalty0 1--145, 1991.

\bibitem[Davies(1987)]{Davies1987}
P.~L. Davies.
\newblock Asymptotic behaviour of s-estimates of multivariate location
  parameters and dispersion matrices.
\newblock \emph{The Annals of Statistics}, 15\penalty0 (3):\penalty0 pp.
  1269--1292, 1987.
\newblock ISSN 00905364.

\bibitem[Davis and Kahan(1970)]{Davis1970}
C.~Davis and W.~M. Kahan.
\newblock The rotation of eigenvectors by a perturbation. iii.
\newblock \emph{SIAM J. on Numerical Analysis}, 7:\penalty0 1--46, 1970.

\bibitem[Devlin et~al.(1981)Devlin, Gnandesikan, and Kettenring]{Devlin1981}
S.~J. Devlin, R.~Gnandesikan, and J.~R. Kettenring.
\newblock Robust estimation of dispersion matrices and principal components.
\newblock \emph{Journal of the American Statistical Association}, 76\penalty0
  (374):\penalty0 pp. 354--362, 1981.
\newblock ISSN 01621459.

\bibitem[Ding et~al.(2006)Ding, Zhou, He, and Zha]{Ding+06}
C.~Ding, D.~Zhou, X.~He, and H.~Zha.
\newblock {R}1-{PCA}: rotational invariant ${L}_1$-norm principal component
  analysis for robust subspace factorization.
\newblock In \emph{ICML '06: Proceedings of the 23rd International Conference
  on Machine Learning}, pages 281--288, New York, NY, USA, 2006. ACM.
\newblock ISBN 1-59593-383-2.
\newblock \doi{10.1145/1143844.1143880}.

\bibitem[Fornasier et~al.(2011)Fornasier, Rauhut, and
  Ward]{Fornasier_low-rankmatrix}
M.~Fornasier, H.~Rauhut, and R.~Ward.
\newblock Low-rank matrix recovery via iteratively reweighted least squares
  minimization.
\newblock \emph{SIAM J. Optim.}, 21\penalty0 (4):\penalty0 1614--1640, 2011.
\newblock ISSN 1052-6234.
\newblock \doi{10.1137/100811404}.

\bibitem[Hardt and Moitra(2013)]{moitra_pca2012}
M.~Hardt and A.~Moitra.
\newblock Algorithms and hardness for robust subspace recovery.
\newblock In \emph{COLT}, pages 354--375, 2013.

\bibitem[Ho et~al.(2003)Ho, Yang, Lim, Lee, and Kriegman]{Ho03}
J.~Ho, M.~Yang, J.~Lim, K.~Lee, and D.~Kriegman.
\newblock Clustering appearances of objects under varying illumination
  conditions.
\newblock In \emph{Proceedings of International Conference on Computer Vision
  and Pattern Recognition}, volume~1, pages 11--18, 2003.

\bibitem[Hsu et~al.(2011)Hsu, Kakade, and Zhang]{Hsu_rpca_11}
D.~Hsu, S.M. Kakade, and Tong Zhang.
\newblock Robust matrix decomposition with sparse corruptions.
\newblock \emph{Information Theory, IEEE Transactions on}, 57\penalty0
  (11):\penalty0 7221 --7234, nov. 2011.
\newblock ISSN 0018-9448.
\newblock \doi{10.1109/TIT.2011.2158250}.

\bibitem[Huber and Ronchetti(2009)]{huber_book}
P.~J. Huber and E.~M. Ronchetti.
\newblock \emph{Robust Statistics}.
\newblock Wiley Series in Probability and Statistics. Wiley, Hoboken, NJ, 2nd
  edition, 2009.
\newblock ISBN 978-0-470-12990-6.
\newblock \doi{10.1002/9780470434697}.

\bibitem[Ke and Kanade(2003)]{ke_kanade03}
Q.~Ke and T.~Kanade.
\newblock Robust subspace computation using ${L}_1$ norm.
\newblock Technical report, Carnegie Mellon, 2003.

\bibitem[Kent and Tyler(1991)]{Kent1991}
J.~T. Kent and D.~E. Tyler.
\newblock Redescending {M}-estimates of multivariate location and scatter.
\newblock \emph{The Annals of Statistics}, 19\penalty0 (4):\penalty0 pp.
  2102--2119, 1991.
\newblock ISSN 00905364.

\bibitem[Kuhn(1973)]{Kuhn73}
H.~W. Kuhn.
\newblock A note on {F}ermat's problem.
\newblock \emph{Mathematical Programming}, 4:\penalty0 98--107, 1973.
\newblock ISSN 0025-5610.
\newblock 10.1007/BF01584648.

\bibitem[Kwak(2008)]{kwak08}
N.~Kwak.
\newblock Principal component analysis based on ${L}_1$-norm maximization.
\newblock \emph{Pattern Analysis and Machine Intelligence, IEEE Transactions
  on}, 30\penalty0 (9):\penalty0 1672--1680, 2008.
\newblock \doi{10.1109/TPAMI.2008.114}.

\bibitem[Lawson(1961)]{Lawson61}
C.~L. Lawson.
\newblock \emph{Contributions to the Theory of Linear Least Maximum
  Approximation}.
\newblock PhD thesis, University of California, Los Angeles, 1961.

\bibitem[Lee et~al.(2005)Lee, Ho, and Kriegman]{KCLee05}
K.-C. Lee, J.~Ho, and D.~Kriegman.
\newblock Acquiring linear subspaces for face recognition under variable
  lighting.
\newblock \emph{Pattern Analysis and Machine Intelligence, IEEE Transactions
  on}, 27\penalty0 (5):\penalty0 684--698, 2005.
\newblock ISSN 0162-8828.
\newblock \doi{10.1109/TPAMI.2005.92}.

\bibitem[Lerman and Zhang(2010)]{lp_recovery_part1_11}
G.~Lerman and T.~Zhang.
\newblock $\ell_p$-{R}ecovery of the most significant subspace among multiple
  subspaces with outliers.
\newblock \emph{ArXiv e-prints}, December 2010.
\newblock To Appear in \emph{Constructive Approximation}.

\bibitem[Lerman and Zhang(2011)]{lp_recovery_part2_11}
G.~Lerman and T.~Zhang.
\newblock Robust recovery of multiple subspaces by geometric ${{l_p}}$
  minimization.
\newblock \emph{Ann. Statist.}, 39\penalty0 (5):\penalty0 2686--2715, 2011.
\newblock ISSN 0090-5364.
\newblock \doi{10.1214/11-AOS914}.

\bibitem[{Lerman} et~al.(2012){Lerman}, {McCoy}, {Tropp}, and
  {Zhang}]{LMTZ2012}
G.~{Lerman}, M.~{McCoy}, J.~A. {Tropp}, and T.~{Zhang}.
\newblock Robust computation of linear models, or how to find a needle in a
  haystack.
\newblock \emph{ArXiv e-prints}, February 2012.

\bibitem[Li and Chen(1985)]{Li_85}
G.~Li and Z.~Chen.
\newblock Projection-pursuit approach to robust dispersion matrices and
  principal components: Primary theory and monte carlo.
\newblock \emph{Journal of the American Statistical Association}, 80\penalty0
  (391):\penalty0 759--766, 1985.
\newblock ISSN 01621459.
\newblock \doi{10.2307/2288497}.

\bibitem[Li et~al.(2004)Li, Huang, Gu, and Tian]{Li_backgroundsubtraction}
L.~Li, W.~Huang, I.~Gu, and Q.~Tian.
\newblock Statistical modeling of complex backgrounds for foreground object
  detection.
\newblock \emph{Image Processing, IEEE Transactions on}, 13\penalty0
  (11):\penalty0 1459 --1472, nov. 2004.
\newblock ISSN 1057-7149.
\newblock \doi{10.1109/TIP.2004.836169}.

\bibitem[Lin et~al.(2009)Lin, Ganesh, Wright, Wu, Chen, and Ma]{rpca_code09}
Z.~Lin, A.~Ganesh, J.~Wright, L.~Wu, M.~Chen, and Y.~Ma.
\newblock Fast convex optimization algorithms for exact recovery of a corrupted
  low-rank matrix.
\newblock In \emph{In Intl. Workshop on Comp. Adv. in Multi-Sensor Adapt.
  Processing, Aruba, Dutch Antilles}, 2009.

\bibitem[Liu et~al.(2010)Liu, Lin, and Yu]{lrr_short}
G.~Liu, Z.~Lin, and Y.~Yu.
\newblock Robust subspace segmentation by low-rank representation.
\newblock In \emph{ICML}, 2010.

\bibitem[Liu et~al.(2013)Liu, Lin, Yan, Sun, Yu, and Ma]{lrr_long}
G.~Liu, Z.~Lin, S.~Yan, J.~Sun, Y.~Yu, and Y.~Ma.
\newblock Robust recovery of subspace structures by low-rank representation.
\newblock \emph{Pattern Analysis and Machine Intelligence, IEEE Transactions
  on}, 35\penalty0 (1):\penalty0 171 --184, 2013.
\newblock ISSN 0162-8828.
\newblock \doi{10.1109/TPAMI.2012.88}.

\bibitem[Lopuha{\"a} and Rousseeuw(1991)]{lopuha_rousseeuw_robust91}
H.~P. Lopuha{\"a} and P.~J. Rousseeuw.
\newblock Breakdown points of affine equivariant estimators of multivariate
  location and covariance matrices.
\newblock \emph{Ann. Statist.}, 19\penalty0 (1):\penalty0 229--248, 1991.
\newblock ISSN 0090-5364.

\bibitem[Maronna(1976)]{Maronna1976}
R.~A. Maronna.
\newblock Robust {M}-estimators of multivariate location and scatter.
\newblock \emph{The Annals of Statistics}, 4\penalty0 (1):\penalty0 pp. 51--67,
  1976.
\newblock ISSN 00905364.

\bibitem[Maronna et~al.(2006)Maronna, Martin, and Yohai]{robust_stat_book2006}
R.~A. Maronna, R.~D. Martin, and V.~J. Yohai.
\newblock \emph{Robust Statistics: Theory and methods}.
\newblock Wiley Series in Probability and Statistics. John Wiley \& Sons Ltd.,
  Chichester, 2006.
\newblock ISBN 978-0-470-01092-1; 0-470-01092-4.

\bibitem[McCoy and Tropp(2011)]{robust_mccoy}
M.~McCoy and J.~Tropp.
\newblock Two proposals for robust {PCA} using semidefinite programming.
\newblock \emph{Elec. J. Stat.}, 5:\penalty0 1123--1160, 2011.

\bibitem[Mendelson(2003)]{mendelson03notes}
S.~Mendelson.
\newblock A few notes on statistical learning theory.
\newblock In \emph{Lecture Notes in Computer Science}, volume 2600, pages
  1--40. Springer-Verlag, 2003.

\bibitem[Nyquist(1988)]{Nyquist_l1_88}
H.~Nyquist.
\newblock Least orthogonal absolute deviations.
\newblock \emph{Computational Statistics \& Data Analysis}, 6\penalty0
  (4):\penalty0 361 -- 367, 1988.
\newblock ISSN 0167-9473.
\newblock \doi{10.1016/0167-9473(88)90076-X}.

\bibitem[Osborne and Watson(1985)]{osborne_watson85}
M.~R. Osborne and G.~A. Watson.
\newblock An analysis of the total approximation problem in separable norms,
  and an algorithm for the total $l_1 $ problem.
\newblock \emph{SIAM Journal on Scientific and Statistical Computing},
  6\penalty0 (2):\penalty0 410--424, 1985.
\newblock \doi{10.1137/0906029}.

\bibitem[Ostrowski(1966)]{Ostrowski66}
A.~M. Ostrowski.
\newblock \emph{Solution of Equations and Systems of Equations}.
\newblock Academic Press, {Second} edition, September 1966.
\newblock ISBN 0471889873.

\bibitem[Soltanolkotabi and Cand{\`e}s(2012)]{Soltanolkotabi2011}
M.~Soltanolkotabi and E.~J. Cand{\`e}s.
\newblock {A geometric analysis of subspace clustering with outliers.}
\newblock \emph{Ann. Stat.}, 40\penalty0 (4):\penalty0 2195--2238, 2012.
\newblock \doi{10.1214/12-AOS1034}.

\bibitem[Sp{\"a}th and Watson(1987)]{spath_watson1987}
H.~Sp{\"a}th and G.~A. Watson.
\newblock On orthogonal linear approximation.
\newblock \emph{Numer. Math.}, 51:\penalty0 531--543, October 1987.
\newblock ISSN 0029-599X.
\newblock \doi{10.1007/BF01400354}.

\bibitem[Stewart(1999)]{Stewart99robustparameter}
C.~V. Stewart.
\newblock Robust parameter estimation in computer vision.
\newblock \emph{SIAM Reviews}, 41:\penalty0 513--537, 1999.

\bibitem[Tao(2012)]{tao_ran_mat_book}
T.~Tao.
\newblock \emph{Topics in Random Matrix Theory}, volume 132 of \emph{Graduate
  Studies in Mathematics}.
\newblock American Mathematical Society, Providence, RI, 2012.
\newblock ISBN 978-0-8218-7430-1.

\bibitem[Tipping and Bishop(1999)]{Tipping99mixtures}
M.~Tipping and C.~Bishop.
\newblock Mixtures of probabilistic principal component analysers.
\newblock \emph{Neural Computation}, 11\penalty0 (2):\penalty0 443--482, 1999.

\bibitem[Torre and Black(2001)]{Torre2001}
F.~De~La Torre and M.~J. Black.
\newblock Robust principal component analysis for computer vision.
\newblock In \emph{Computer Vision, 2001. ICCV 2001. Proceedings. Eighth IEEE
  International Conference on}, volume~1, pages 362 --369 vol.1, 2001.
\newblock \doi{10.1109/ICCV.2001.937541}.

\bibitem[Torre and Black(2003)]{Torre03aframework}
F.~De~La Torre and M.~J. Black.
\newblock A framework for robust subspace learning.
\newblock \emph{International Journal of Computer Vision}, 54:\penalty0
  117--142, 2003.
\newblock ISSN 0920-5691.
\newblock \doi{10.1023/A:1023709501986}.

\bibitem[Tseng(2000)]{Tseng00nearest}
P.~Tseng.
\newblock Nearest $q$-flat to $m$ points.
\newblock \emph{Journal of Optimization Theory and Applications}, 105:\penalty0
  249--252, 2000.
\newblock ISSN 0022-3239.
\newblock 10.1023/A:1004678431677.

\bibitem[Tyler(1987)]{tyler_dist_free87}
D.~E. Tyler.
\newblock A distribution-free {$M$}-estimator of multivariate scatter.
\newblock \emph{Ann. Statist.}, 15\penalty0 (1):\penalty0 234--251, 1987.
\newblock ISSN 0090-5364.
\newblock \doi{10.1214/aos/1176350263}.

\bibitem[Vershynin(2012)]{vershynin_book}
R.~Vershynin.
\newblock Introduction to the non-asymptotic analysis of random matrices.
\newblock In \emph{Compressed sensing}, pages 210--268. Cambridge Univ. Press,
  Cambridge, 2012.

\bibitem[Voss and Eckhardt(1980)]{Voss80}
H.~Voss and U.~Eckhardt.
\newblock Linear convergence of generalized weiszfeld's method.
\newblock \emph{Computing}, 25:\penalty0 243--251, 1980.
\newblock ISSN 0010-485X.
\newblock \doi{10.1007/BF02242002}.

\bibitem[{Wang} and {Singer}(2013)]{wang_singer_2013}
L.~{Wang} and A.~{Singer}.
\newblock Exact and stable recovery of rotations for robust synchronization.
\newblock \emph{Information and Inference}, 2013.
\newblock \doi{10.1093/imaiai/iat005}.

\bibitem[Watson(2001)]{watson2001orth_l1}
G.~A. Watson.
\newblock \emph{Some Problems in Orthogonal Distance and Non-Orthogonal
  Distance Regression}.
\newblock Defense Technical Information Center, 2001.
\newblock URL \url{http://books.google.com/books?id=WKKWGwAACAAJ}.

\bibitem[Watson(2002)]{Watson02}
G.~A. Watson.
\newblock On the gauss-newton method for $l_1$ orthogonal distance regression.
\newblock \emph{IMA Journal of Numerical Analysis}, 22\penalty0 (3):\penalty0
  345--357, 2002.
\newblock \doi{10.1093/imanum/22.3.345}.

\bibitem[Weiszfeld(1937)]{Weiszfeld1937}
E.~Weiszfeld.
\newblock Sur le point pour lequel la somme des distances de n points donne's
  est minimum.
\newblock \emph{Tohoku Math. J.}, 43:\penalty0 35--386, 1937.

\bibitem[Xu et~al.(2010{\natexlab{a}})Xu, Caramanis, and
  Mannor]{Xu2010_highdimensional}
H.~Xu, C.~Caramanis, and S.~Mannor.
\newblock Principal component analysis with contaminated data: The high
  dimensional case.
\newblock In \emph{COLT}, pages 490--502, 2010{\natexlab{a}}.

\bibitem[Xu et~al.(2010{\natexlab{b}})Xu, Caramanis, and Sanghavi]{Xu2010}
H.~Xu, C.~Caramanis, and S.~Sanghavi.
\newblock Robust {PCA} via outlier pursuit.
\newblock In \emph{NIPS}, pages 2496--2504, 2010{\natexlab{b}}.

\bibitem[Xu et~al.(2012)Xu, Caramanis, and Sanghavi]{Xu2012}
H.~Xu, C.~Caramanis, and S.~Sanghavi.
\newblock Robust {PCA} via outlier pursuit.
\newblock \emph{Information Theory, IEEE Transactions on}, PP\penalty0
  (99):\penalty0 1, 2012.
\newblock ISSN 0018-9448.
\newblock \doi{10.1109/TIT.2011.2173156}.

\bibitem[Xu and Yuille(1995)]{Xu1995}
L.~Xu and A.L. Yuille.
\newblock Robust principal component analysis by self-organizing rules based on
  statistical physics approach.
\newblock \emph{Neural Networks, IEEE Transactions on}, 6\penalty0
  (1):\penalty0 131--143, 1995.
\newblock ISSN 1045-9227.
\newblock \doi{10.1109/72.363442}.

\bibitem[{Zhang}(2012)]{Teng_log_rpca}
T.~{Zhang}.
\newblock Robust subspace recovery by geodesically convex optimization.
\newblock \emph{ArXiv e-prints}, 2012.

\bibitem[Zhang et~al.(2009)Zhang, Szlam, and Lerman]{MKF_workshop09}
T.~Zhang, A.~Szlam, and G.~Lerman.
\newblock Median {$K$}-flats for hybrid linear modeling with many outliers.
\newblock In \emph{Computer Vision Workshops ({ICCV} Workshops), 2009 {IEEE}
  12th International Conference on Computer Vision}, pages 234--241, Kyoto,
  Japan, 2009.
\newblock \doi{10.1109/ICCVW.2009.5457695}.

\bibitem[Zhang et~al.(2010)Zhang, Szlam, Wang, and Lerman]{LBF_cvpr10}
T.~Zhang, A.~Szlam, Y.~Wang, and G.~Lerman.
\newblock Randomized hybrid linear modeling by local best-fit flats.
\newblock In \emph{Computer Vision and Pattern Recognition (CVPR), 2010 IEEE
  Conference on}, pages 1927 --1934, jun. 2010.
\newblock \doi{10.1109/CVPR.2010.5539866}.

\bibitem[Zhang et~al.(2012)Zhang, Szlam, Wang, and Lerman]{LBF_journal12}
T.~Zhang, A.~Szlam, Y.~Wang, and G.~Lerman.
\newblock Hybrid linear modeling via local best-fit flats.
\newblock \emph{International Journal of Computer Vision}, 100:\penalty0
  217--240, 2012.
\newblock ISSN 0920-5691.
\newblock \doi{10.1007/s11263-012-0535-6}.

\end{thebibliography}

\end{document}